\pdfoutput=1

\documentclass[letterpaper]{article}

\usepackage{fullpage}
\usepackage{booktabs}
\usepackage[utf8]{inputenc}
\usepackage{amsmath}
\usepackage{amssymb,amsthm}
\usepackage[numbers]{natbib}

\usepackage{caption,enumitem,mathtools,cases}
\usepackage{subcaption}
\usepackage{multirow}
\usepackage{fancyhdr,graphicx,color}
\usepackage[utf8]{inputenc}
\usepackage{graphicx}
\usepackage{graphics}
\theoremstyle{definition}

\newtheorem{remark}{Remark}
\theoremstyle{plain}
\newtheorem{lemma}{Lemma}
\newtheorem{theorem}{Theorem}
\newtheorem{proposition}{Proposition}

\newcommand{\argmax}{\mathop{\mathrm{argmax}}}

\newcommand{\cN}{\mathcal{N}}

\newcommand{\R}{\mathbb{R}}
\newcommand{\tr}{\mathsf{Tr}}
\newcommand{\rank}{\mathsf{rank}}
\newcommand{\p}{{\rm I}\kern-0.18em{\rm P}}
\newcommand{\E}{{\rm I}\kern-0.18em{\rm E}}

\newcommand{\B}{\boldsymbol}

\newcommand{\gd}{\mathrm{GD}}
\newcommand{\sam}{\mathrm{SAM}}
\newcommand{\diag}{\mathrm{diag}}

\newcommand{\bias}{\mathrm{Bias}^2}
\newcommand{\var}{\mathrm{Var}}
\newcommand{\error}{\mathrm{Error}}
\newcommand{\errortest}{\mathrm{Error}}

\newcommand{\snr}{\mathrm{SNR}}
\newcommand{\Hk}{\mathcal{H}_K}
\newcommand{\Hkp}{\mathcal{H}_{K_+}}
\newcommand{\Hkn}{\mathcal{H}_{K_-}}

\newcommand{\KX}{\B{K}_{\B{X}}}
\newcommand{\Kdot}{\B{K}(\B{X},\cdot)}

\newtheorem*{assumption*}{Assumption}
\date{}

\usepackage{hyperref}

\begin{document}
	
	\title{On Statistical Properties of Sharpness-Aware Minimization: Provable Guarantees}
	
	
\author{Kayhan Behdin,$^{\dagger}$  Rahul Mazumder$^{\dagger,\ddag}$ 
	\\[4pt]
	$^{\dagger}$\textit{MIT Operations Research Center, Cambridge, MA}
	\\[2pt]
	$^{\ddag}$\textit{MIT Sloan Schools of Management, Cambridge, MA}}
	
\maketitle
\begin{abstract}
	Sharpness-Aware Minimization (SAM) is a recent optimization framework aiming to improve the deep neural network generalization, through obtaining flatter (i.e. less sharp) solutions. As SAM has been numerically successful, recent papers have studied the theoretical aspects of the framework and have shown SAM solutions are indeed flat. However, there has been limited theoretical exploration regarding statistical properties of SAM. In this work, we directly study the statistical performance of SAM, and present a new theoretical explanation of why SAM generalizes well. To this end, we study two statistical problems, neural networks with a hidden layer and kernel regression, and prove under certain conditions, SAM has smaller prediction error over Gradient Descent (GD). Our results concern both convex and non-convex settings, and show that SAM is particularly well-suited for non-convex problems. Additionally, we prove that in our setup, SAM solutions are less sharp as well, showing our results are in agreement with the previous work. Our theoretical findings  are validated using numerical experiments on numerous scenarios, including deep neural networks.
\end{abstract}

\section{Introduction}\label{sec:intro}
Training Deep Neural Networks (DNN) can be challenging, as it requires minimizing non-convex loss functions with numerous local minima (and saddle points). 
As different local minima have different generalization properties, recent research has been focused on developing optimization methods and techniques that improve the quality of DNN training, leading to better generalization over unseen data. 

An important property of a DNN solution's landscape is its sharpness, which is defined as how rapidly the loss value changes locally. A flatter solution is a solution where the highest and the lowest loss values in the region do not differ too much. Sharpness measures in practice include the largest eigenvalue~\citep{sharpnesspaper} or trace~\citep{minimumsharpness} of the Hessian of the loss. Sharpness-Aware Minimization (SAM)~\citep{sampaper} is an optimization framework that builds on the observation that sharpness of the training loss correlates
with the generalization performance of a DNN. Specifically, flatter solutions to DNNs have been found to generalize better~\citep{flatness1,flatness2,flatness3,flatness4,sampaper}. Thus, in SAM, the loss function is modified, in a way that it encourages convergence to flatter regions of the loss. SAM has been shown to be empirically successful in numerous tasks~\citep{samvit,behdin2022improved,nlpmsam} and has been extended to several variations~\citep{gsampaper,samforfree}. Thus, there has been a growing interest in understanding the theoretical underpinnings of SAM. In this paper, our goal is to further the theoretical understanding of SAM, by exploring the implicit regularization implications of the algorithm dynamics.

\begin{figure}[t!]
	\centering
	\begin{tabular}{cccc}
		CIFAR100 & Theory, noiseless & CIFAR10-noisy &  Theory, noisy \\
		\includegraphics[width=0.2\linewidth,trim =0.8cm 0cm .8cm 0cm, clip = true]{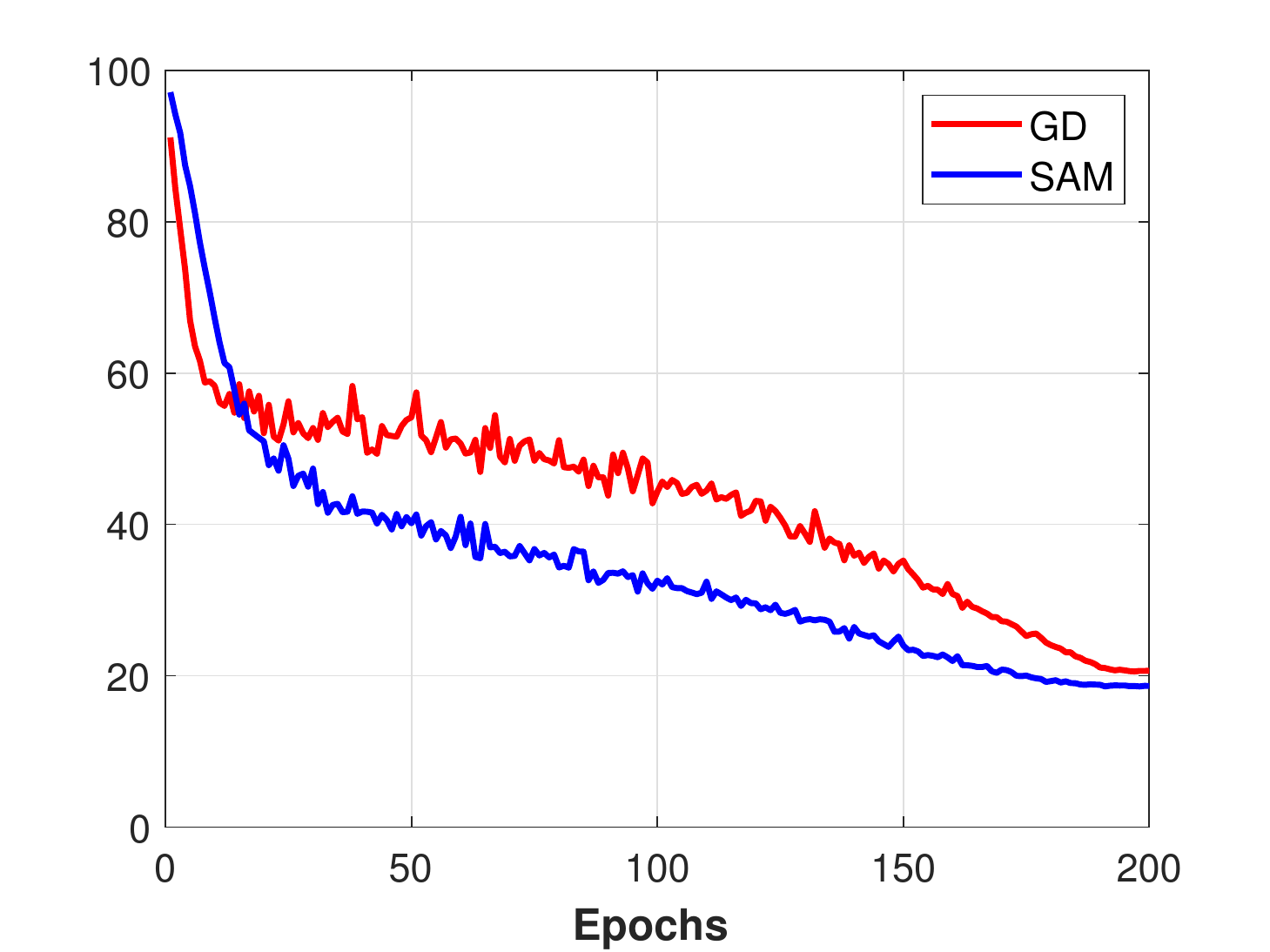}&
		\includegraphics[width=0.2\linewidth,trim =0.8cm 0cm .8cm 0cm, clip = true]{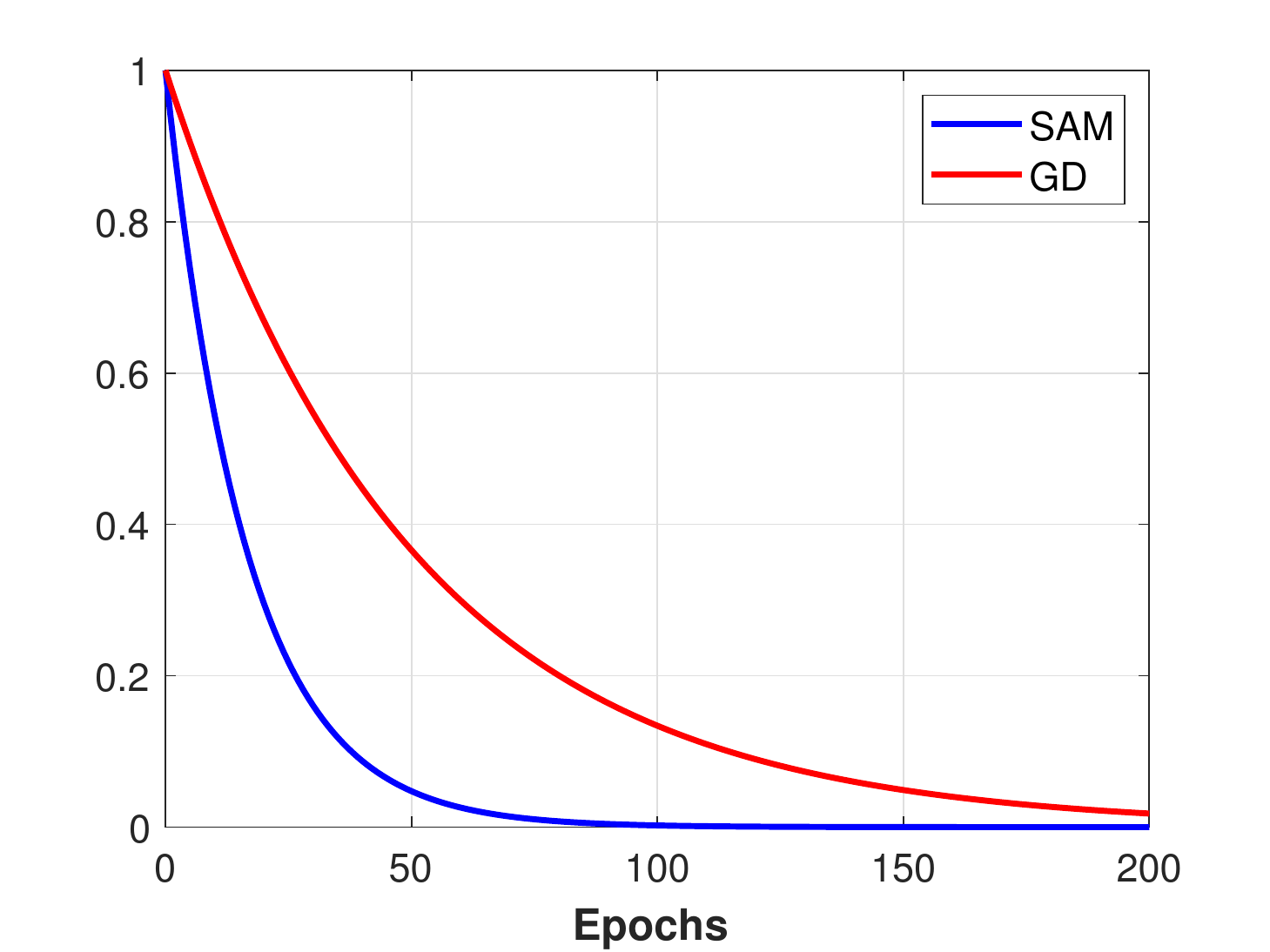}&
		\includegraphics[width=0.2\linewidth,trim =0.8cm 0cm .8cm 0cm, clip = true]{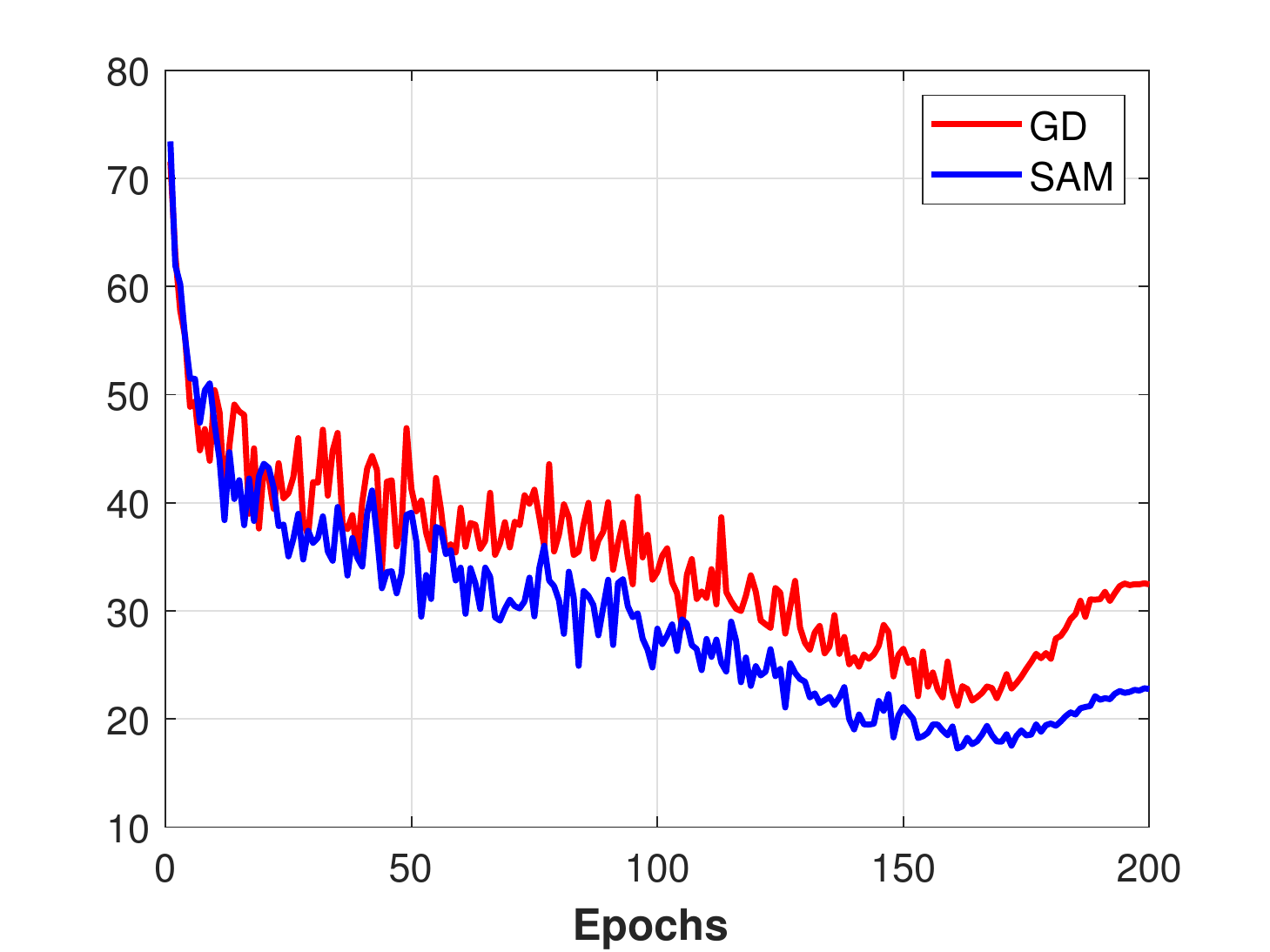}&
		\includegraphics[width=0.2\linewidth,trim =.8cm 0cm .8cm 0cm, clip = true]{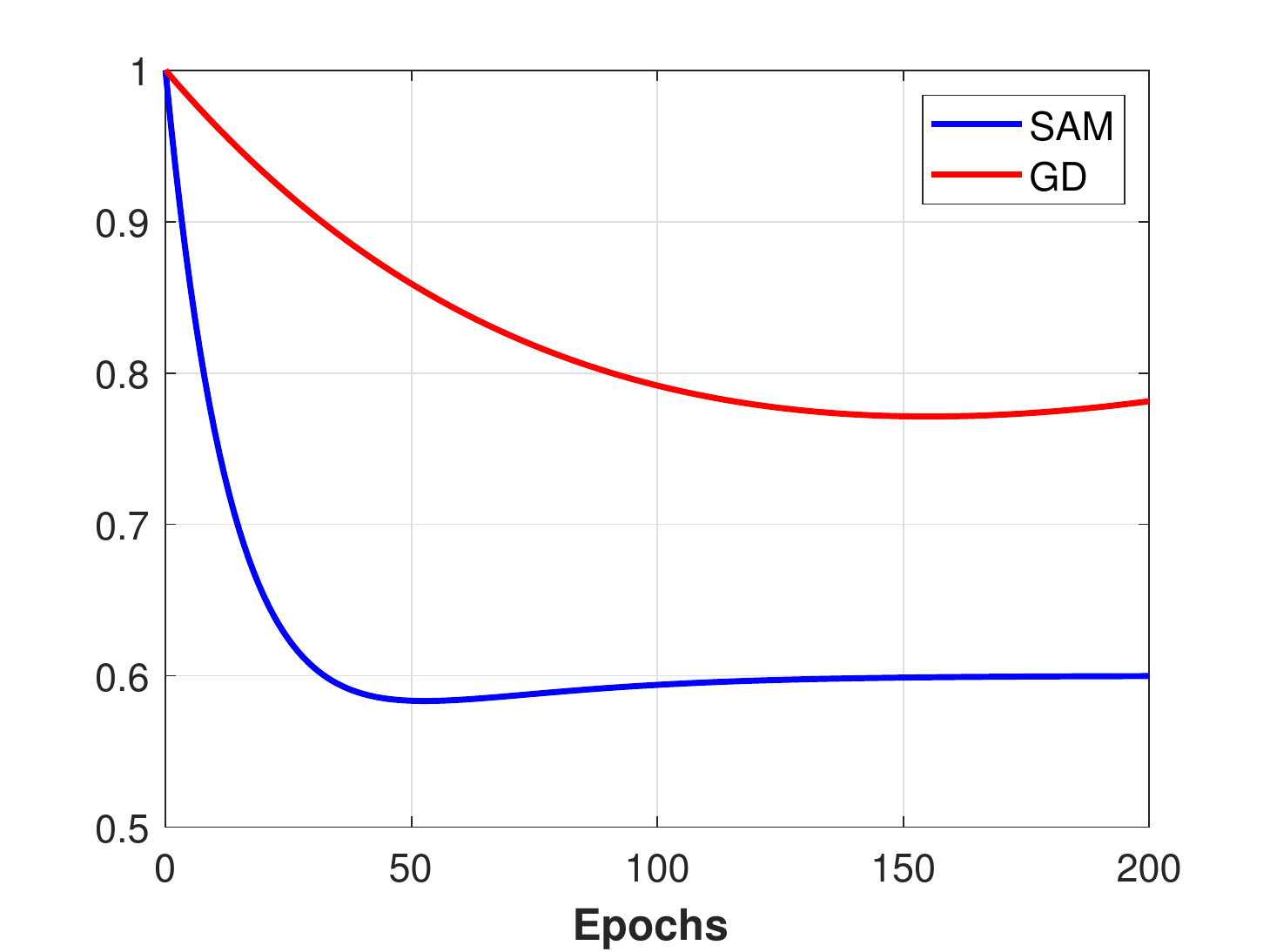} \\
		(a) & (b) & (c) & (d)
	\end{tabular}
	\caption{\small  Comparison of GD and SAM classification error on the validation set, over epochs for different cases. (a): CIFAR100 with clean labels on ResNet50. (b): Theoretical error curve for a particular model from our analysis for a noiseless case. (c): CIFAR10 with noisy training labels on ResNet50. (d): Theoretical error curve for a noisy case. Note that our theoretical plots capture the relative profile of real data.}
	\label{fig:intro}
\end{figure}

\textbf{Related Work.}~\cite{sampaper} introduced SAM and presented upper bounds for the generalization performance of SAM. Their bound suggests that SAM should generalize well, however, their result does not completely explain why SAM performs better than a vanilla training using Stochastic Gradient Descent (SGD). 
Most current papers explain the statistical performance of SAM through analyzing the SAM solution loss landscape and geometry, specially, its sharpness. Particularly, continuous time analysis (i.e. when the step size is infinitesimally small) has been used to show SAM can choose more sparse solutions~\citep{icmlpaper}, can regularize eigenvalues of the Hessian of the loss~\citep{optsam1} and eventually select flatter solutions~\citep{newsde}. The recent work of~\cite{optsam2} also explores the connections between SAM and variational inference, and how SAM seeks flatter solutions. \cite{neweigen} show that SAM regularizes the eigenvalues of the Hessian, resulting in a flatter solution. Another interesting work is by~\cite{sambounce} which explores SAM's trajectory for quadratic loss functions and explains how SAM can lead to flatter minima. Although the fact that SAM solutions are flatter partially explains the good generalization of SAM, we note that sharp minima can generalize as well~\citep{sharpcan,kaddour2022when} and sharpness can be generally manipulated by reparameterizing the network~\citep{minimumsharpness}.  This shows the need for a \emph{statistical} analysis of SAM, rather than a \emph{geometric} one.

\textbf{Summary of Results and Approach.} In a departure from the literature, we directly study the statistical properties and performance of SAM. To this end, we consider two statistical problems, a neural network with a hidden layer, and (indefinite) kernel regression, as kernel methods have been shown to be closely related to DNNs and understanding kernel methods is valuable in the DNN literature~\cite{kerneldeep1,kerneldeep2,kerneldeep3,kerneldeep4,kerneldeep5}. We present a crisp characterization of the prediction error of SAM and Gradient Descent (GD) for these two problems over the course of the algorithm, and show that under certain conditions, SAM can have a lower prediction error compared to GD. In our analysis, we study both convex and non-convex problems and show that SAM particularly works well in the non-convex cases, where GD might have unbounded error unlike SAM. Moreover, we show that SAM solutions in our setup tend to be flatter compared to GD, which theoretically shows the correlation between statistical performance and sharpness.

On a technical level, we characterize the SAM trajectory for the aforementioned problems and show a bias-variance trade-off for the prediction error of the algorithm, where bias generally decreases over iterations and variance increases. We show that SAM has a lower bias compared to GD, while GD's variance can be lower than SAM's. This shows SAM performs better when bias is the dominant term, for example when the noise is not too large or the total number of epochs is finite, as is in practice~\citep{early-deep}, specially for large models~\citep{undertrain1,undertrain2}. Moreover, we show that in non-convex settings, GD can have unbounded bias and variance while SAM is able to keep the error bounded, showing a better performance. Our numerical results on several models including deep neural networks agree with our theoretical insights. 

We use numerical experiments to illustrate some of our results. In Figure~\ref{fig:intro}(a), we compare SAM and GD classification error on the validation set over epochs when training ResNet50 network on CIFAR100 dataset (see Section~\ref{sec:Numerical} for more details on numerical experiments). We see that SAM has better accuracy over GD for almost all epochs, specially in earlier phases of training, which can be explained by our theory. As the training labels are not noisy in this case, bias is likely to be dominant and as we show, SAM's bias is less than GD's for all iterations under our model assumptions. In fact, in Figure~\ref{fig:intro}(b) we show the error plot calculated from our theory for a noiseless model\footnote{The details of plots in Figures~\ref{fig:intro}(b,d) are discussed in Appendix~\ref{app:toy}. The plots show error for a kernel regression problem with least-squares loss.}, which follows the same trends as Figure~\ref{fig:intro}(a), showing how our theory can explain the differences of GD/SAM in practice. In another case, we compare the performance of SAM and GD for CIFAR10 with training label noise in Figure~\ref{fig:intro}(c). Both methods perform worse in later epochs, which can be due to variance becoming dominant. However, GD performs even worse than SAM in the noisy setup. As we show, in the non-convex settings GD can have larger (and even unbounded) variance over SAM, which explains the performance gap seen here. Particularly, Figure~\ref{fig:intro}(d) plots the error from our theory for a noisy model, which again, shows similar trends to the real data plots, such as non-monotonicity of the error and the increasing gap between SAM and GD in later iterations.

We note that our approach is different from the previous work. Instead of studying geometric properties of SAM's solution such as its sharpness, which can partially explain why SAM generalizes better, we directly study the statistical performance of SAM. Hence, we present a direct explanation for SAM's performance in practice, rather than relying on the correlation between flatness and generalization.  Moreover, our analysis is different from previous work, which does not require us to assume the step size is infinitesimally small, unlike most current work~\citep{newsde,icmlpaper,optsam1,optsam2}. This provides insights for non-infitnesimal step sizes used in practice.

\textbf{Our contributions.} Our contributions in this paper can be summarized as follows:  \textbf{(i)} We study the statistical performance of SAM for one layer neural networks and (indefinite) kernel regression; \textbf{(ii)} We show that for these two problem classes, SAM has lower prediction error over GD under certain conditions, specially for non-convex settings; \textbf{(iii)} We show that in our settings, SAM tends to be flatter, confirming the correlation between generalization and flatness; \textbf{(iv)} We verify our theoretical findings using numerical experiments on synthetic and real data, and models including DNNs.

\section{SAM: An Overview}
Let $f:\R^p\mapsto \R$ be the objective function that we seek to minimize. In many machine learning applications in particular, we have $f(\B{w})=\sum_{i=1}^n f_i(\B{w})/n$ where $f_i$ is the loss value corresponding to the $i$-th observation. A standard approach to minimizing $f$ is the GD approach where the model parameters, or weights, are updated by the iterations
\begin{equation}\label{gd-update}
	\begin{aligned}
		\B{w}_{k+1}^{\gd} &=  \B{w}_{k}^{\gd} - \eta \nabla f( \B{w}_k^{\gd})
	\end{aligned}
\end{equation}
where $\eta>0$ is the step size or learning rate. In SAM~\citep{sampaper}, the goal is to find a flatter solution that does not fluctuate too much in a neighborhood of the solution. Therefore, SAM modifies $f$ as
\begin{equation}\label{fsamexact}
	f^{\sam}(\B{w})=\max_{\|\B{\varepsilon}\|_2\leq \rho}f(\B{w}+\B{\varepsilon})
\end{equation}
and GD is then applied over $f^{\sam}$, which captures the worst objective locally. The hope is that by minimizing $f^{\sam}$, a solution is found that does not perform bad locally, and hence the local loss function is flat. As calculating $f^{\sam}$ in closed form is difficult,~\cite{sampaper} suggest to approximate $f$ with a linear function, i.e.,
$$\argmax_{\|\B{\varepsilon}\|_2\leq \rho}f(\B{w}+\B{\varepsilon})\approx  \argmax_{\|\B{\varepsilon}\|_2\leq \rho}f(\B{w})+\B{\varepsilon}^T\nabla f(\B{w}).$$
The linear approximation leads to~\citep{sampaper}: 
$$f^{\sam}(\B{w})\approx f(\B{w}+\rho \nabla f(\B{w})/\|\nabla f(\B{w})\|_2 ).$$
Taking the gradient of this approximation and ignoring second order terms, the SAM updates are given as (we refer to~\cite{sampaper} for details of derivation)
\begin{equation}\label{sam-update}
	\begin{aligned}
		\B{w}_{k+1}^{\sam} &=  \B{w}_{k}^{\sam} - \eta \nabla f( \B{w}_k^{\sam}+\rho\nabla f( \B{w}_k^{\sam})).
	\end{aligned}
\end{equation}
We note that in~\eqref{sam-update}, we ignored the normalization of the inner gradient. But recent work~\citep{icmlpaper} has shown that the effect of such normalization can be neglected and we follow suit. We also note that our analysis in this work is done directly on~\eqref{sam-update} (based on the linear approximation to $f$) which is implemented in practice, unlike the original loss $f^{\sam}$ which is hard to compute. 

\subsection{Overview of Results}\label{sec:overview}
Throughout the paper, we assume $n$ data points $(y_i,\B{x}_i)_{i=1}^n$ are given with $\B{x}_i\in\R^d$. In our statistical model, each observation is $y_i=y_i^*+\epsilon_i$ where $y_i^*$ is the true noiseless observation and $\epsilon_i$'s are the zero-mean independent noise values with $\E[\B{\epsilon}\B{\epsilon}^T]=\sigma^2\B{I}$. We let $\Phi(\B{w};\B{x}_i)$ to be our predicted value for observation $i$, where $\B{w}\in\R^p$ parameterizes the model. We consider the least squares loss as
\begin{equation}\label{quadratic-loss}
	f(\B{w})=\frac{1}{2}\sum_{i=1}^n\left(y_i-\Phi(\B{w};\B{x}_i)\right)^2.
\end{equation}
The expected prediction error for a solution $\B{w}$ is therefore defined as 
\begin{equation}\label{bias-var-new}
	\error(\B{w})=\E_{\B{\epsilon}}\left[\frac{1}{n}\sum_{i=1}^n\left(y_i^*-\Phi(\B{w};\B{x}_i)\right)^2\right].
\end{equation}
One hence can decompose error as 
\begin{equation}\label{new-error}
	\error(\B{w})=\underbrace{\frac{1}{n}\sum_{i=1}^n\left(y_i^*-\E_{\B{\epsilon}}\left[\Phi(\B{w};\B{x}_i)\right]\right)^2}_{\bias(\B{w})} + \underbrace{\frac{1}{n}\E_{\B{\epsilon  }}\left[\sum_{i=1}^n\left(\Phi(\B{w};\B{x}_i)-\E_{\B{\epsilon}}\left[\Phi(\B{w};\B{x}_i)\right]\right)^2\right]}_{\var(\B{w})}.
\end{equation}
The bias term in~\eqref{new-error} captures how far the expected predicted value is from the true model, while the variance term is the variance of the prediction resulting from the noise. 

We discuss the details of models we study in Section~\ref{sec:def-nn} for the neural network model, and in Section~\ref{sec:def-kernel} for the kernel regression case. Our goal is to show that under certain condition, SAM has lower statistical error compared to GD.
To this end, we will characterize the bias and variance terms in~\eqref{new-error}. Specifically, we show that in all cases that we consider, SAM has a lower bias compared to GD. Moreover, SAM has higher variance in convex settings, but has significantly lower variance in non-convex settings. This quantifies that SAM is well-suited for non-convex problems.
\section{Statistical Models}\label{sec:models}
Before stating our results, we discuss two important statistical models that we consider and present a formal problem definition for each.

\subsection{Neural Networks with a Hidden Layer}\label{sec:def-nn}
Let $\phi(x):\R\mapsto\R$ be a possibly non-linear activation function. A neural network with one hidden layer and $L$ hidden neurons can be defined as $\Phi(\B{w};\B{x})=\sum_{l=1}^L\phi(\B{x}^T\B{w}^{(l)})$  where $\B{w}^{(l)}\in\R^d$ and $\B{w}=(\B{w}^{(1)},\cdots,\B{w}^{(L)})\in\R^{p}$ where $p=dL$. For the rest of the paper, we consider the ReLU as the activation function, $\phi(x)=\max(0,x)$. 
Let $\B{a}(\B{w};\B{x})=(\B{a}_1(\B{w};\B{x}),\cdots,\B{a}_L(\B{w};\B{x}))\in\R^p$ where for $l\in[L]$, we have $ \B{a}_l(\B{w};\B{x})\in\R^d$,
\begin{equation}
	\B{a}_l(\B{w};\B{x})=\begin{cases} \B{0} & \mbox{if } \B{x}^T\B{w}^{(l)}\leq 0 \\
		\B{x} & \mbox{if } \B{x}^T\B{w}^{(l)}> 0.
	\end{cases}
\end{equation}
Under this notation, for the ReLU activation we have $\Phi(\B{w};\B{x})=\B{a}(\B{w},\B{x})^T\B{w}$. We study the sequence $\B{w}_k^{\sam}$ from~\eqref{sam-update} where $f(\B{w})$ is given in~\eqref{quadratic-loss}. In particular, we let $\B{w}_k^{\gd}$ to be the sequence from~\eqref{gd-update} with $\rho=0$. We assume both SAM and GD use the same step size $\eta$ and they both start from an initial solution such as $\B{w}_0$. We also consider the following assumptions.

\begin{enumerate}[leftmargin=*,label=\textbf{(A\arabic*})]
	\item \label{assumption-relu} There exists $\bar{k}\geq 1$ such that for $0 \leq k\leq \bar{k}$ and $i\in[n]$, we have $$\B{a}(\B{w}_k^{\sam};\B{x}_i)=\B{a}(\B{w}_k^{\gd};\B{x}_i)=\B{a}(\B{w}_0;\B{x}_i).$$
	\item \label{assumption-underlying} There exists $\bar{\B{w}}\in\R^p$ such that $\B{a}(\bar{\B{w}};\B{x}_i)=\B{a}(\B{w}_0;\B{x}_i)$ for $i\in[n]$ and $y_i^*=\B{a}(\bar{\B{w}};\B{x}_i)^T\bar{\B{w}}$.
\end{enumerate}

Assumption~\ref{assumption-relu} states that the quantities $\B{x}^T\B{w}^{(l)}$ do not change sign over the course of the algorithm to avoid non-differentiability of ReLU. This can be ensured by choosing a sufficiently small step size or studying the method near a local minimum where the solution does not change significantly, a common approach to studying DNNs~\citep{flatness4,sharpnesspaper,flatness5}. Moreover, as $\B{a}(\B{w};\B{x})\in\R^{dL}$, Assumption~\ref{assumption-underlying} is likely to hold true if $L$ is sufficiently large (i.e. the total number of hidden neurons is large).

It is worth noting that if $\phi(x)=x$ and $L=1$, $\Phi(\B{w};\B{x})=\B{w}^T\B{x}$ which simplifies the model described above to the ordinary least-squares problem. By taking $\B{a}(\B{w};\B{x})=\B{x}$, we have $\Phi(\B{w};\B{x})=\B{a}(\B{w};\B{x})^T\B{w}=\B{x}^T\B{w}$ similar to the ReLU case above. Moreover, Assumption~\ref{assumption-relu} holds trivially for the linear regression case, and Assumption~\ref{assumption-underlying} simplifies to existence of $\bar{\B{w}}\in\R^d$ such that $y_i^*=\B{x}_i^T\bar{\B{w}}$ for $i\in[n]$ which is standard in the linear regression literature. Therefore, the framework developed here for ReLU networks can be readily applied to the linear regression problem.

\subsection{Kernel Regression} \label{sec:def-kernel}
Kernel methods and feature mappings have been a staple of machine learning algorithms in different applications~\citep{2008kernel}. Moreover, kernel methods have been studied to better understand optimization and generalization in machine learning~\citep{kernelgeneral}. This is specially interesting as a long line of work has explored connections and similarities between DNNs and kernels~\cite{kerneldeep1,kerneldeep2,kerneldeep3,kerneldeep4,kerneldeep5}, making the analysis of kernel methods even more important. Let $K:\R^d\times\R^d \mapsto \R$ be a kernel and $\B{X}\in\R^{n\times d}$ be the model matrix with rows of $\B{x}_1,\cdots,\B{x}_n$. We define the Gram matrix associated with this kernel and data as 
$\KX = [K(\B{x}_i,\B{x}_j)]$.
A classical assumption in kernel learning is that $K$ is Positive Semidefinite (PSD), that is $\KX$ is PSD for any $\B{X}\in\R^{n\times d}$ and $n\geq 1$. However, there has been a growing interest in learning with indefinite kernels as they often appear in practice due to noisy observations and/or certain data structures (see~\cite{indefinite1,indefinite2,indefinire3,indefinitenew} and references therein). Therefore, throughout this paper, we do not assume $K$ is PSD. In fact, we assume $K=K_+-K_-$ where $K_+,K_-$ are two PSD kernels, resulting in $K$ being indefinite. We use $\Hk$ to denote the Reproducing Kre\u{\i}n Kernel Space (RKKS) for which $K$ is the reproducing kernel. Note that $\Hk=\Hkp\bigoplus \Hkn$ where $\Hkp,\Hkn$ are Reproducing Kernel Hilbert Spaces (RKHS) associated with $K_+,K_-$ and $\bigoplus$ denotes orthogonal direct sum~\cite{indefinite1}. We also assume $K$ is symmetric, that is $K(\B{x}_i,\B{x}_j)=K(\B{x}_j,\B{x}_i)$ for all $\B{x}_i,\B{
	x}_j$.  Given pairs of observations $(y_i,
\B{x}_i)_{i=1}^n$, we seek to learn the function $h\in\Hk$ such that $h(\B{x}_i)\approx y_i$ for all $i$. To this end, for $h\in\Hk$ we define the loss
\begin{equation}
	L[h] = \frac{1}{2}\sum_{i=1}^n (h(\B{x}_i)-y_i)^2.
\end{equation}
We note that $L[h]$ is a function of $h\in\Hk$. The gradient of this loss then can be calculated as\footnote{We provide a short review of kernel gradients in Appendix~\ref{app:kernels}.}
\begin{equation}
	\nabla L[h] = \sum_{i=1}^n (h(\B{x}_i) -y_i)K(\B{x}_i,\cdot)\in\Hk
\end{equation}
where $K(\B{x},\cdot):\R^d\mapsto\R$ denotes the evaluation function, $K(\B{x},\cdot)(\B{y})=K(\B{x},\B{y})$. Although the SAM algorithm was introduced in the context of losses in $\R^p$, one can mimic SAM in the RKKS. Specifically, we define KernelSAM, an equivalent of SAM algorithm in the RKKS, by iterations
\begin{equation}\label{kernel-sam}
	h_{k+1}^{\sam}=h_k^{\sam} - \eta \nabla L[h_k^{\sam}+\rho\nabla L[h_k^{\sam}]].
\end{equation}
Our first result is a representer theorem for KernelSAM. For $\B{w}\in\R^n$, we will use the notation
$$\B
{w}^T\Kdot:=\sum_{i=1}^n w_i K(\B{x}_i,\cdot)\in\Hk.$$

\begin{theorem}\label{representerthm}
	Suppose $h_0^{\sam}=0$. Then, for $k\geq 1$, there exists $\B{w}_k^{\sam}\in\R
	^n$ such that
	$h_k^{\sam}=(\B{w}_k^{\sam})^T\Kdot$.\footnote{An explicit expression for updates can be found in~\eqref{beta-representer}.}
\end{theorem}
Theorem~\ref{representerthm} shows that at each iteration, the SAM solution can be represented as a linear combination of $K(\B{x}_i,\cdot)$ which allows us to directly study $\B{w}_k^{\sam}$. Therefore, using the notation from Section~\ref{sec:overview}, 
\begin{equation}\label{kernel-phi}
	\Phi(\B{w}_k^{\sam};\B{x})=\sum_{j=1}^n (w_k^{\sam})_jK(\B{x}_j,\B{x})=h_k^{\sam}(\B{x}).
\end{equation}

Similar to the case of ReLU networks, we seek to characterize the error for KernelSAM. To this end, we assume the model is well-specified and there exists $\bar{\B{w}}\in\R^n$ such that 
\begin{equation}
	y_i=\underbrace{\sum_{j=1}^n\bar{w}_j K(\B{x}_j,\B{x}_i)}_{y_i^*} + \epsilon_i
\end{equation}
where $\epsilon_i$'s are the noise values, independent of $\B{X}$, with the property $\E[\B{\epsilon}]=\B{0}$ and $\E[\B{\epsilon}\B{\epsilon}^T]=\sigma^2\B{I}$. 
With this notation, we let $\bar{h}=\bar{\B{w}}^T\Kdot$ to be the noiseless estimator.  The expected error for $h_k^{\sam}=(\B{w}_k^{\sam})^T\Kdot\in\Hk$ is defined as 
$$\error(\B{w}_k^{\sam})=\E_{\B{\epsilon}}\left[\frac{1}{n}\sum_{i=1}^n\left(\bar{h}(\B{x}_i)-h_k^{\sam}(\B{x}_i)\right)^2\right]=\E_{\B{\epsilon}}\left[\frac{1}{n}\sum_{i=1}^n\left(y_i^*-\Phi(\B{w}_k^{\sam};\B{x}_i)\right)^2\right]$$
with $\Phi(\cdot;\cdot)$ defined in~\eqref{kernel-phi}. Our final result in this section shows that under the model discussed here, KernelSAM is equivalent to applying SAM on a (non-convex) quadratic objective.
\begin{theorem}\label{thm:kernel-quad}
	The solution $\B{w}_k^{\sam}$ defined in Theorem~\ref{representerthm} follows~\eqref{sam-update} where 
	\begin{equation}\label{kernel-loss}
		f(\B{w})=\frac{1}{2}(\B{w}-\bar{\B{w}})^T\KX(\B{w}-\bar{\B{w}})-\B{w}^T\B{\epsilon}.  
	\end{equation}
\end{theorem}
As we study indefinite kernels, $\KX$ might be indefinite and therefore $f(\B{w})$ in Theorem~\ref{thm:kernel-quad} can be non-convex. This shows that our analysis of SAM applies to both convex (as in the linear regression case discussed in Section~\ref{sec:def-nn} and PSD kernels) and non-convex functions, as for indefinite kernels.

\section{Main Results}\label{sec:mainres}


\subsection{ReLU Networks}
In this section, we review our theoretical results for the ReLU networks discussed in Section~\ref{sec:def-nn}. We note that as discussed, this model also readily applies to the least-squares linear regression problem and therefore, we do not study that problem separately. Let $\B{A}\in\R^{n\times p}$ be the matrix with $i$-th row equal to $\B{a}(\B{w}_0;\B{x}_i)$. Let us consider the following Singular Value Decomposition (SVD) of $\B{A}$,
$
\B{A}=\B{V}\B{\Sigma}\B{U}^T = \begin{bmatrix}\B{V}_1 & \B{V}_2
\end{bmatrix}\begin{bmatrix}\B{\Sigma}_1 & \\ & \B{0}
\end{bmatrix}\begin
{bmatrix}\B{U}_1^T \\ \B{U}_2^T
\end{bmatrix}$
where $\B{\Sigma}_1\in\R^{r\times r}$ collects nonzero singular values of $\B{A}$ and $r$ is the rank of $\B{A}$. We let $\B{D}_1=\B{\Sigma}_1^2$. Theorem~\ref{thm-relu} characterizes the error for the neural model discussed in Section~\ref{sec:def-nn}.
\begin{theorem}\label{thm-relu}
Suppose $\B{w}_0=\B{U}_1\B{U}_1^T\B{w}_0$ and $\B{0}\prec \B{I}-\eta\B{D}_1-\eta\rho\B{D}_1^2\prec \B{I}$ and let $\B{u}=\B{U}_1^T(\bar{\B{w}}-\B{w}_0)$. Then, under the model from Section~\ref{sec:def-nn} one has for $k\leq \bar{k}$
\begin{equation}
	\begin{aligned}\label{linear-solution-thm1}
		\bias(\B{w}_k^{\sam})&=\frac{1}{n}\sum_{i=1}^r (1-\eta d_i-\eta\rho d_i^2)^{2k}d_i u_i^2 \\
		\var(\B{w}_k^{\sam})&=\frac{\sigma^2}{n}\tr\left( \left(\B{I}-(\B{I}-\eta\B{D}_1-\eta\rho\B{D}_1^2)^k\right)^2 \right).
	\end{aligned}
\end{equation}  
In particular, for $\bar{k}\geq k\geq 0$ one has $ \bias(\B{w}_k^{\sam})\leq  \bias(\B{w}_k^{\gd})$ and $\var(\B{w}_k^{\sam})\geq \var(\B{w}_k^{\gd})$.
\end{theorem}
We note that Theorem~\ref{thm-relu} is applicable to GD by setting $\rho=0$. Theorem~\ref{thm-relu} precisely characterizes the expected SAM trajectory, and its corresponding bias and variance terms for the neural network model. 
Specifically,
we see that bias for SAM in each iteration is smaller than GD, while the variance for SAM is larger. We note that as $k$ increases, the bias term decreases while the variance increases. Therefore, if the optimization is run for finitely many steps, the bias term is more likely to be the dominant term and as SAM has a lower bias, SAM is more likely to outperform GD. This intuitive argument is formalized in Proposition~\ref{lin-gen-prop}.

\begin{proposition}\label{lin-gen-prop}
Suppose there exists a numerical constant $c_0>1$ such that
\begin{equation}\label{lin-gen-prop-cond}
	\begin{aligned}
		1-\eta d_r  \leq c_0(1-\eta d_1 - \eta\rho d_1^2),~~
		1-\eta d_1 \geq \sqrt{c_0} (1-\eta d_r -\eta\rho d_r^2).
	\end{aligned}
\end{equation}
Let $\snr=\|\B{X}(\bar{\B{w}}-\B{w}_0)\|_2^2/r\sigma^2$ and assume $\snr\geq 1$. Under the assumptions of Theorem~\ref{thm-relu}, if
$$k\leq \frac{\log[2/(\snr + 1)]}{\log[(1-\eta d_1-\eta\rho d_1^2)^2/(1-\eta d_r -\eta\rho d_r^2)]}\land \bar{k}$$
one has
$\error(\B{w}_{k}^{\sam})\leq \error(\B{w}_{k}^{\gd})$.
\end{proposition}
Proposition~\ref{lin-gen-prop} shows that, assuming noise is not too large, SAM has a lower error compared to GD if the optimization is run for finitely many steps.  
\begin{remark}
As noted, in practice DNNs are trained for a limited number of epochs~\citep{early-deep}, and it is believed~\citep{undertrain1,undertrain2} recent large neural networks, specially language models, tend to be undertrained due to resources limitations. This shows the assumption that $k$ is finite is realistic.  
\end{remark}

\begin{remark}
An interesting special case of Theorem~\ref{thm-relu} is the noiseless case where $\sigma=0$. We note that Theorem~\ref{thm-relu} implies that SAM has a lower error than GD for all iterations $k\geq 1$ for this case.
\end{remark}
\begin{remark}
In Appendix~\ref{app:props}, we discuss the selection of $\eta,\rho$ to ensure condition~\eqref{lin-gen-prop-cond} holds. On a high level, condition~\eqref{lin-gen-prop-cond} suggests taking $\rho\geq \eta$ to take advantage of SAM performance.
\end{remark}
\begin{remark}
Proposition~\ref{lin-gen-prop} suggests that the total number of iterations should be smaller in noisy cases. As we demonstrate numerically in Section~\ref{sec:Numerical}, this is necessary to avoid overfitting to noise.
\end{remark}
\subsection{Kernel Regression}
Assume the eigenvalue decomposition $\KX=\B{U}\B{D}\B{U}^T$. For simplicity, we assume $\rank(\KX)=n$. We let $\B{U}_1,\B{D}_1$ and $\B{U}_2,\B{D}_2$ collect eigenvectors and eigenvalues of $\KX$ corresponding to positive and negative eigenvalues, respectively. We also let $\B{D}=\diag(d_1,\cdots,d_n)$ with $d_1\geq \cdots \geq d_n$.
\begin{theorem}\label{kernel-bias-var-thm}
Suppose $h_0^{\sam}=0$ and let $\B{u}=\B{U}^T{\bar{\B{w}}}$. Then, $\var(\B{w}_k^{\sam})=\var^+(\B{w}_k^{\sam})+\var^-(\B{w}_k^{\sam})$ where
\begin{equation}\label{kernel-thm-eq}
	\begin{aligned}
		\bias(\B{w}_k^{\sam})&=\frac{1}{n}\sum_{i=1}^n (1-\eta d_i-\eta\rho d_i^2)^{2k}d_i^2 u_i^2 \\
		\var^+(\B{w}_k^{\sam})&=\frac{\sigma^2}{n}\tr\left( \left(\B{I}-(\B{I}-\eta\B{D}_1-\eta\rho\B{D}_1^2)^k\right)^2 \right)\\
		\var^-(\B{w}_k^{\sam})&=\frac{\sigma^2}{n}\tr\left( \left(\B{I}-(\B{I}-\eta\B{D}_2-\eta\rho\B{D}_2^2)^k\right)^2 \right).
	\end{aligned}
\end{equation}  
\end{theorem}
In Theorem~\ref{kernel-bias-var-thm}, $\var^+,\var^-$ capture the variance from positive and negative eigenvalues of $\KX$, respectively. As we see, the behavior in a non-convex case where some eigenvalues are negative is wildly different from the case where all eigenvalues are non-negative. In particular, if $d_n<0$, not only $\bias(\B{w}_k^{\sam})\leq \bias(\B{w}_k^{\gd})$, but the GD bias actually diverges to infinity, while the SAM bias converges to zero under the assumptions of Theorem~\ref{kernel-bias-var-thm}. In terms of variance, we see that similarly, $\var^-$ for SAM stays bounded, while it can diverge to infinity for GD. This shows that in the indefinite setting, GD might have unbounded error in the limit of $k\to\infty$ while SAM can keep the error bounded. We also see $\var^+$ shows a behavior similar to the variance from the ReLU case (Theorem~\ref{thm-relu}), implying that when the number of iterations is limited, SAM has smaller error than GD. This shows that SAM is even more suited to the non-convex case as it performs well for both finite and infinite number of iterations. This explanation is formalized in Proposition~\ref{kernel-gen-prop}.

\begin{proposition}\label{kernel-gen-prop}
Suppose there exists a numerical constant $c_0>1$ such that
\begin{equation}\label{kernel-gen-prop-cond}
	\begin{aligned}
		1-\eta d_r & \leq c_0(1-\eta d_1 - \eta\rho d_1^2),~~ 
		1-\eta d_1 \geq \sqrt{c_0} (1-\eta d_r -\eta\rho d_r^2)
	\end{aligned}
\end{equation}
where $r$ is such that $d_r>0,d_{r+1}<0$. Moreover, assume there exists $\varepsilon>0$ that for $j\geq r+1$,
$$1-\eta d_j-\eta\rho d_j^2 \leq 1\leq 1+\varepsilon \leq 1-\eta d_j.$$
Let $\snr=\|\KX\bar{\B{w}}\|_2^2/r\sigma^2$ and assume $\snr\geq 1$. Then, under the assumptions of Theorem~\ref{kernel-bias-var-thm}, if
\begin{equation}\label{kernel-prop-cond-min}
	k\leq \frac{\log[2/(\snr + 1)]}{\log[(1-\eta d_1-\eta\rho d_1^2)^2/(1-\eta d_r -\eta\rho d_r^2)]}~~\text{AND}~~k\geq \frac{\log 2}{\log (1+\varepsilon)}
\end{equation}
one has
$\error(\B{w}_{k}^{\sam})\leq \error(\B{w}_{k}^{\gd})$. Moreover, if $d_n<0$,
$$\lim_{k\to\infty}\error(\B{w}_k^{\gd})=\infty,\lim_{k\to\infty}\error(\B{w}_k^{\sam})<\infty.$$
\end{proposition}
Similar to Proposition~\ref{lin-gen-prop}, Proposition~\ref{kernel-gen-prop} shows that when the total number of iterations is not too large, SAM performs better. Moreover, as discussed, SAM is able to keep the error bounded in the non-convex case, while GD's error diverges as $k\to\infty$. 
\subsection{SAM solutions are flat.}
As discussed, numerous papers have numerically observed the correlation between flatness and generalization, where flatter solutions tend to generalize better. In our work, we directly explained how SAM can perform better statistically compared to GD. However, one might ask if such a correlation between flatness and error exists in our setup. Here, we answer this question in the affirmative. Let us define the sharpness for SAM (and GD similarly) as the expected local fluctuations in the loss, 
\begin{equation}
\kappa^{\sam}_k=\max_{\|\B\varepsilon\|_2\leq \rho_0}\E_{\B{\epsilon}}[f(\E_{\B\epsilon}[\B{w}_k^{\sam}]+\B\varepsilon)-f(\E_{\B\epsilon}[\B{w}_k^{\sam}])]
\end{equation}
for some $\rho_0>0$ which might be different from $\rho$, and $f$ is given in~\eqref{quadratic-loss} for the ReLU case and in~\eqref{kernel-loss} for the kernel regression setup. Note that this can be considered as the expected value of sharpness defined by~\cite{sampaper}, $\max_{\|\B\varepsilon\|_2\leq \rho_0}f(\B{w}+\B\varepsilon)-f(\B{w})$ which motivates the SAM algorithm. 
\begin{proposition}\label{sharpness-prop}
\textbf{(1)} Under the assumptions of Theorem~\ref{thm-relu}, for $k\geq 1$
$$\kappa^{\gd}_k-\kappa^{\sam}_k\geq \rho_0^2\frac{d_r-d_1}{2}+\rho_0\left(\sqrt{\sum_{i=1}^r (1-\eta d_i)^{2k}d_i^2u_i^2}-\sqrt{\sum_{i=1}^r (1-\eta d_i-\eta\rho d_i^2)^{2k}d_i^2u_i^2}\right).$$
\noindent\textbf{(2)} Under the assumptions of Theorem~\ref{kernel-bias-var-thm}, if $d_n<0$
$$\lim_{k\to\infty}\kappa^{\gd}_k=\infty> \kappa_k^{\sam}~~\forall k\geq 1.$$
\end{proposition}
Proposition~\ref{sharpness-prop} shows that for the ReLU setup, SAM has lower sharpness compared to GD for sufficiently small $\rho_0,k$. Specially, if $d_r=d_1$, SAM has lower sharpness for $k,\rho_0>0$. Moreover, for the indefinite kernel setup, this proposition shows that GD has unbounded sharpness, unlike SAM. This further confirms the connections between generalization and flatness observed theoretically~\cite{ding2022flat} and numerically~\cite{sampaper,flatness4} in the literature. Moreover, this is in agreement with the previous work showing SAM leads to flatter solutions compared to GD~\cite{sambounce,neweigen,optsam1,optsam2,sampaper}.

\section{Numerical Experiments}\label{sec:Numerical}
We conduct various numerical experiments on linear models, kernel methods and deep neural networks to examine the theory we developed and to gain further insight into SAM. Due to space limitations, we only discuss main insights from our DNN experiments here. We use CIFAR10/100~\citep{cifar} data and noisy versions of CIFAR10 provided by~\cite{cifarn} to train ResNet50 network~\citep{resnets} in our experiments here. Additional results for ResNet18 network as well as experiments on linear/kernel models can be found in Appendix~\ref{app:numerical}.

\noindent\textbf{Large $\rho$ vs small $\rho$:} First, we consider the case with clean (noiseless) labels, where one can expect the bias to be the dominant term. In this case, our theory would suggest taking a larger $\rho$ lowers the error. Moreover, our theory anticipates SAM performs specially better than (S)GD in earlier epochs where the difference in bias is even larger. 

We show that these insights hold true in our experiments. Particularly, in Figure~\ref{fig:vs} [Two Left Panels] we observe that when $\rho>0$, SAM performs better than GD in almost all epochs. We see that as we increase $\rho$, SAM performs quite well over the first 150 epochs. However, the gains from large $\rho$ tend to fade in later epochs as smaller values of $\rho$ get to lower bias values as well. Nevertheless, we see that in terms of accuracy, it is better to choose a larger $\rho$ rather than a small $\rho$ (the accuracy values are given in figure legends. Also see Figure~\ref{fig:varyrho18} for more details). In case of CIFAR10, $\rho=0.1$ is the best value of $\rho$ in our experiments and taking $\rho=0.5$ results in a smaller loss of accuracy, compared to taking $\rho=0$ (i.e. GD). In the case of CIFAR100, $\rho=0.5$ results in better accuracy compared to $\rho=0.1$, which shows that generally, overestimating $\rho$ is less harmful than underestimating $\rho$. This mostly agrees with theory that taking larger $\rho$ in noiseless settings is better, although we note that in practice, variance might not be exactly zero so large $\rho$ might perform slightly worse than smaller $\rho$, as is the case for CIFAR10.

\noindent\textbf{Early stopping in noisy settings:} Next, we consider a noisy setting and show that to avoid overfitting, the training has to be stopped early, showing the assumption that the number of epochs is finite is realistic. We use two versions of noisy CIFAR10, \texttt{random label 1} and \texttt{worse labels} from~\cite{cifarn} which we call random and worse, respectively. The random version has about $17\%$ noise in training labels, while worse has about $40\%$ noise. The validation labels for both datasets are noiseless. As we see in Figure~\ref{fig:vs} [Two Right Panels], as the noise increases both methods tend to overfit in the later stages of training, and overfitting is stronger when noise is higher. This shows that in noisy settings training has to be stopped earlier as noise increases.

\noindent\textbf{Performance under noise:} As we see from Figure~\ref{fig:vs}, in noisy settings the gap between SAM and GD is even larger. This can be explained as in non-convex settings, GD can have unbounded variance (cf. Theorem~\ref{kernel-bias-var-thm}) which leads to worse performance of GD specially in later epochs. 

\noindent\textbf{Decaying $\rho$ helps:}  We observe that having large $\rho$ helps in initial phases of training, while having a smaller $\rho$ might help in the later phases. Therefore, we propose to start SAM with a large value of $\rho$ to decrease bias and then decay $\rho$ over the course of algorithm to limit the increase of variance (the details are discussed in Appendix~\ref{app:numerical}). The result for this case are shown in Table~\ref{table-1}. Full is the accuracy at the end of training (epoch 200) and Early corresponds to early stopping (epoch 120 for SGD and 50 for SAM-based methods). As can be seen, starting with larger than optimal $\rho$ and decaying leads to accuracy results similar or slightly better than using the optimal fixed $\rho$. Interestingly, using large $\rho$ leads to considerably better performance if training has to be stopped early, which is often the case in practice specially for large models~\cite{undertrain1,undertrain2} due to resource limitations.

\begin{figure}[t!]
\centering
\begin{tabular}{cccc}
	CIFAR10 & CIFAR100 & CIFAR10-Random & CIFAR10-Worse\\
	\includegraphics[width=0.21\linewidth,trim =0.8cm 0cm .8cm 0cm, clip = true]{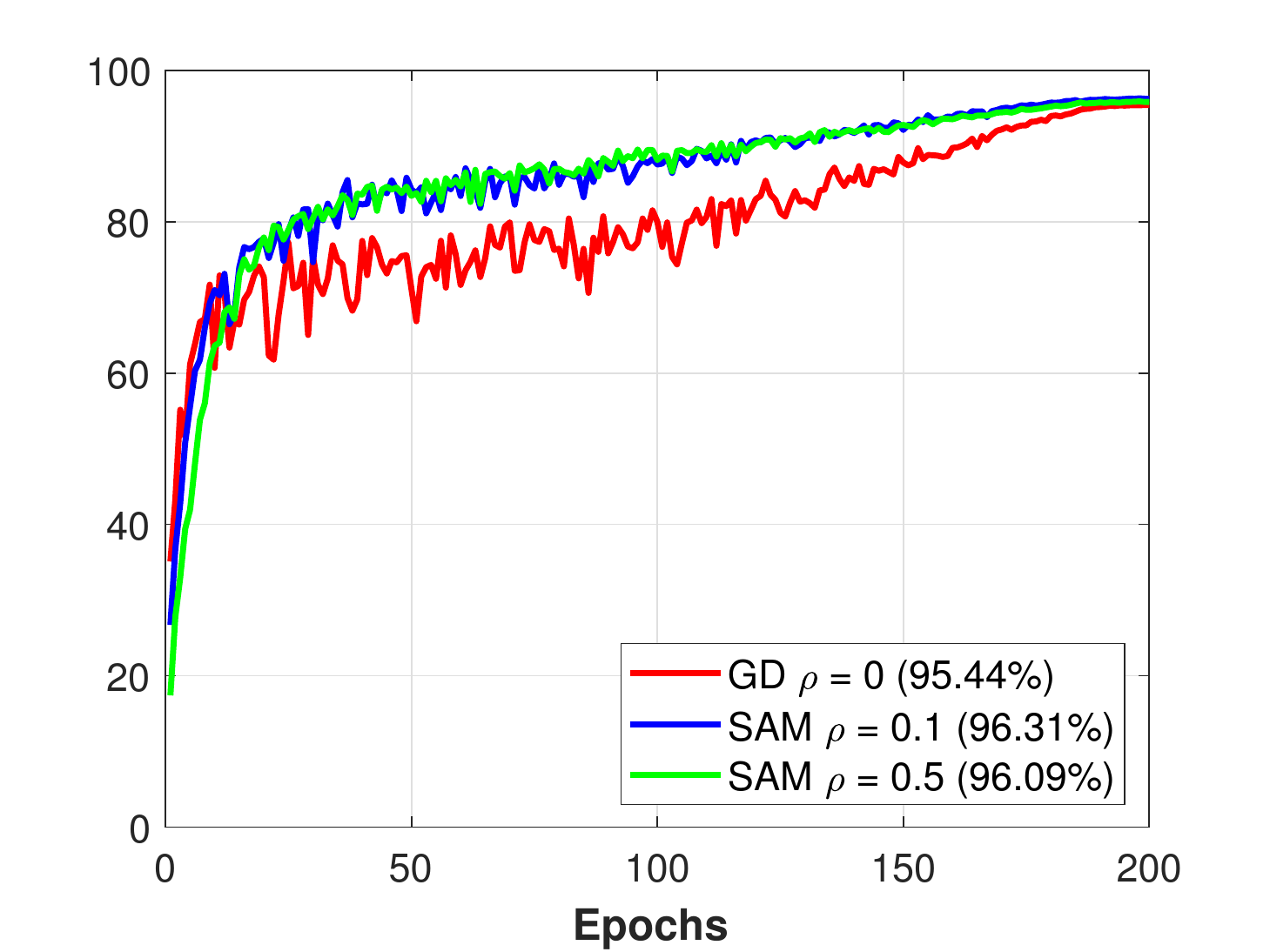}&   
	\includegraphics[width=0.21\linewidth,trim =.8cm 0cm .8cm 0cm, clip = true]{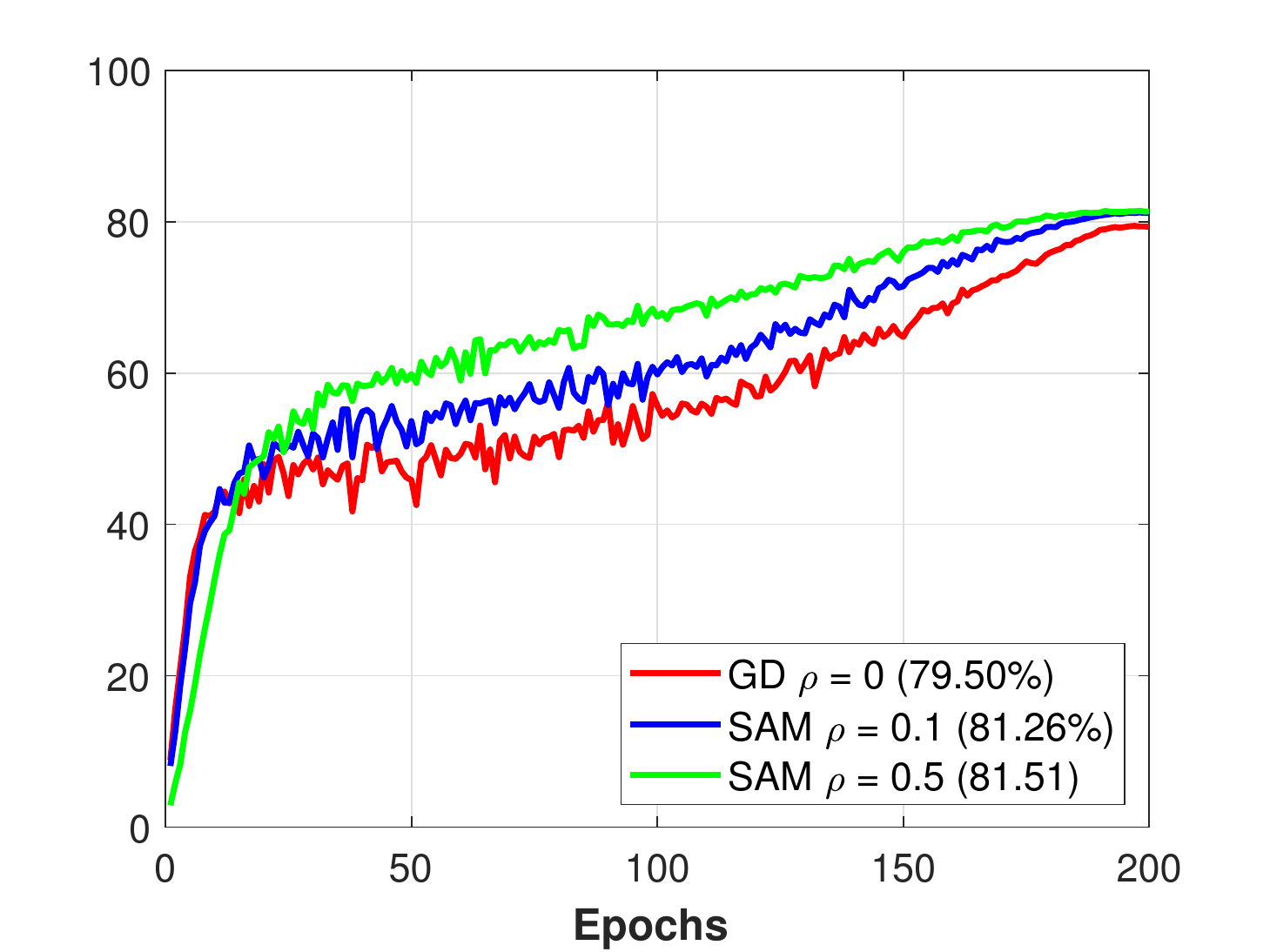}  &
	\includegraphics[width=0.21\linewidth,trim =0.8cm 0cm .8cm 0cm, clip = true]{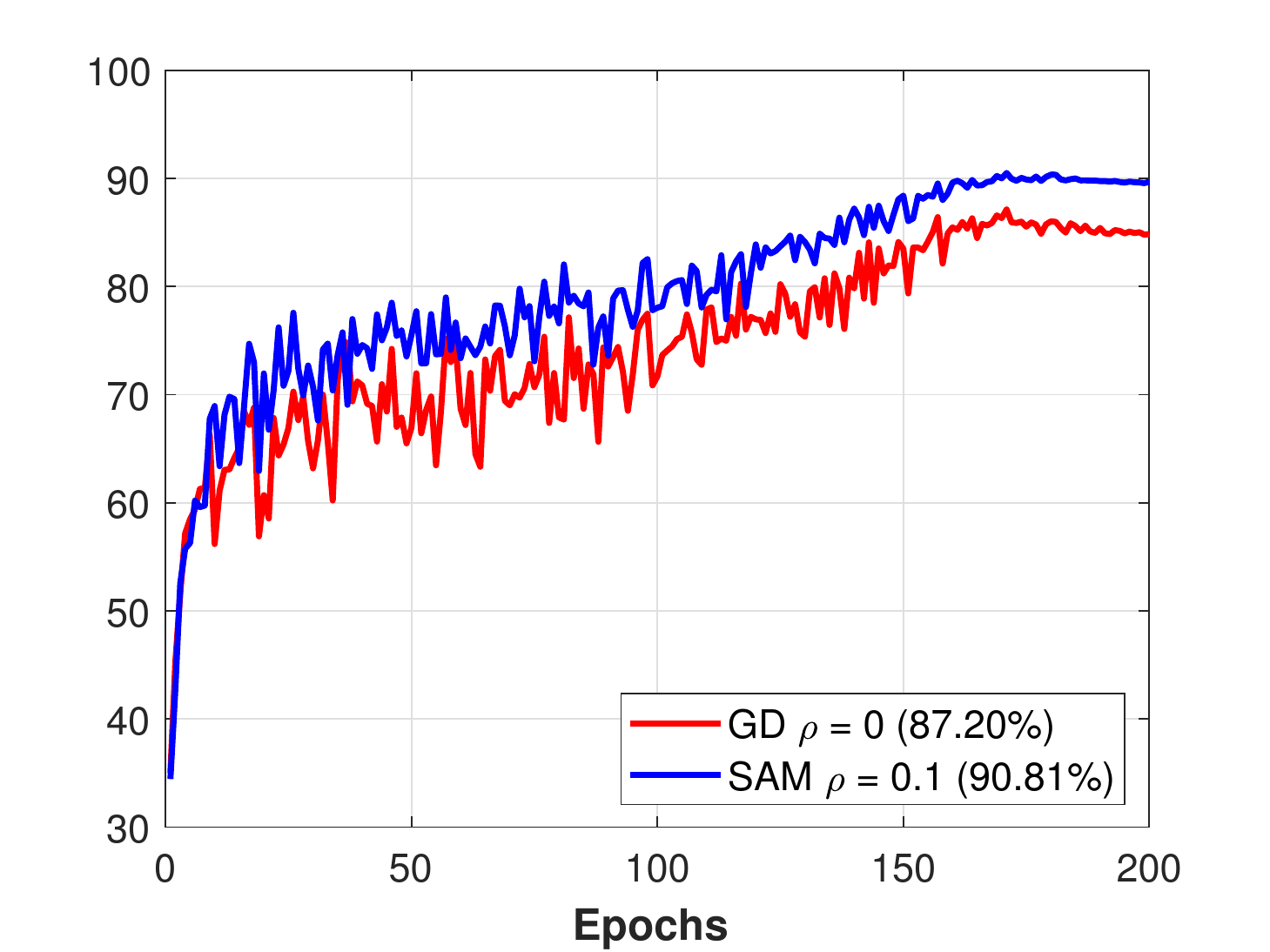}&   
	\includegraphics[width=0.21\linewidth,trim =.8cm 0cm .8cm 0cm, clip = true]{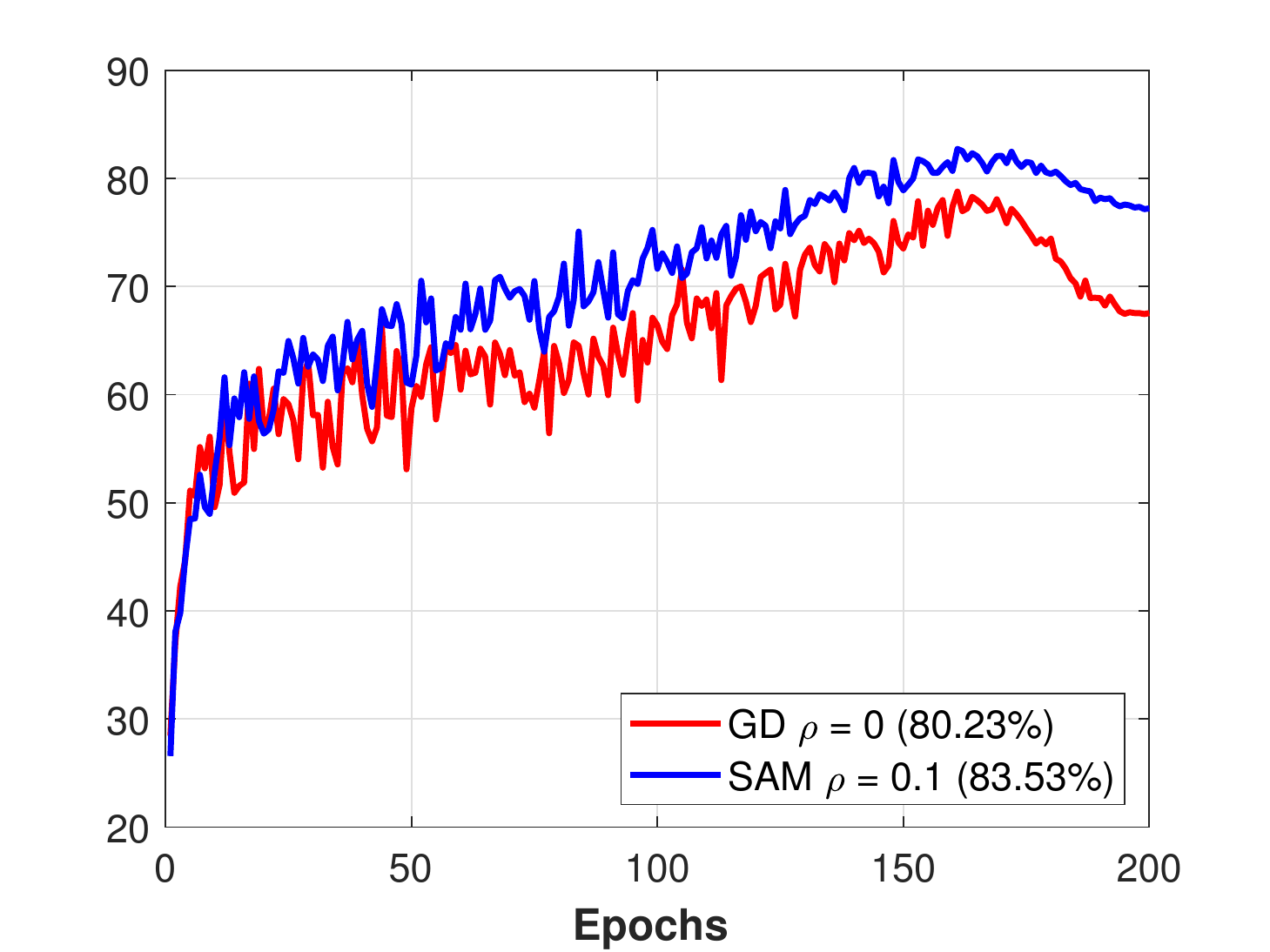}   
\end{tabular}
\caption{\small  Accuracy over epochs for SAM and GD with ResNet50 and different datasets. The number in the parenthesis in the legend shows the average best accuracy.}
\label{fig:vs}
\end{figure}

\begin{table}[t]
\caption{\small Comparison of SGD, SAM with optimal $\rho$ and SAM with $\rho$ decaying over the course of algorithm.}
\label{table-1}
\vskip 0.15in
\begin{center}
	\begin{footnotesize}
		\begin{sc}
			\begin{tabular}{cccc}
				\toprule
				Dataset & Method & Full  & Early    \\
				\midrule
				\multirow{3}{*}{CIFAR10} & SGD & $95.44\pm0.06$ & $82.99\pm0.75$ \\
				& SAM & $96.31\pm0.06$ & $81.43\pm2.73$ \\
				&  SAM-Decay & $96.42\pm0.10$ & $86.79\pm0.38$ \\
				\multirow{3}{*}{CIFAR100} & SGD & $79.50\pm0.33$ & $58.87\pm0.62$ \\
				& SAM & $82.01\pm0.09$ & $60.20\pm0.97$ \\
				&  SAM-Decay & $82.02\pm0.27$  & $61.92\pm1.63$\\
				\bottomrule
			\end{tabular}
		\end{sc}
	\end{footnotesize}
\end{center}
\vskip -0.1in
\end{table}

\section{Conclusion and Future Work}
We presented a direct explanation of why SAM generalizes well through 
studying the statistical performance of SAM/GD for two classes of problems. Specifically, we showed that SAM works well for neural networks with a hidden ReLU layer if the noise is not too high. We also showed that in indefinite kernel regression, corresponding to a non-convex optimization problem, SAM can have bounded error while GD has unbounded error.  
An interesting question is that how stochastic version of SAM would differ from the full-batch setting studied here. In Appendix~\ref{sec:stochastic}, we study a stochastic version of SAM and compare it to SGD for a special case. As we see, SAM tends to benefit from stochasticity even more, specially in high-dimensional settings. A deeper analysis of stochastic SAM is left for a future work.

	\section*{Acknowledgments}
	This research is supported in part by a grant from the Office of Naval Research (N000142112841). Authors would like to thank MIT SuperCloud for providing computational resources for this work.

	\newpage
	
	\bibliographystyle{plainnat}
	\bibliography{reff}

	\appendix
	\numberwithin{equation}{section}
	\numberwithin{lemma}{section}
	\numberwithin{proposition}{section}
	\numberwithin{figure}{section}
	\numberwithin{table}{section}

	

	\newpage
	
\section{Details of example from Section~\ref{sec:intro}}\label{app:toy}
Under the notation used in Section~\ref{sec:mainres}, the noiseless example shows a one-dimensional case with $d_1=1,\eta=0.015, u=1$ and $\rho=1$. The plot shows $\error$ for GD and SAM. In the noisy setting, we set $n=2,\sigma^2=0.2$. The model follows $\eta=0.0045, d_2=-0.0007/\eta$ and $\rho=-1/d_2$. We also set $u_1=u_2=1$. The plot similarly shows the error.

\section{Review of Kernel Gradients}\label{app:kernels}
Note that as discussed in Section~\ref{sec:def-kernel}, $\Hk=\Hkp\bigoplus \Hkn$. Therefore, for $f\in\Hk$, there exists $f_+\in\Hkp,f_-\in\Hkn$ such that $f=f_+-f_-$. Moreover, the inner product in $\Hk$ is defined as
\begin{equation}
	\langle f,g\rangle = \langle f_+,g_+\rangle - \langle f_-,g_-\rangle.
\end{equation}
Note that similar to the RKHS case, for $f\in\Hk$ we have
$$\langle K(\B{x},\cdot),f\rangle = \langle K_+(\B{x},\cdot),f_+\rangle-\langle K_-(\B{x},\cdot),f_-\rangle=f_+(\B{x})-f_-(\B{x})=f(\B{x}).$$
Let $\B{x}\in\R^p,y\in\R$ and $L[h]=(h(\B{x})-y)^2$ for $h\in\Hk$. The gradient of $L[h]$ is a function such as $k\in\Hk$ where $k$ is a good first-order approximation to $L[h]$. In particular, for any bounded $g\in\Hk$,
\begin{align}
	L[h+\epsilon g] & = (h(\B{x})+\epsilon g(\B{x})-y)^2 \nonumber \\
	& = (h(\B{x})-y)^2 + \epsilon^2 g(\B{x})^2 + 2\epsilon g(\B{x})\left(h(\B{x})-y\right) \nonumber \\
	& = L[h] + 2\epsilon g(\B{x})\left(h(\B{x})-y\right) + \mathcal{O}(\epsilon^2) \nonumber \\
	& = L[h] + \epsilon \langle2\left(h(\B{x})-y\right)K(\B{x},\cdot),g\rangle + \mathcal{O}(\epsilon^2).
\end{align}
Therefore, 
\begin{equation}
	\lim_{\epsilon\to 0}\frac{L[h+\epsilon g]-L[h]}{\epsilon}=\langle2\left(h(\B{x})-y\right)K(\B{x},\cdot),g\rangle.
\end{equation}
Hence, we take $\nabla L[h]=2(h(\B{x})-y)K(\B{x},\cdot).$

\section{Discussion on Propositions~\ref{lin-gen-prop} and~\ref{kernel-gen-prop}}\label{app:props}
In this section, we study what conditions~\eqref{kernel-gen-prop-cond} implies on the model. Particularly, we set $d_1=1$. As two examples, we take $d_r\in\{0.8,0.95\}$ and $d_n\in\{-0.6,-1\}$. We also like the bounds of Proposition~\ref{kernel-gen-prop} to be valid for at least $k\geq 20$. Therefore, we take $\varepsilon = \log 2/\log 20 -1$. Next, we sweep $\eta$ and $\rho$ and choose the values that satisfy~\eqref{kernel-gen-prop-cond} for some $c_0>1$. We plot the results in Figure~\ref{fig:app-comp-eta} for different values of $d_r,d_n$, where we highlight every pair of $(\eta,\rho)$ that satisfy the condition in dark blue. As can be seen in this figure, taking $\eta$ to be small and $\rho\gg \eta$ results in~\eqref{kernel-gen-prop-cond} being satisfied. This makes intuitive sense as taking $\eta$ small helps to satisfy $1-\eta d_i-\eta\rho d_i^2< 1$ and taking $\rho\gg \eta$ helps to take advantage of SAM regularization.

\begin{figure}[t!]
	\centering
	\begin{tabular}{ccc}
		& $d_r=0.8$ & $d_r=0.95$ \\
		\rotatebox{90}{~~~~~~~~~~~~~~~~~~~~$d_n=-0.6$}  & \includegraphics[width=0.4\linewidth,trim =0cm 0cm 0cm 0cm, clip = true]{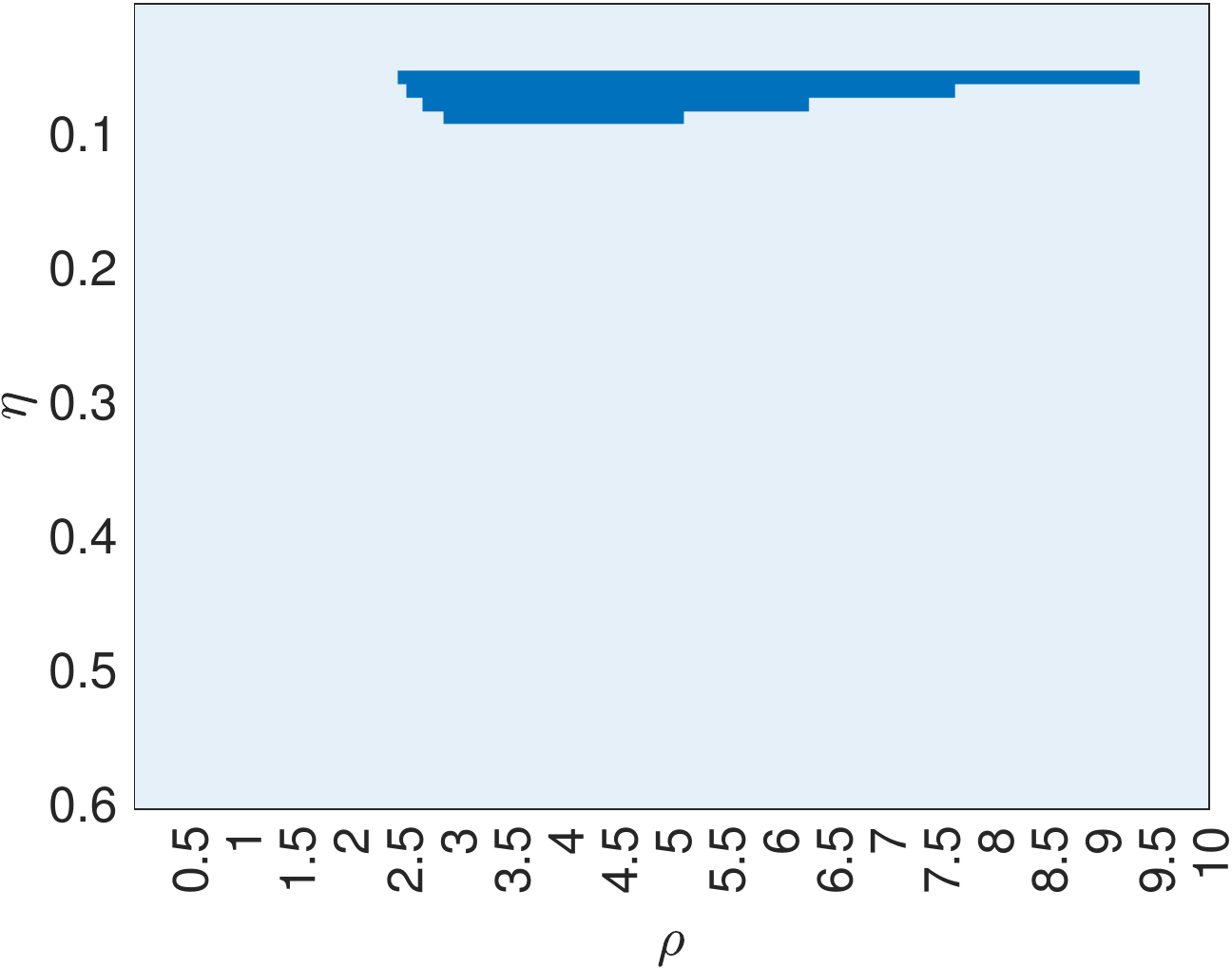}&
		\includegraphics[width=0.4\linewidth,trim =0cm 0cm 0cm 0cm, clip = true]{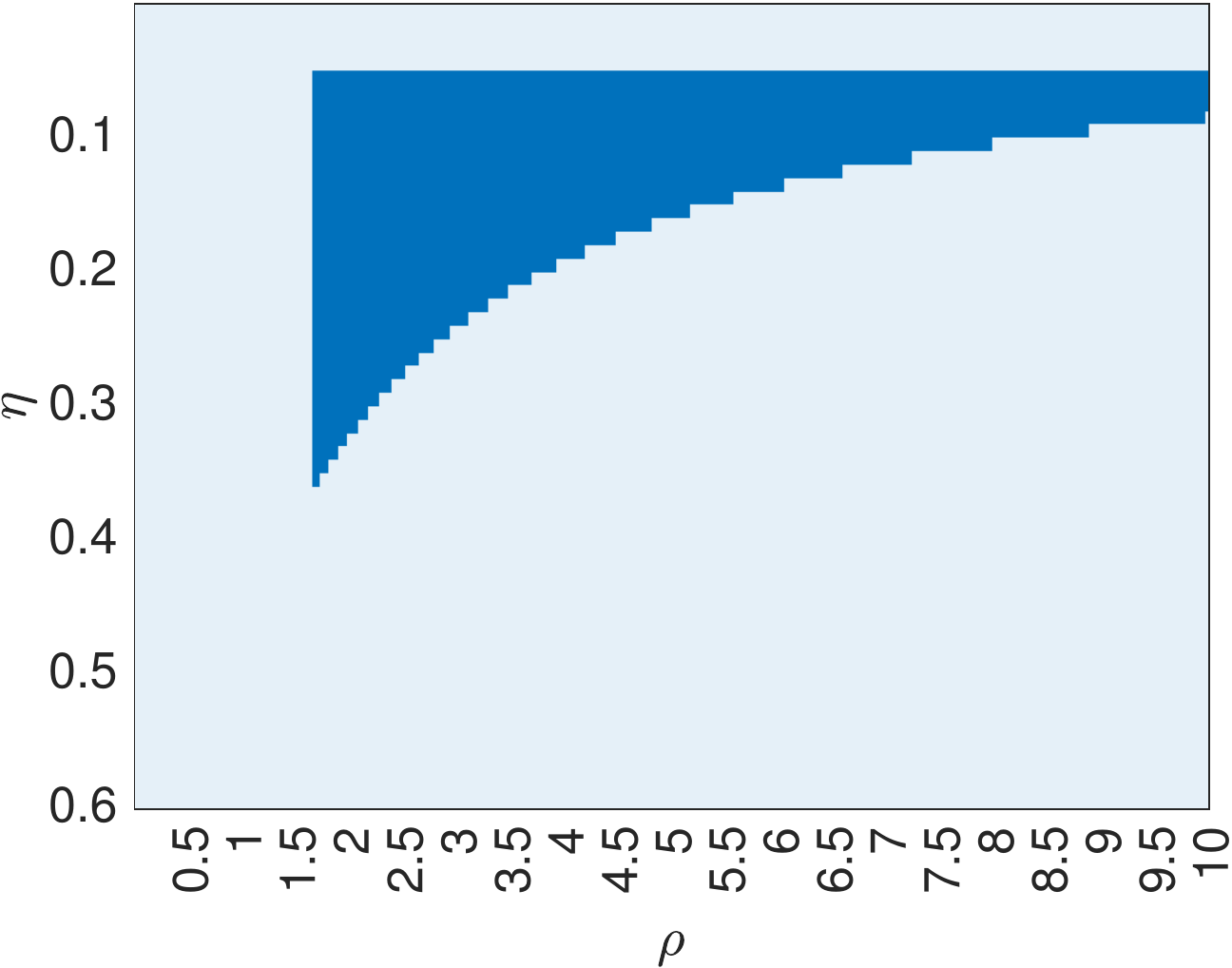}\\
		\rotatebox{90}{~~~~~~~~~~~~~~~~~~~~~~$d_n=-1$} & \includegraphics[width=0.4\linewidth,trim =0cm 0cm 0cm 0cm, clip = true]{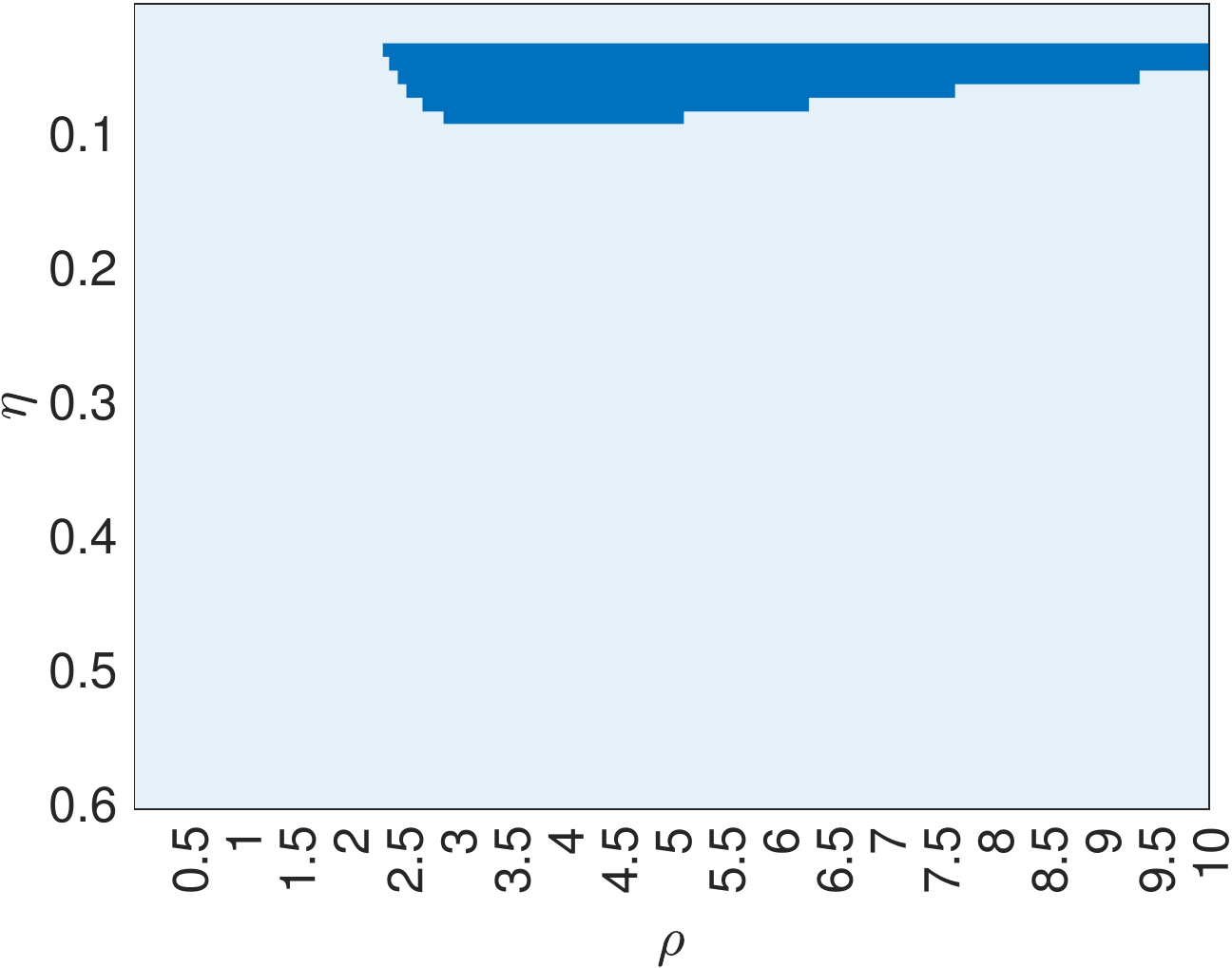}&
		\includegraphics[width=0.4\linewidth,trim =0cm 0cm 0cm 0cm, clip = true]{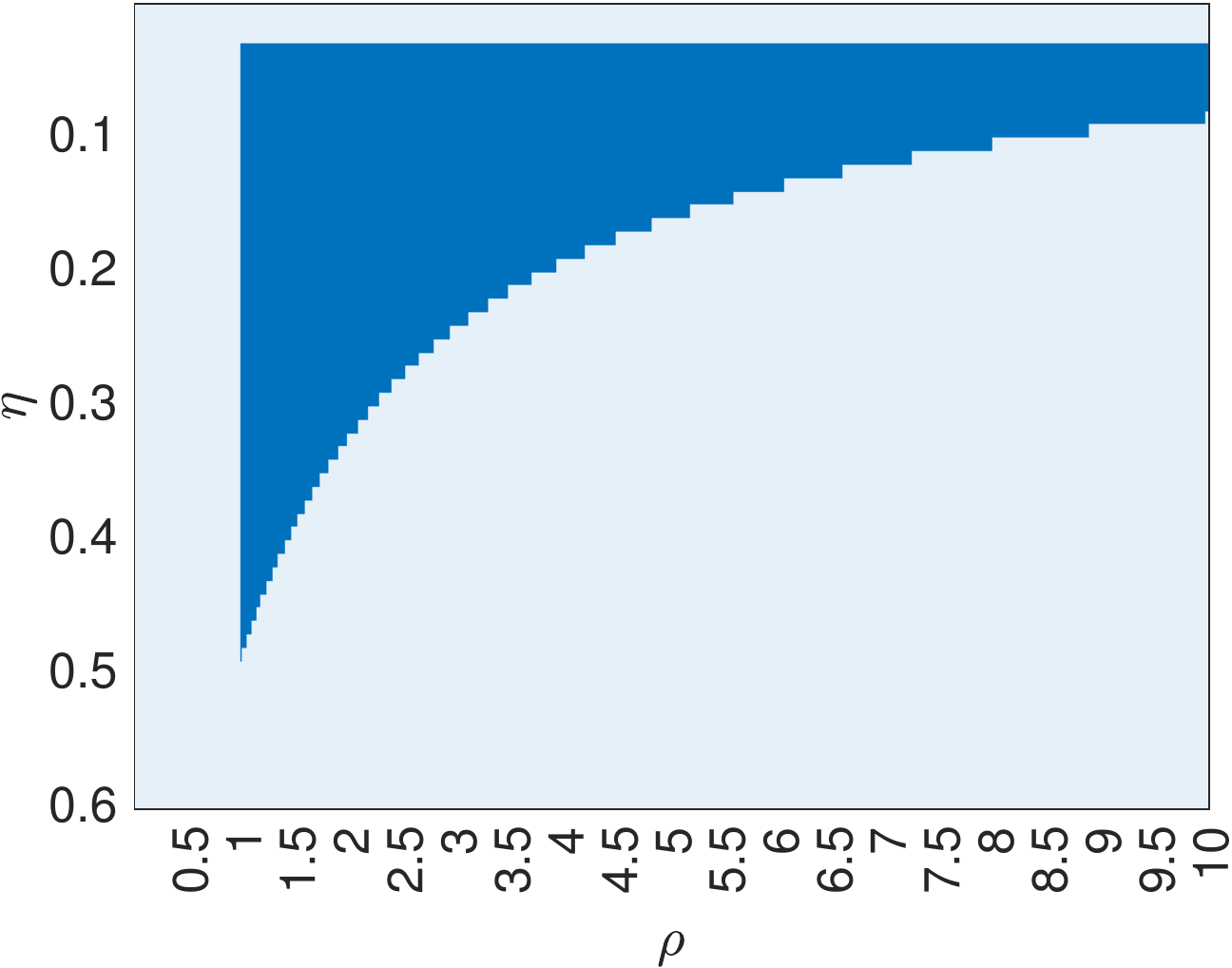} 
	\end{tabular}
	\caption{\small  Comparison of values of $(\eta,\rho)$ that satisfy condition~\eqref{kernel-gen-prop-cond} for different values of $d_r,d_n$. See Appendix~\ref{app:props} for more details. }
	\label{fig:app-comp-eta}
\end{figure}

\section{Effect of Stochasticity: A Special Case}\label{sec:stochastic}
In this section, we study SAM when stochastic mini-batches are used and discuss how stochasticity helps SAM's performance. To this end, we limit our analysis to the linear regression case, and assume for $k\geq 1$, $y_k=\bar{\B{w}}^T\B{x}_k+\epsilon_i$, where $\epsilon_i$'s are iid noise values as before, and $\B{x}_k$'s are independent of each other and noise. In fact, we assume for $k\geq 1$, $\B{x}_k\sim\cN(\B{0},\B{I})$ follows a normal distribution. The loss corresponding to the point $k$ is defined as 
$$f_k(\B{w})=\frac{1}{2}(y_k-\B{x}_k^T\B{w})^2.$$
The stochastic versions of the algorithms hence follow 
\begin{equation}
	\B{w}_{k+1}^{\sam}=\B{w}_k^{\sam}-\eta\nabla f_k(\B{w}_k^{\sam}+\rho\nabla f_k(\B{w}_k^{\sam})).
\end{equation}
To better understand stochastic SAM, following the recent work on SGD~\citep{smith2021origin}, we consider the expected trajectory the algorithm takes, that is, $\E[{\B{w}}_k^{\sam}]$ where the expectation is taken over $\B{x}_i,\epsilon_i$ for $i\geq 1$. Moreover, as the observations are random, we define error over an unseen data point, which follows the same distribution as the training data, i.e., we consider a random design. Specifically, we define 
\begin{equation}\label{error-stochastic}
	\error(\B{w}_k^{\sam}) = \E_{\B{x}_0,\epsilon_0}\left[\left(\B{x}_0^T(\E[\B{w}_k^{\sam}]-\bar{\B{w}})-\epsilon_0\right)^2\right]
\end{equation}
where $\B{x}_0,\epsilon_0$ follow the same distribution as $\B{x}_k,\epsilon_k$ and are independent.

\begin{proposition}\label{example-iid}
	Suppose $0<1-\eta-\eta\rho(p+2)\leq 1-\eta<1$. Then, under the stochastic setup, 
	\begin{equation}\label{iid-prop-eq}
		\error(\B{w}_k^{\sam})-\error(\B{w}_k^{\gd})=\left[(1-\eta-\eta\rho(p+2))^{2k}-(1-\eta)^{2k}\right]\|\bar{\B{w}}\|_2^2\leq 0.
	\end{equation}
\end{proposition}
Note that Proposition~\ref{example-iid} shows that SAM outperforms GD in the stochastic setup under this linear regression setup. The term $(p+2)\eta\rho$ appearing in~\eqref{iid-prop-eq} is an effect of stochasticity, and this term in the full-batch setting is expected to be $\eta\rho$ (cf. Theorem~\ref{thm-relu} when $\B{D}_1=\B{I}$). In the high-dimensional setting where $p\gg 1$, this additional term resulting from stochasticity improves the SAM error significantly, showing the suitability of SAM for both stochastic and high-dimensional settings. A deeper analysis of stochastic SAM on more complex model is left for future work.

\section{Proof of Main Results}

\subsection{A Preliminary Result}
\begin{theorem}\label{thm-prelim}
	Let $\B{w}_k^{\sam}$ follow~\eqref{sam-update} with
	$$f(\B{w})=\frac{1}{2}(\B{w}-\bar{\B{w}})^T\B{H}(\B{w}-\bar{\B{w}})+\B{g}^T(\B{w}-\bar{\B{w}}).$$
	Then,
	\begin{equation}
		\B{w}_{k+1}^{\sam}  = \eta \sum_{i=0}^k (\B{I} - \eta \B{H} -\eta\rho\B{H}^2)^i (\B{I}+\rho \B{H})(\B{H}\bar{\B{w}}-\B{g})+(\B{I} - \eta \B{H} -\eta\rho\B{H}^2)^{k+1}\B{w}_0.\label{thm1-iterations-1}
	\end{equation}
\end{theorem}
\begin{proof}
	As 
	$$f(\B{w})=\frac{1}{2}(\B{w}-\bar{\B{w}})^T\B{H}(\B{w}-\bar{\B{w}})+\B{g}^T(\B{w}-\bar{\B{w}})$$
	we have
	$$\B{w}_k^{\sam}+\rho\nabla f(\B{w}_k^{\sam})=\B{w}_k^{\sam}+\rho\B{g}+\rho\B{H}(\B{w}_k^{\sam}-\bar{\B{w}})=(\B{I}+\rho\B{H})\B{w}_k^{\sam}+\rho(\B{g}-\B{H}\bar{\B{w}}).$$
	Therefore, by writing SAM updates:
	\begin{align}
		\B{w}_{k+1}^{\sam} &=  \B{w}_{k}^{\sam} - \eta \nabla f( \B{w}_k^{\sam}+\rho\nabla f( \B{w}_k^{\sam}))\nonumber \\
		& = \B{w}_{k}^{\sam} - \eta(\B{g} + \B{H}( \B{w}_k^{\sam}+\rho\nabla f( \B{w}_k^{\sam})-\bar{\B{w}})) \nonumber \\
		& = \B{w}_{k}^{\sam} - \eta\left\{\B{g} + \B{H}\left( (\B{I}+\rho\B{H})\B{w}_k^{\sam} + \rho(\B{g}-\B{H}\bar{\B{w}})-\bar{\B{w}}\right)\right\} \nonumber \\
		& = (\B{I} - \eta \B{H} -\eta\rho\B{H}^2) \B{w}_k^{\sam}+ \eta (\B{I}+\rho \B{H})(\B{H}\bar{\B{w}}-\B{g}) \label{thm-prelim-helper} \\
		& = \eta \sum_{i=0}^k (\B{I} - \eta \B{H} -\eta\rho\B{H}^2)^i (\B{I}+\rho \B{H})(\B{H}\bar{\B{w}}-\B{g})+(\B{I} - \eta \B{H} -\eta\rho\B{H}^2)^{k+1}\B{w}_0
	\end{align}
	where the last equality is a result of an inductive argument.
\end{proof}

\subsection{Proof of Theorem~\ref{representerthm}}

\begin{proof}
	By the definition,
	\begin{equation}
		L[h] = \frac{1}{2}\sum_{i=1}^n (y_i-h(\B{x}_i))^2
	\end{equation}
	and
	\begin{equation}
		\nabla L[h] = \sum_{i=1}^n (h(\B{x}_i) -y_i)K(\B{x}_i,\cdot).
	\end{equation}
	Hence, the KernelSAM gradient can be written as
	\begin{align}
		\nabla L[h+\rho\nabla L[h]] &=  \sum_{i=1}^n \left[(h+\rho\nabla L[h])(\B{x}_i)-y_i\right]K(\B{x}_i,\cdot) \nonumber \\
		& = \sum_{i=1}^n \left[\left(h+\rho \sum_{j=1}^n (h(\B{x}_j) -y_j)K(\B{x}_j,\cdot)\right)(\B{x}_i)-y_i\right]K(\B{x}_i,\cdot)\nonumber \\
		& \stackrel{(a)}{=}\sum_{i=1}^n \left[ h(\B{x}_i)+\rho \sum_{j=1}^n (h(\B{x}_j)-y_j)K(\B{x}_i,\B{x}_j)-y_i\right]K(\B{x}_i,\cdot) \nonumber\\ 
		& = \sum_{i=1}^n \left[h(\B{x}_i)-y_i\right]K(\B{x}_i,\cdot) + \rho \sum_{i=1}^n \sum_{j=1}^n \left[(h(\B{x}_j)-y_j)K(\B{x}_i,\B{x}_j)\right]K(\B{x}_i,\cdot)\label{rep-thm-helper1}
	\end{align}
	where in $(a)$, we used the fact $K(\B{x}_i,\cdot)(\B{x}_j)=K(\B{x}_i,\B{x}_j)$. As a result, we have
	$$ \nabla L[h+\rho\nabla L[h]]=\B{v}(h)^T\Kdot$$
	where $\B{v}(h)\in\R^n$, and
	$$v_i(h)= h(\B{x}_i)-y_i + \rho\sum_{j=1}^n \left[(h(\B{x}_j)-y_j)K(\B{x}_i,\B{x}_j)\right].$$
	Note that $h_0^{\sam}=0=\B{0}^T\Kdot$. Suppose for the sake of induction that $h_k^{\sam}=(\B{w}_k^{\sam})^T\Kdot$. Then, \begin{align}
		h_{k+1}^{\sam}&  = h_{k}^{\sam}-\eta \nabla L[h_k^{\sam}+\rho\nabla L[h_k^{\sam}]] \nonumber \\
		& = \left(\B{w}_k^{\sam}-\eta \B{v}(h_k^{\sam})\right)^T\Kdot \nonumber \\
		& = (\B{w}_{k+1}^{\sam})^T\Kdot
	\end{align}
	where
	\begin{equation}\label{beta-representer}
		\B{w}_{k+1}^{\sam}=\B{w}_k^{\sam}-\eta \B{v}(h_k^{\sam}).
	\end{equation}
	This completes the proof.
	
\end{proof}

\subsection{Proof of Theorem~\ref{thm:kernel-quad}}
The proof of this theorem is based on the following technical lemma. 
\begin{lemma}\label{thm4-helper-lemma}
	Under the assumptions of Theorem~\ref{thm:kernel-quad}, one has $h_k^{\sam}=(\B{w}_k^{\sam})^T\Kdot$ where 
	$$\B{w}_{k+1}^{\sam} = (\B{I}-\eta \KX -\eta\rho\KX^2)\B{w}_k^{\sam}+\eta(\B{I}+\rho\KX)(\KX\bar{\B{w}}+\B{\epsilon}).$$
\end{lemma}
\begin{proof}
	Suppose
	$$h=\sum_{i=1}^n w_{i}K(\B{x}_i,\cdot).$$
	Then, the residual corresponding to the $i$-th observation is given as
	\begin{align}
		R_{i}& =h(\B{x}_i)-y_i \nonumber\\
		& = h(\B{x}_i)-\bar{h}(\B{x}_i)-\epsilon_i \nonumber\\
		& = \sum_{j=1}^n(w_j-\bar{w}_j)K(\B{x}_i,\B{x}_j)-\epsilon_i
	\end{align}
	or in the matrix/vector notation,
	\begin{equation}
		\B{R}=\KX(\B{w}-\bar{\B{w}})-\B{\epsilon}.
	\end{equation}
	Thus,
	\begin{align}
		\sum_{i=1}^n \left[h(\B{x}_i)-y_i\right]K(\B{x}_i,\cdot)&=  \sum_{i=1}^n R_i K(\B{x}_i,\cdot)\nonumber\\
		& = \sum_{i=1}^n\left(\sum_{j=1}^n(w_j-\bar{w}_j)K(\B{x}_i,\B{x}_j)-\epsilon_i\right) K(\B{x}_i,\cdot)\nonumber \\
		& = \sum_{i=1}^n\sum_{j=1}^n (w_j-\bar{w}_j)K(\B{x}_i,\B{x}_j)K(\B{x}_i,\cdot)-\sum_{i=1}^n\epsilon_iK(\B{x}_i,\cdot)\nonumber \\
		& = (\B{w}-\bar{\B{w}})^T\KX\Kdot-\B{\epsilon}^T\Kdot
	\end{align}
	using our vector notation. Next,
	\begin{align}
		\sum_{j=1}^n \left[(h(\B{x}_j)-y_j)K(\B{x}_i,\B{x}_j)\right]&=\sum_{j=1}^n \left[\sum_{l=1}^n(w_l-\bar{w}_l)K(\B{x}_j,\B{x}_l)-\epsilon_j\right]K(\B{x}_i,\B{x}_j)\nonumber \\
		& = \sum_{l=1}^n (w_l-\bar{w}_l)\sum_{j=1}^n K(\B{x}_i,\B{x}_j)K(\B{x}_j,\B{x}_l)-\sum_{j=1}^n K(\B{x}_i,\B{x}_j)\epsilon_j\nonumber\\
		& = \sum_{l=1}^n(\KX^2)_{i,l}(w_l-\bar{w}_l) - [\KX\B{\epsilon}]_i\nonumber \\
		& = \left[\KX^2 (\B{w}-\bar{\B{w}})\right]_i- [\KX\B{\epsilon}]_i.
	\end{align}
	This leads to
	\begin{align}
		\sum_{i=1}^n \sum_{j=1}^n \left[(h(\B{x}_j)-y_j)K(\B{x}_i,\B{x}_j)\right]K(\B{x}_i,\cdot) = (\B{w}-\bar{\B{w}})^T\KX^2\Kdot-\B{\epsilon}^T\KX\Kdot.
	\end{align}
	Therefore, from~\eqref{rep-thm-helper1},
	\begin{align}
		\nabla L[h+\rho\nabla L[h]] = \left(\KX(\B{w}-\bar{\B{w}})-\B{\epsilon}+\rho  \KX^2(\B{w}-\bar{\B{w}})-\rho\KX \B{\epsilon}\right)^T\Kdot.
	\end{align}
	In particular, if $h_k^{\sam}= (\B{w}_k^{\sam})^T\Kdot$, then $h_{k+1}^{\sam}= (\B{w}_{k+1}^{\sam})^T\Kdot$ where
	\begin{align}
		\B{w}_{k+1}^{\sam} &= \B{w}_{k}^{\sam} - \eta\left(\KX (\B{w}_k^{\sam}-\bar{\B{w}})-\B{\epsilon}+\rho  \KX^2(\B{w}_k^{\sam}-\bar{\B{w}})-\rho\KX \B{\epsilon}\right) \nonumber \\
		& = (\B{I}-\eta \KX -\eta\rho\KX^2)\B{w}_k^{\sam}+\eta(\B{I}+\rho\KX)(\KX\bar{\B{w}}+\B{\epsilon}).
	\end{align}
\end{proof}
\begin{proof}[\textbf{Proof of Theorem~\ref{thm:kernel-quad}}]
	The proof follows from comparing Lemma~\ref{thm4-helper-lemma} to~\eqref{thm-prelim-helper}.
\end{proof}

\subsection{Proof of Theorem~\ref{thm-relu}}
\begin{proof}
	Note that under the setup, 
	\begin{align}
		f(\B{w}_k^{\sam})&=\frac{1}{2}\sum_{i=1}^n(y_i-\Phi(\B{w}_k^{\sam};\B{x}_i))^2\nonumber \\
		& \stackrel{(a)}{=} \frac{1}{2}\sum_{i=1}^n\left(\epsilon_i+\B{a}(\bar{\B{w}};\B{x}_i)^T\bar{\B{w}}-\B{a}(\B{w}_k^{\sam};\B{x}_i)^T\B{w}_k^{\sam}\right)^2\nonumber \\
		& \stackrel{(b)}{=} \frac{1}{2} \|\B{\epsilon}+\B{A}(\bar{\B{w}}-\B{w}_k^{\sam})\|_2^2\nonumber\\
		& = \frac{1}{2}(\bar{\B{w}}-\B{w}_k^{\sam})^T\B{A}^T\B{A}(\bar{\B{w}}-\B{w}_k^{\sam}) -\B{\epsilon}^T\B{A}(\B{w}_k^{\sam}-\bar{\B{w}}) + \frac{1}{2}\|\B{\epsilon}\|_2^2\label{thm4-helper-new-0}
	\end{align}
	where $(a)$ is by Assumption~\ref{assumption-underlying} and $(b)$ is by Assumption~\ref{assumption-relu}.
	Comparing~\eqref{thm4-helper-new-0} to Theorem~\ref{thm-prelim}, we see that $\B{H}=\B{A}^T\B{A}$ and $\B{g}=-\B{A}^T\B{\epsilon}$.
	Next, note that
	\begin{align}
		& \sum_{i=0}^k (\B{I} - \eta \B{H} -\eta\rho\B{H}^2)^i (\B{I}+\rho \B{H})(\B{H}\bar{\B{w}}-\B{g}) \nonumber \\
		\stackrel{(a)}{=} & \B{U}\left(\sum_{i=0}^k(\B{I}-\eta\B{D}-\eta\rho\B{D}^2)^i\right)(\B{I}+\rho \B{D})\B{U}^T(\B{U}_1\B{D}_1\B{U}_1^T\bar{\B{w}}+\B{U}_1\B{\Sigma}_1\B{V}_1^T\B{\epsilon})  \nonumber \\
		\stackrel{(b)}{=} & \B{U}_1\left(\sum_{i=0}^k(\B{I}-\eta\B{D}_1-\eta\rho\B{D}_1^2)^i\right)(\B{I}+\rho \B{D}_1)\B{U}_1^T(\B{U}_1\B{D}_1\B{U}_1^T\bar{\B{w}}+\B{U}_1\B{\Sigma}_1\B{V}_1^T\B{\epsilon})  \nonumber \\
		\stackrel{(c)}{=} & \B{U}_1\diag\left(\left\{\frac{1-(1-\eta d_j-\eta\rho d_j^2)^{k+1}}{\eta d_j + \eta\rho d_j^2}\right\}_{j=1}^r\right)\diag\left(\{1+\rho d_j\}_{j=1}^r\right)(\B{D}_1\B{U}_1^T\bar{\B{w}}+\B{\Sigma}_1\B{V}_1^T\B{\epsilon})\nonumber \\
		\stackrel{(d)}{=} & \B{U}_1\diag\left(\left\{\frac{1-(1-\eta d_j-\eta\rho d_j^2)^{k+1}}{\eta d_j + \eta\rho d_j^2}\right\}_{j=1}^r\right)\diag\left(\{1+\rho d_j\}_{j=1}^r\right)(\B{D}_1\B{U}_1^T\bar{\B{w}}+\B{D}_1\B{\Sigma}_1^{-1}\B{V}_1^T\B{\epsilon})\nonumber \\
		= & \B{U}_1\diag\left(\left\{\frac{1-(1-\eta d_j-\eta\rho d_j^2)^{k+1}}{\eta d_j + \eta\rho d_j^2}\right\}_{j=1}^r\right)\diag\left(\{d_j+\rho d_j^2\}_{j=1}^r\right)(\B{U}_1^T\bar{\B{w}}+\B{\Sigma}_1^{-1}\B{V}_1^T\B{\epsilon})\nonumber\\
		= & \frac{1}{\eta}\B{U}_1\diag\left(\left\{{1-(1-\eta d_j-\eta\rho d_j^2)^{k+1}}\right\}_{j=1}^r\right)(\B{U}_1^T\bar{\B{w}}+\B{\Sigma}_1^{-1}\B{V}_1^T\B{\epsilon})\nonumber \\ =&\frac{1}{\eta}\B{U}_1(\B{I}-(\B{I}-\eta\B{D}_1-\eta\rho\B{D}_1^2)^{k+1})(\B{U}_1^T\bar{\B{w}}+\B{\Sigma}_1^{-1}\B{V}_1^T\B{\epsilon})\label{thm1-iterations-2}
	\end{align}
	where $(a)$ is by substituting $\B{H}=\B{U}_1\B{D}_1\B{U}_1^T=\B{U}\B{D}\B{U}^T$, $(b)$ is by the fact the fact $\B{U}_2^T\B{U}_1=\B{0}$, $(c)$ is using $\B{U}_1^T\B{U}_1=\B{I}$ and
	$$\sum_{i=0}^k (1-x)^i = \frac{1-(1-x)^{k+1}}{x}$$
	and $(d)$ is true as $\B{D}_1$ is invertible.
	Moreover,
	\begin{align}
		\B{U}(\B{I}-\eta_1\B{D}_1-\eta\rho\B{D}_1^2)^{k}\B{U}^T\B{w}_0& =\B{U}(\B{I}-\eta_1\B{D}_1-\eta\rho\B{D}_1^2)^{k}\B{U}^T\B{U}_1\B{U}_1^T\B{w}_0\nonumber \\
		& =\B{U}_1(\B{I}-\eta_1\B{D}_1-\eta\rho\B{D}_1^2)^{k}\B{U}_1^T\B{w}_0\label{w0-thm3}
	\end{align}
	where the first equation is by the assumption $\B{w}_0=\B{U}_1\B{U}_1^T\B{w}_0$.
	By~\eqref{thm1-iterations-1},~\eqref{thm1-iterations-2} and~\eqref{w0-thm3}, we achieve:
	\begin{equation}\label{thm1-wk}
		\B{w}_{k+1}^{\sam}=\B{U}_1(\B{I}-(\B{I}-\eta\B{D}_1-\eta\rho\B{D}_1^2)^{k+1})(\B{U}_1^T\bar{\B{w}}+\B{\Sigma}_1^{-1}\B{V}_1^T\B{\epsilon})\\+\B{U}_1(\B{I}-\eta_1\B{D}_1-\eta\rho\B{D}_1^2)^{k+1}\B{U}_1^T\B{w}_0.
	\end{equation}
	By taking the expectation, we have
	\begin{align}
		\E_{\B{\epsilon}}[\B{w}_{k}^{\sam}]  =\B{U}_1(\B{I}-(\B{I}-\eta\B{D}_1-\eta\rho\B{D}_1^2)^{k})\B{U}_1^T\bar{\B{w}}+\B{U}_1(\B{I}-\eta_1\B{D}_1-\eta\rho\B{D}_1^2)^{k}\B{U}_1^T\B{w}_0.\label{sharp-tohelp-thm3}
	\end{align}
	This implies
	\begin{equation}\label{thm1-e-uu}
		\E_{\B{\epsilon}}[\B{w}_{k}^{\sam}] -\B{U}_1\B{U}_1^T\bar{\B{w}}=-\B{U}_1(\B{I}-\eta\B{D}_1-\eta\rho\B{D}_1^2)^{k}\B{U}_1^T(\bar{\B{w}}-\B{w}_0).
	\end{equation}
	Next, by the definition of bias in~\eqref{bias-var-new},
	\begin{align}
		n \bias(\B{w}_k^{\sam})&= (\E_{\B{\epsilon}}[\B{w}_{k}^{\sam}]-\bar{\B{w}})^T\B{H}(\E_{\B{\epsilon}}[\B{w}_{k}^{\sam}]-\bar{\B{w}}) \nonumber \\
		& \stackrel{(a)}{=} (\E_{\B{\epsilon}}[\B{w}_{k}^{\sam}]-\bar{\B{w}})^T\B{U}_1\B{D}_{1}\B{U}_1^T(\E_{\B{\epsilon}}[\B{w}_{k}^{\sam}]-\bar{\B{w}})\nonumber \\
		& \stackrel{(b)}{=} (\E_{\B{\epsilon}}[\B{w}_{k}^{\sam}]-\B{U}_1\B{U}_1^T\bar{\B{w}})^T\B{U}_1\B{D}_{1}\B{U}_1^T(\E_{\B{\epsilon}}[\B{w}_{k}^{\sam}]-\B{U}_1\B{U}_1^T\bar{\B{w}})\nonumber \\
		& \stackrel{(c)}{=}  (\bar{\B{w}}-\B{w}_0)^T \B{U}_1^T(\B{I}-\eta\B{D}_1-\eta\rho\B{D}_1^2)^k \B{D}_{1}(\B{I}-\eta\B{D}_1-\eta\rho\B{D}_1^2)^k\B{U}_1^T(\bar{\B{w}}-\B{w}_0)\nonumber \\
		& = \sum_{i=1}^r (1-\eta d_i-\eta\rho d_i^2)^{2k}d_{i} u_i^2
	\end{align}
	where $(a)$ is by the SVD of $\B{H}$, $(b)$ is true as $\bar{\B{w}}=\B{U}_1\B{U}_1^T\bar{\B{w}} + \B{U}_2\B{U}_2^T\bar{\B{w}}$ and the fact $\B{U}_1^T\B{U}_2=\B{0}$, and $(c)$ is by~\eqref{thm1-e-uu}. This completes the bias part of the theorem. Next, note that
	\begin{equation*}
		\B{w}_{k}^{\sam}-\B{U}_1\B{U}_1^T\bar{\B{w}}=-\B{U}_1(\B{I}-\eta\B{D}_1-\eta\rho\B{D}_1^2)^{k}\B{U}_1^T(\bar{\B{w}}-\B{w}_0)+\\ \B{U}_1(\B{I}-(\B{I}-\eta\B{D}_1-\eta\rho\B{D}_1^2)^{k})\B{\Sigma}_1^{-1}\B{V}_1^T\B{\epsilon}.
	\end{equation*}
	As a result,
	\begin{align}
		n\error(\B{w}_k^{\sam})=& \E_{\B{\epsilon}}\left[(\B{w}_k^{\sam}-\bar{\B{w}})^T\B{H}(\B{w}_k^{\sam}-\bar{\B{w}})\right] \nonumber\\
		= & \E_{\B{\epsilon}}\left[(\B{w}_k^{\sam}-\bar{\B{w}})^T\B{U}_1\B{D}_{1}\B{U}_1^T(\B{w}_k^{\sam}-\bar{\B{w}})\right]\nonumber \\
		= & \E_{\B{\epsilon}}\left[(\B{w}_k^{\sam}-\B{U}_1\B{U}_1^T\bar{\B{w}})^T\B{U}_1\B{D}_{1}\B{U}_1^T(\B{w}_k^{\sam}-\B{U}_1\B{U}_1^T\bar{\B{w}})\right]\nonumber \\
		= & \E_{\B{\epsilon}}\Big[(\bar{\B{w}}-\B{w}_0)^T \B{U}_1(\B{I}-\eta\B{D}_1-\eta\rho\B{D}_1^2)^k \B{D}_{1}(\B{I}-\eta\B{D}_1-\eta\rho\B{D}_1^2)^k\B{U}_1^T(\bar{\B{w}}-\B{w}_0)\label{thm1-term1}\\
		&+ \B{\epsilon}^T\B{V}_1\B{\Sigma}_1^{-2}\B{D}_{1} \left(\B{I}-(\B{I}-\eta\B{D}_1-\eta\rho\B{D}_1^2)^k\right)^2 \B{V}_1^T\B{\epsilon}\label{thm1-term2} \\
		& - 2(\bar{\B{w}}-\B{w}_0)^T\B{U}_1(\B{I}-\eta\B{D}_1-\eta\rho\B{D}_1^2)^k\B{D}_{1}(\B{I}-(\B{I}-\eta\B{D}_1-\eta\rho\B{D}_1^2)^k)\B{\Sigma}_1^{-1}\B{V}_1^T\B{\epsilon}\Big]\label{thm1-term3}.
	\end{align}
	Note that $\E_{\B{\epsilon}}[\eqref{thm1-term3}]=0$, $\E_{\B{\epsilon}}[\eqref{thm1-term1}]=n\bias(\B{w}_k^{\sam})$ which implies $\E_{\B{\epsilon}}[\eqref{thm1-term2}]=n\var(\B{w}_k^{\sam})$. Finally, one has
	\begin{align} \E_{\B{\epsilon}}[\B{\epsilon}^T\B{V}_1 \left(\B{I}-(\B{I}-\eta\B{D}_1-\eta\rho\B{D}_1^2)^k\right)^2 \B{V}_1^T\B{\epsilon}]
		& = \E_{\B{\epsilon}}\left[\tr\left(\B{\epsilon}\B{\epsilon}^T\B{V}_1 \left(\B{I}-(\B{I}-\eta\B{D}_1-\eta\rho\B{D}_1^2)^k\right)^2 \B{V}_1^T\right)\right]\nonumber \\
		& = \tr\left(\E_{\B{\epsilon}}\left[\B{\epsilon}\B{\epsilon}^T\B{V}_1 \left(\B{I}-(\B{I}-\eta\B{D}_1-\eta\rho\B{D}_1^2)^k\right)^2 \B{V}_1^T\right]\right)\nonumber \\
		& \stackrel{(a)}{=} \sigma^2 \tr\left(\B{V}_1 \left(\B{I}-(\B{I}-\eta\B{D}_1-\eta\rho\B{D}_1^2)^k\right)^2 \B{V}_1^T\right)\nonumber \\
		& = \sigma^2 \tr\left( \left(\B{I}-(\B{I}-\eta\B{D}_1-\eta\rho\B{D}_1^2)^k\right)^2 \right)
	\end{align}
	where $(a)$ uses $\E_{\B{\epsilon}}[\B{\epsilon}\B{\epsilon}^T]=\sigma^2\B{I}$.
	For the next part of the proof, 
	\begin{align}
		\bias(\B{w}_k^{\sam})& = \frac{1}{n}\sum_{j=1}^r (1-\eta d_j-\eta\rho d_j^2)^{2k}d_ju_j^2 \nonumber \\
		& \leq \frac{1}{n}\sum_{j=1}^r (1-\eta d_j)^{2k}d_ju_j^2\nonumber \\
		& =\bias(\B{w}_k^{\gd}).
	\end{align}
	The proof for variance follows.
\end{proof}

\subsection{Proof of Proposition~\ref{lin-gen-prop}}
Before proceeding with the proof of the proposition, we first show a few technical lemmas.
\begin{lemma}\label{lem-prop2-helper}
	Let 
	$$ q(x) = a\exp(x\log b) - c\exp(x\log d)$$
	where $a\geq c$ and $b< d$. Then, $q(x)\geq 0$ for $x\in[0,\log(c/a)/\log(b/d)]$.
\end{lemma}
\begin{proof}
	One has
	\begin{align}
		q(x) = a\exp(x\log b) - c\exp(x\log d)=0\Rightarrow \exp(x\log(b/d))= \frac{c}{a} \Rightarrow x = \frac{\log(c/a)}{\log(b/d)}\geq 0
	\end{align}
	as $c/a\leq 1$ and $d/b>1$. Hence, $q(x)$ does not change sign between $0$ and $\log(c/a)/\log(b/d)$ by the intermediate value theorem. Moreover, $q(0)= a-c \geq 0$ implying $q(x)\geq 0$ in this interval.
\end{proof}

\begin{lemma}\label{lemma-lbub}
	Under the assumptions of Theorem~\ref{thm-relu}, 
	\begin{equation*}
		\frac{1}{n}(1-\eta d_1-\eta\rho d_1^2)^{2k}\|\B{X}(\bar{\B{w}}-\B{w}_0)\|_2^2\leq\bias(\B{w}_k^{\sam})\leq \frac{1}{n}(1-\eta d_r-\eta\rho d_r^2)^{2k}\|\B{X}(\bar{\B{w}}-\B{w}_0)\|_2^2
	\end{equation*}
	and
	\begin{multline*}
		\frac{\sigma^2r}{n}(1-\eta d_1-\eta\rho d_1^2)^{2k} - \frac{2\sigma^2 r}{n}(1-\eta d_r-\eta\rho d_r^2)^k  \\ \leq \var(\B{w}_k^{\sam})-\frac{\sigma^2r}{n} \leq\\
		\frac{\sigma^2r}{n}(1-\eta d_r-\eta\rho d_r^2)^{2k} - \frac{2\sigma^2 r}{n}(1-\eta d_1-\eta\rho d_1^2)^k.
	\end{multline*}
\end{lemma}
\begin{proof}
	From Theorem~\ref{thm-relu},
	\begin{equation*}
		\frac{1}{n}(1-\eta d_1-\eta\rho d_1^2)^{2k}\|\B{X}(\bar{\B{w}}-\B{w}_0)\|_2^2\leq\bias(\B{w}_k^{\sam})\leq \frac{1}{n}(1-\eta d_r-\eta\rho d_r^2)^{2k}\|\B{X}(\bar{\B{w}}-\B{w}_0)\|_2^2
	\end{equation*}
	where we used 
	\begin{align}
		\|\B{X}(\bar{\B{w}}-\B{w}_0)\|_2^2& =(\bar{\B{w}}-\B{w}_0)^T\B{X}^T\B{X}(\bar{\B{w}}-\B{w}_0)=(\bar{\B{w}}-\B{w}_0)^T\B{U}_1\B{D}_1\B{U}_1^T(\bar{\B{w}}-\B{w}_0)\nonumber\\&=\B{u}^T\B{D}_1\B{u}=\sum_{j=1}^r d_ju_j^2.\nonumber
	\end{align}
	Moreover, from Theorem~\ref{thm-relu} we also have
	\begin{align*}
		\var(\B{w}_k^{\sam})&=\frac{\sigma^2}{n} \tr ( (\B{I}-(\B{I}-\eta\B{D}_1-\eta\rho\B{D}_1^2)^k)^2 )\\
		& = \frac{\sigma^2r}{n}+\frac{\sigma^2}{n} \tr ( (\B{I}-\eta\B{D}_1-\eta\rho\B{D}_1^2)^{2k} )-\frac{2\sigma^2}{n} \tr ( (\B{I}-\eta\B{D}_1-\eta\rho\B{D}_1^2)^k ).
	\end{align*}
	The rest of the proof follows.
\end{proof}

\begin{proof}[\textbf{Proof of Proposition~\ref{lin-gen-prop}}.]
	Noting that $\error(\B{w}_k^{\sam})=\bias(\B{w}_k^{\sam})+\var(\B{w}_k^{\sam})$, from Lemma~\ref{lemma-lbub} we have
	\begin{equation}\label{prop2-helper1}
		\error(\B{w}_k^{\sam})-\frac{\sigma^2r}{n}\leq \left(\frac{\|\B{X}(\bar{\B{w}}-\B{w}_0)\|_2^2+\sigma^2 r}{n}\right)(1-\eta d_r-\eta\rho d_r^2)^{2k}-\frac{2\sigma^2 r}{n}(1-\eta d_1-\eta\rho d_1^2)^k
	\end{equation}
	and
	\begin{equation}
		\error(\B{w}_k^{\gd})-\frac{\sigma^2r}{n}\geq \left(\frac{\|\B{X}(\bar{\B{w}}-\B{w}_0)\|_2^2+\sigma^2 r}{n}\right)(1-\eta d_1)^{2k}-\frac{2\sigma^2 r}{n}(1-\eta d_r)^k
	\end{equation}
	by setting $\rho=0$. From~\eqref{lin-gen-prop-cond},
	\begin{align}
		&\error(\B{w}_k^{\gd})-\frac{\sigma^2r}{n} \nonumber  \\ \geq& \left(\frac{\|\B{X}(\bar{\B{w}}-\B{w}_0)\|_2^2+\sigma^2 r}{n}\right)c_0^k(1-\eta d_r-\eta\rho d_r^2)^{2k}-\frac{2\sigma^2 r}{n}c_0^k(1-\eta d_1-\eta\rho d_1^2)^k \nonumber \\
		= &c_0^k\left(\left(\frac{\|\B{X}(\bar{\B{w}}-\B{w}_0)\|_2^2+\sigma^2 r}{n}\right)(1-\eta d_r-\eta\rho d_r^2)^{2k}-\frac{2\sigma^2 r}{n}(1-\eta d_1-\eta\rho d_1^2)^k\right)\nonumber \\
		\geq & c_0^k\left( \error(\B{w}_k^{\sam})-\frac{\sigma^2r}{n}\right)
	\end{align}
	where the last inequality is by~\eqref{prop2-helper1}. As a result,
	\begin{align*}
		\error(\B{w}_k^{\gd})\geq c_0^k\error(\B{w}_k^{\sam}) + (1-c_0^k)\frac{\sigma^2r}{n}
	\end{align*}
	or equivalently,
	\begin{align}\label{prop2-helper3}
		\error(\B{w}_k^{\sam}) - \error(\B{w}_k^{\gd})\leq (1-c_0^k)\left(\error(\B{w}_k^{\sam})-\frac{\sigma^2r}{n}\right).
	\end{align}
	Next, let $a= (\|\B{X}(\bar{\B{w}}-\B{w}_0)\|_2^2+\sigma^2 r)/n$, $b=(1-\eta d_1-\eta\rho d_1^2)^2$, $c=2\sigma^2 r/n$ and $d=(1-\eta d_r-\eta\rho d_r^2)$. Define $q(x)$ as in Lemma~\ref{lem-prop2-helper}. Then, from Lemma~\ref{lemma-lbub} we have
	\begin{align}
		& \error(\B{w}_k^{\sam})-\frac{\sigma^2r}{n}\nonumber \\ \geq& \left(\frac{\|\B{X}(\bar{\B{w}}-\B{w}_0)\|_2^2+\sigma^2 r}{n}\right)(1-\eta d_1-\eta\rho d_1^2)^{2k}-\frac{2\sigma^2 r}{n}(1-\eta d_r-\eta\rho d_r^2)^k\nonumber \\
		=& q(k).
	\end{align}
	From Lemma~\ref{lem-prop2-helper}, if 
	$$k\leq \frac{\log(c/a)}{\log(b/d)}=\frac{\log(2\sigma^2r/(\|\B{X}(\bar{\B{w}}-\B{w}_0)\|_2^2+\sigma^2r))}{\log((1-\eta d_1-\eta\rho d_1^2)^2/(1-\eta d_r-\eta\rho d_r^2)}$$
	then $q(k)\geq 0$ or equivalently, $\error(\B{w}_k^{\sam})-{\sigma^2r}/{n}\geq 0$. Therefore, from~\eqref{prop2-helper3}, 
	\begin{align}
		\error(\B{w}_k^{\sam}) - \error(\B{w}_k^{\gd})\leq \underbrace{(1-c_0^k)}_{\leq 0}\underbrace{\left(\error(\B{w}_k^{\sam})-\frac{\sigma^2r}{n}\right)}_{\geq 0}\leq 0
	\end{align}
	which completes the proof. 
\end{proof}

\subsection{Proof of Theorem~\ref{kernel-bias-var-thm}}

\begin{proof}
	First, suppose $h=\B{w}^T\Kdot$ for some $\B{w}\in\R^n$. Then, 
	\begin{align}
		\sum_{i=1}^n\left(h(\B{x}_i)-\bar{h}(\B{x}_i)\right)^2 & =      \sum_{i=1}^n\left(\sum_{j=1}^n(w_j-\bar{w}_j)K(\B{x}_i,\B{x}_j)\right)^2 \nonumber \\
		&= \|\KX(\B{w}-\bar{\B{w}})\|_2^2 \nonumber \\
		&= (\B{w}-\bar{\B{w}})^T\KX^2(\B{w}-\bar{\B{w}}).\label{thm4-helper1}
	\end{align}
	Moreover, 
	\begin{align}
		\bias(h) & = \frac{1}{n}\sum_{i=1}^n \left(\E[h(\B{x}_i)]-\bar{h}(\B{x}_i)\right)^2 \\
		& = \frac{1}{n}\sum_{i=1}^n\left(\E[\B{w}^T[\KX]_i]-\bar{\B{w}}^T[\KX]_i\right)^2\nonumber \\
		& = \frac{1}{n}\sum_{i=1}^n\left([\KX]_i^T(\E[\B{w}]-\bar{\B{w}})\right)^2\nonumber \\
		& = \frac{1}{n}(\E[\B{w}]-\bar{\B{w}})^T\KX^2(\E[\B{w}]-\bar{\B{w}})\label{thm4-helper20}
	\end{align}
	where $[\KX]_i$ denotes the $i$-th column of $\KX$. From Lemma~\ref{thm4-helper-lemma}, one has 
	\begin{align}
		\B{w}_{k+1}^{\sam}&=\eta \sum_{i=0}^k (\B{I}-\eta \KX -\eta\rho\KX^2)(\B{I}+\rho\KX)(\KX\bar{\B{w}}+\B{\epsilon})\nonumber\\&=\eta \sum_{i=0}^k (\B{I}-\eta \KX -\eta\rho\KX^2)^i(\KX+\rho\KX^2)(\bar{\B{w}}+\KX^{-1}\B{\epsilon}) \nonumber \\
		& = \B{U}\left(\B{I}-\left(\B{I}-\eta\B{D}-\eta\rho\B{D}^2\right)^{k+1}\right)(\B{U}^T\bar{\B{w}}+\B{D}^{-1}\B{U}^T\B{\epsilon}).
	\end{align}
	As a result,
	\begin{align}
		\E[ \B{w}_{k+1}^{\sam}] = \B{U}\left(\B{I}-\left(\B{I}-\eta\B{D}-\eta\rho\B{D}^2\right)^{k+1}\right)\B{U}^T\bar{\B{w}}\label{thm4-helper21}
	\end{align}
	and
	\begin{align}\label{thm4-helper1.5}
		\B{w}_{k+1}^{\sam}-\bar{\B{w}} = -\B{U}\left(\B{I}-\eta\B{D}-\eta\rho\B{D}^2\right)^{k+1}\B{U}^T\bar{\B{w}}+\B{U}\left(\B{I}-\left(\B{I}-\eta\B{D}-\eta\rho\B{D}^2\right)^{k+1}\right)\B{D}^{-1}\B{U}^T\B{\epsilon}.
	\end{align}
	From~\eqref{thm4-helper1},
	\begin{align}
		\sum_{i=1}^n\left(h_k^{\sam}(\B{x}_i)-\bar{h}(\B{x}_i)\right)^2  = &  (\B{w}_k^{\sam}-\bar{\B{w}})^T\KX^2(\B{w}_k^{\sam}-\bar{\B{w}})\label{thm4-helper-p3}\\
		= & (\B{w}_k^{\sam}-\bar{\B{w}})^T\B{U}\B{D}^2\B{U}^T(\B{w}_k^{\sam}-\bar{\B{w}})\nonumber \\
		\stackrel{(a)}{=} & \bar{\B{w}}^T\B{U}\left(\B{I}-\eta\B{D}-\eta\rho\B{D}^2\right)^{2k}\B{D}^2\B{U}^T\bar{\B{w}} \label{thm4-helper-p1}\\
		& + \B{\epsilon}^T\B{U}\left(\B{I}-\left(\B{I}-\eta\B{D}-\eta\rho\B{D}^2\right)^{k}\right)^2\B{U}^T\B{\epsilon}\label{thm4-helper-p4}\\
		& -2\bar{\B{w}}^T\B{U}\left(\B{I}-\eta\B{D}-\eta\rho\B{D}^2\right)^{k}\B{D}\left(\B{I}-\left(\B{I}-\eta\B{D}-\eta\rho\B{D}^2\right)^{k}\right)\B{U}^T\B{\epsilon}\label{thm4-helper-p2}
	\end{align}
	where $(a)$ is by~\eqref{thm4-helper1.5}. On the other hand, from~\eqref{thm4-helper20} and~\eqref{thm4-helper21}, we have 
	\begin{align}
		\bias(h) & =  \frac{1}{n}(\E[\B{w}]-\bar{\B{w}})^T\KX^2(\E[\B{w}]-\bar{\B{w}}) \nonumber\\
		& =\bar{\B{w}}^T\B{U}\left(\B{I}-\eta\B{D}-\eta\rho\B{D}^2\right)^{2k}\B{D}^2\B{U}^T\bar{\B{w}}.
	\end{align}
	As a result, $\E[\eqref{thm4-helper-p1}]=n\bias(h)$, and also $\E[\eqref{thm4-helper-p2}]=0$, $\E[\eqref{thm4-helper-p3}]=n\error(h)$, showing $\E[\eqref{thm4-helper-p4}]=n\var(h)$. The rest of the proof follows from the proof of Theorem~\ref{thm-relu}.

\end{proof}

\subsection{Proof of Proposition~\ref{kernel-gen-prop}}
\begin{proof}
	Let $\error^+(\B{w}_k^{\sam})=\bias_+(\B{w}_k^{\sam})+\var^+(\B{w}_k^{\sam})$, $\error^-(\B{w}_k^{\sam})=\bias_-(\B{w}_k^{\sam})+\var^-(\B{w}_k^{\sam})$ where 
	\begin{equation}
		\begin{aligned}
			\bias_+(\B{w}_k^{\sam})&=\frac{1}{n}\sum_{i=1}^r (1-\eta d_i-\eta\rho d_i^2)^{2k}d_i^2 u_i^2\\
			\bias_-(\B{w}_k^{\sam})&=\frac{1}{n}\sum_{i=r+1}^n (1-\eta d_i-\eta\rho d_i^2)^{2k}d_i^2 u_i^2.
		\end{aligned}
	\end{equation}
	Note that by following the same steps of the proof of Proposition~\ref{lin-gen-prop}, we have under the assumptions of the proposition, 
	$$\error^+(\B{w}_k^{\sam})\leq \error^+(\B{w}_k^{\gd}). $$
	Moreover, note that $\bias_-(\B{w}_k^{\sam})\leq \bias_-(\B{w}_k^{\gd})$ similar to bias corresponding to the convex part. Finally, 
	\begin{align}
		\var^-(\B{w}_k^{\gd}) - \var^-(\B{w}_k^{\sam}) & = \frac{\sigma^2}{n}\sum_{i=r+1}^n \left\{\left((1-\eta d_i)^k-1\right)^2-\left(1-(1-\eta d_i-\eta\rho d_i^2)^k\right)^2 \right\} \nonumber \\
		& \geq \frac{\sigma^2}{n}\sum_{i=r+1}^n \left\{\left((1+\varepsilon)^k-1\right)^2-1 \right\} \nonumber \\
		& = \frac{(n-r)\sigma^2}{n}(1+\varepsilon)^k\left((1+\varepsilon)^k-2\right)\geq 0
	\end{align}
	where the last inequality follows the lower bound on $k$ from the proposition.
\end{proof}

\subsection{Proof of Proposition~\ref{sharpness-prop}}

\begin{proof}
	Let $f(\B{w})$ be defined as in~\eqref{quadratic-loss} for the ReLU case and in~\eqref{kernel-loss} for the kernel case. Then, 
	$$\E_{\B\epsilon}[f(\B{w})]=\frac{1}{2}(\B{w}-\bar{\B{w}})^T\B{H}(\B{w}-\bar{\B{w}})$$
	where $\B{H}=\B{A}^T\B{A}$ or $\B{H}=\KX$ for these two cases. Consider the eigenvalue decomposition $\B{H}=\B{U}\B{D}\B{U}^T=\B{U}_1\B{D}_1\B{U}_1^T+\B{U}_2\B{D}_2\B{U}_2^T$ where $\B{D}_1\succ 0\succ \B{D}_2$ for the kernel case and $\B{D}_1\succ\B{D}_2=\B{0}$ for the ReLU case. Then,
	\begin{align}
		\E_{\B{\epsilon}}[f(\B{w}+\B\varepsilon)]-\E_{\B{\epsilon}}[f(\B{w})] & = \frac{1}{2}(\B{w}+\B\varepsilon-\bar{\B{w}})^T\B{H}(\B{w}+\B\varepsilon-\bar{\B{w}}) - \frac{1}{2}(\B{w}-\bar{\B{w}})^T\B{H}(\B{w}-\bar{\B{w}}) \nonumber \\
		& = \frac{1}{2}\B\varepsilon^T\B{H}\B\varepsilon + \B\varepsilon^T\B{H}(\B{w}-\bar{\B{w}}) \nonumber \\
		& = \frac{1}{2}\B\varepsilon^T(\B{U}_1\B{D}_1\B{U}_1^T+\B{U}_2\B{D}_2\B{U}_2^T)\B\varepsilon +\B\varepsilon^T\B{U}\B{D}\B{U}^T(\B{w}-\bar{\B{w}})  \nonumber \\
		& = \frac{1}{2}(\B{U}_1^T\B\varepsilon)^T\B{D}_1(\B{U}_1^T\B\varepsilon)+\frac{1}{2}(\B{U}_2^T\B\varepsilon)^T\B{D}_2(\B{U}_2^T\B\varepsilon)+\B\varepsilon^T\B{U}\B{D}\B{U}^T(\B{w}-\bar{\B{w}}) .\label{e48}
	\end{align}
	\textbf{Part 1:} First, consider the case with $\B{D}_2=\B{0}$. Then,
	\begin{align}
		& \max_{\|\B\varepsilon\|_2\leq \rho_0} \E_{\B\epsilon}[f(\B{w}+\B\varepsilon)]-\E_{\B\epsilon}[f(\B{w})] \nonumber \\
		\geq& \max_{\|\B\varepsilon\|_2= \rho_0} \E_{\B\epsilon}[f(\B{w}+\B\varepsilon)]-\E_{\B\epsilon}[f(\B{w})]\nonumber \\
		\stackrel{(a)}{\geq} &\max_{\|\B\varepsilon\|_2= \rho_0}\left[ \frac{1}{2}\lambda_{\min}(\B{D}_1)\|\B{U}_1^T\B\varepsilon\|_2^2+  (\B{U}_1^T\B\varepsilon)^T\B{D}_1\B{U}_1^T(\B{w}-\bar{\B{w}})\right]\nonumber \\
		\geq &\max_{\substack{\|\B\varepsilon\|_2= \rho_0 \\ \B\varepsilon=\B{U}_1\B{v} }}\left[ \frac{1}{2}\lambda_{\min}(\B{D}_1)\|\B{U}_1^T\B\varepsilon\|_2^2+  (\B{U}_1^T\B\varepsilon)^T\B{D}_1\B{U}_1^T(\B{w}-\bar{\B{w}})\right]\nonumber \\
		\stackrel{(b)}{\geq} &\max_{\|\B{v}\|_2=\rho_0}\left[ \frac{1}{2}\lambda_{\min}(\B{D}_1)\|\B{v}\|_2^2+  \B{v}^T\B{D}_1\B{U}_1^T(\B{w}-\bar{\B{w}})\right]\nonumber \\
		=& \frac{1}{2}\lambda_{\min}(\B{D}_1)\rho_0^2 +\max_{\|\B{v}\|_2=\rho_0}\B{v}^T\B{D}_1\B{U}_1^T(\B{w}-\bar{\B{w}})
		\nonumber \\
		\stackrel{(c)}{=}& \frac{1}{2}\lambda_{\min}(\B{D}_1)\rho_0^2 +\rho_0 \|\B{D}_1\B{U}_1^T(\B{w}-\bar{\B{w}})\|_2
		\label{sharp-helper1}
	\end{align}
	where $\lambda_{\min}$ denotes the smallest eigenvalue of the matrix, $(a)$ is by~\eqref{e48}, $(b)$ is true as if $\B\varepsilon=\B{U}_1\B{v}$,
	$$\rho_0^2=\|\B{\varepsilon}\|_2^2=\B{v}^T\B{U}_1^T\B{U}_1\B{v}=\|\B{v}\|_2^2$$
	and $(c)$ follows $\B{v}^T\B{D}_1\B{U}_1^T(\B{w}-\bar{\B{w}})\leq \|\B{v}\|_2\|\B{D}_1\B{U}_1^T(\B{w}-\bar{\B{w}})\|_2 $. Similarly,
	\begin{align}
		&  \max_{\|\B\varepsilon\|_2\leq \rho_0} \E_{\B\epsilon}[f(\B{w}+\B\varepsilon)]-\E_{\B\epsilon}[f(\B{w})] \nonumber \\ \leq& \max_{\|\B\varepsilon\|_2\leq \rho_0}\left[ \frac{1}{2}\lambda_{\max}(\B{D}_1)\|\B{U}_1^T\B\varepsilon\|_2^2+  (\B{U}_1^T\B\varepsilon)^T\B{D}_1\B{U}_1^T(\B{w}-\bar{\B{w}})\right]\nonumber \\
		\stackrel{(a)}{\leq}& \frac{1}{2}\lambda_{\max}(\B{D}_1)\rho_0^2+ \max_{\|\B\varepsilon\|_2\leq \rho_0} (\B{U}_1^T\B\varepsilon)^T\B{D}_1\B{U}_1^T(\B{w}-\bar{\B{w}})\nonumber \\
		\leq&\frac{1}{2}\lambda_{\max}(\B{D}_1)\rho_0^2 + \rho_0\|\B{D}_1\B{U}_1^T(\B{w}-\bar{\B{w}})\|_2.\label{sharp-helper2}
	\end{align}    
	where $\lambda_{\max}$ denotes the largest eigenvalue and $(a)$ is true as $\|\B{U}_1^T\B\varepsilon\|_2\leq \|\B\varepsilon\|_2$.
	Note that from~\eqref{sharp-tohelp-thm3}
	\begin{align}
		\B{U}_1^T(\E_{\B\epsilon}[\B{w}_k^{\sam}]-\bar{\B{w}}) = -(\B{I}-\eta\B{D}_1-\eta\rho\B{D}_1^2)^k\B{U}_1^T(\bar{\B{w}}-\B{w}_0)
	\end{align}
	which implies
	\begin{align}
		\|\B{D}_1\B{U}_1^T(\E_{\B\epsilon}[\B{w}_k^{\sam}]-\bar{\B{w}})\|_2 &= \sqrt{\|\B{D}_1\B{U}_1^T(\E_{\B\epsilon}[\B{w}_k^{\sam}]-\bar{\B{w}})\|_2^2} \nonumber \\
		& = \sqrt{(\bar{\B{w}}-\B{w}_0)^T\B{U}_1\B{D}_1^2(\B{I}-\eta\B{D}_1-\eta\rho\B{D}_1^2)^{2k}\B{U}_1^T(\bar{\B{w}}-\B{w}_0)}\nonumber \\
		& = \sqrt{\sum_{i=1}^r (1-\eta d_i-\eta\rho d_i^2)^{2k}d_i^2u_i^2}.\label{sharp-helper3}
	\end{align}
	The proof follows from~\eqref{sharp-helper1},~\eqref{sharp-helper2} and~\eqref{sharp-helper3}.\\
	
	\textbf{Part 2}:
	From~\eqref{e48},
	\begin{align}
		& \max_{\|\B\varepsilon\|_2\leq \rho_0} \E_{\B\epsilon}[f(\B{w}+\B\varepsilon)]-\E_{\B\epsilon}[f(\B{w})] \nonumber \\
		\geq& \max_{\|\B\varepsilon\|_2= \rho_0} \E_{\B\epsilon}[f(\B{w}+\B\varepsilon)]-\E_{\B\epsilon}[f(\B{w})]\nonumber \\
		\geq &\max_{\|\B\varepsilon\|_2= \rho_0}\left[ \frac{1}{2}\lambda_{\min}(\B{D})\|\B{U}^T\B\varepsilon\|_2^2+  (\B{U}^T\B\varepsilon)^T\B{D}\B{U}^T(\B{w}-\bar{\B{w}})\right]\nonumber \\
		\geq & \frac{1}{2}\lambda_{\min}(\B{D})\rho_0^2 +\max_{\|\B{U}^T\B{\varepsilon}\|_2=\rho_0}(\B{U}^T\B{\B\varepsilon})^T\B{D}\B{U}^T(\B{w}-\bar{\B{w}})
		\nonumber \\
		\stackrel{}{=}& \frac{1}{2}\lambda_{\min}(\B{D})\rho_0^2 +\rho_0 \|\B{D}\B{U}^T(\B{w}-\bar{\B{w}})\|_2
		\label{sharp-helper1-kernel}
	\end{align}
	as $\|\B{U}^T\B{\varepsilon}\|_2=\|\B\varepsilon\|_2$. Next,
	\begin{align}
		&  \max_{\|\B\varepsilon\|_2\leq \rho_0} \E_{\B\epsilon}[f(\B{w}+\B\varepsilon)]-\E_{\B\epsilon}[f(\B{w})] \nonumber \\ \leq& \max_{\|\B\varepsilon\|_2\leq \rho_0}\left[ \frac{1}{2}\lambda_{\max}(\B{D})\|\B\varepsilon\|_2^2+  \B\varepsilon^T\B{U}\B{D}\B{U}^T(\B{w}-\bar{\B{w}})\right]\nonumber \\
		\leq&\frac{1}{2}\lambda_{\max}(\B{D})\rho_0^2 + \rho_0\|\B{D}\B{U}^T(\B{w}-\bar{\B{w}})\|_2.\label{sharp-helper2-kernel}
	\end{align}    
	
	From~\eqref{thm4-helper21},
	\begin{align}
		\B{U}^T(\E_{\B\epsilon}[\B{w}_k^{\sam}]-\bar{\B{w}}) & = (\B{I}-\eta\B{D}-\eta\rho\B{D}^2)^k\B{U}^T\bar{\B{w}}
	\end{align}
	implying
	\begin{align}
		\|\B{D}\B{U}^T(\E_{\B\epsilon}[\B{w}_k^{\sam}]-\bar{\B{w}})\|_2 &= \sqrt{\|\B{D}\B{U}^T(\E_{\B\epsilon}[\B{w}_k^{\sam}]-\bar{\B{w}})\|_2^2} \nonumber \\ 
		& = \sqrt{\bar{\B{w}}^T\B{U}\B{D}^2(\B{I}-\eta\B{D}-\eta\rho\B{D}^2)^{2k}\B{U}^T\bar{\B{w}}}\nonumber \\
		& = \sqrt{\sum_{i=1}^n (1-\eta d_i-\eta\rho d_i^2)^{2k}d_i^2u_i^2}.\label{sharp=kernel-helper3}
	\end{align}
	The proof follows from~\eqref{sharp-helper1-kernel}, ~\eqref{sharp-helper2-kernel} and~\eqref{sharp=kernel-helper3} as $1-\eta d_i-\eta\rho d_i^2\leq 1$ for $i\in[n]$ and $1<1-\eta d_i$ for any $i\geq r+1$.
\end{proof}

\subsection{Proof of Proposition~\ref{example-iid}}
\begin{lemma}\label{prop-iid-helper-lemma}
	Let $\B{x}\sim\cN(\B{0},\B{I})$. Then, $\E[\B{x}\B{x}^T\B{x}\B{x}^T]=(p+2)\B{I}$.
\end{lemma}
\begin{proof}
	Let $\hat{\B{H}}=\B{x}\B{x}^T$, $\hat{\B{\Theta}}=\hat{\B{H}}^2$ and $\B{\Theta}=\E[\hat{\B{\Theta}}]$.
	Note that
	\begin{align}
		\Theta_{i,j} & =\E[\hat{\Theta}_{i,j}] \nonumber \\
		& = \E[(\hat{\B{H}}^2)_{i,j}] \nonumber \\
		& = \E\left[\sum_{l=1}^p \hat{H}_{i,l}\hat{H}_{l,j}\right]\nonumber \\
		& = \sum_{l=1}^p\E[x_ix_jx_l^2].\label{iid-helper1}
	\end{align}
	Next, we consider the following cases in~\eqref{iid-helper1}:
	\begin{enumerate}
		\item $i\neq j$: In this case,
		\begin{align*}
			\sum_{l=1}^p\E[x_ix_jx_l^2]=\underbrace{\E[x_ix_l^3]}_{=0}+\underbrace{\E[x_jx_l^3]}_{=0}+\sum_{l\neq i,j}\underbrace{\E[x_ix_jx_l^2]}_{=0}=0. 
		\end{align*}
		\item $i=j$: Then,
		\begin{align*}
			\sum_{l=1}^p\E[x_ix_jx_l^2]=\underbrace{\E[x_i^4]}_{=3}+\sum_{l\neq i}\underbrace{\E[x_i^2x_l^2]}_{=1}=p+2. 
		\end{align*}
	\end{enumerate}
	Therefore, $\B{\Theta}=(p+2)\B{I}$.
\end{proof}
\begin{proof}[\textbf{Proof of Proposition~\ref{example-iid}}]
	Note that 
	\begin{align}
		f_k = \frac{1}{2}(\B{x}_k^T(\bar{\B{w}}-\B{w})+\epsilon_k)^2=\frac{1}{2}(\B{w}-\bar{\B{w}})^T\B{x}_k\B{x}_k^T(\B{w}-\bar{\B{w}})-\epsilon_k\B{w}^T\B{x}_k+\cdots.\label{f-k-def}
	\end{align}
	We start by writing SAM updates for the stochastic case. The intermediate solution of SAM is given as
	\begin{align}
		{\B{w}}_k^{\sam} + \rho \nabla f_k({\B{w}}_k^{\sam}) & =
		{\B{w}}_k^{\sam} + \rho\left(\B{g}_k +\B{H}_k({\B{w}}_{k}^{\sam}-\bar{\B{w}})\right)\nonumber \\
		& = (\B{I}+\rho\B{H}_k){\B{w}}_{k}^{\sam} + \rho (\B{g}_k - \B{H}_k\bar{\B{w}}).
	\end{align}
	where from~\eqref{f-k-def}, we have $\B{H}_k=\B{x}_k\B{x}_k^T$ and $\B{g}_k=-\epsilon_k\B{x}_k$. Therefore, the SAM update direction is given as
	\begin{align}
		\B{\nu}_k & = \nabla f_k({\B{w}}_k^{\sam} + \rho \nabla f_k({\B{w}}_k^{\sam})) \nonumber \\
		& = \B{g}_k + \B{H}_k((\B{I}+\rho\B{H}_k){\B{w}}_{k}^{\sam} + \rho (\B{g}_k - \B{H}_k\bar{\B{w}})-\bar{\B{w}}) \nonumber \\
		& = (\B{H}_k+\rho\B{H}_k^2){\B{w}}_k^{\sam} + (\B{I}+\rho\B{H}_k)(\B{g}_k-\B{H}_k\bar{\B{w}}).
	\end{align}
	Hence,
	\begin{align}
		{\B{w}}_{k+1}^{\sam} & = {\B{w}}_k^{\sam} - \eta\B{\nu}_k \nonumber \\
		& = (\B{I}-\eta\B{H}_k-\eta\rho\B{H}_k^2){\B{w}}_k^{\sam} -\eta (\B{I}+\rho\B{H}_k)(\B{g}_k-\B{H}_k\bar{\B{w}}).\label{thm-st-1}
	\end{align}
	In the next step, we take expectation: 
	\begin{align}
		\E[{\B{w}}_{k+1}^{\sam}] & = \E_{\B{x}_k,\epsilon_k}\left[\E_{\B{x}_1,\epsilon_1,\cdots,\B{x}_{k-1},\epsilon_{k-1}}\left[\left.{\B{w}}_{k+1}^{\sam}\right\vert \B{x}_k,\epsilon_k\right]\right]\nonumber \\
		&= \E_{\B{x}_k,\epsilon_k}\left[\E_{\B{x}_1,\epsilon_1,\cdots,\B{x}_{k-1},\epsilon_{k-1}}\left[\left.(\B{I}-\eta\B{H}_k-\eta\rho\B{H}_k^2){\B{w}}_k^{\sam} -\eta (\B{I}+\rho\B{H}_k)(\B{g}_k-\B{H}_k\bar{\B{w}})\right\vert \B{x}_k,\epsilon_k\right]\right]\nonumber \\
		& = \E_{\B{x}_k,\epsilon_k}\left[(\B{I}-\eta\B{H}_k-\eta\rho\B{H}_k^2)\E_{\B{x}_1,\epsilon_1,\cdots,\B{x}_{k-1},\epsilon_{k-1}}\left[\left.{\B{w}}_k^{\sam} \right\vert \B{x}_k,\epsilon_k\right]-\eta (\B{I}+\rho\B{H}_k)(\B{g}_k-\B{H}_k\bar{\B{w}})\right] \nonumber \\
		&\stackrel{(a)}{=} \E_{\B{x}_k,\epsilon_k}\left[(\B{I}-\eta\B{H}_k-\eta\rho\B{H}_k^2)\E\left[{\B{w}}_k^{\sam} \right]-\eta (\B{I}+\rho\B{H}_k)(\B{g}_k-\B{H}_k\bar{\B{w}})\right]\nonumber \\
		& \stackrel{(b)}{=} (\B{I}-\eta\E[\B{H}_k]-\eta\rho\E[\B{H}_k^2])\E[{\B{w}}_{k}^{\sam}]+\eta(\E[\B{H}_k]+\rho\E[\B{H}_k^2])\bar{\B{w}} \nonumber \\
		& = (\B{I}-\eta\B{H}-\eta\rho\B{\Theta})\E[{\B{w}^{\sam}_k}]+\eta(\B{H}+\rho\B{\Theta})\bar{\B{w}}.\label{thm2-helper1}
	\end{align}
	with $\B{\Theta}=\E[\B{H}_k^2]$, where in $(a)$ we used the fact that ${\B{w}}_k^{\sam}$ is independent of $\B{x}_k,\epsilon_k,\B{x}_{k+1},\epsilon_{k+1,\cdots}$, and $(b)$ is true as $\E[\B{g}_k]=\E[\B{H}_k\B{g}_k]=\B{0}$ due to the independence of $\epsilon_k$ and $\B{x}_k$. Thus, from~\eqref{thm2-helper1}
	\begin{align}
		\E[{\B{w}}_{k+1}^{\sam}] &= \eta \sum_{i=0}^k (\B{I}-\eta\B{H}-\eta\rho\B{\Theta})^i(\B{H}+\rho\B{\Theta})\bar{\B{w}}\nonumber \\
		& = \eta \sum_{i=1}^k(1-\eta-\eta\rho(p+2))^i(1+\rho(p+2))\bar{\B{w}}\nonumber \\
		&= \left(1-(1-\eta-\eta\rho(p+2))^{k+1}\right)\bar{\B{w}}
	\end{align}
	where the second equality is by Lemma~\ref{prop-iid-helper-lemma}.
	Finally, we calculate error from~\eqref{error-stochastic} as 
	\begin{align}
		\E_{\B{x}_0,\epsilon_0}\left[\left(\B{x}_0^T(\E[\B{w}_k^{\sam}]-\bar{\B{w}})-\epsilon_0\right)^2\right] & = (\E[\B{w}_k^{\sam}]-\bar{\B{w}})^T\E[\B{x}_0\B{x}_0^T](\E[\B{w}_k^{\sam}]-\bar{\B{w}}) + \E[\epsilon_0^2] \nonumber\\
		& = (\E[\B{w}_k^{\sam}]-\bar{\B{w}})^T(\E[\B{w}_k^{\sam}]-\bar{\B{w}})+\sigma^2 \nonumber \\
		& = (1-\eta-\eta\rho(p+2))^{2k}\|\bar{\B{w}}\|_2^2
	\end{align}
	which completes the proof.
\end{proof}

\section{Additional Numerical Experiments}\label{app:numerical}

\subsection{Linear Regression Experiments}\label{app:linear-exp}
First, we examine the theory we developed on the linear regression problem. To this end, we set $p=100,n=60$ corresponding to an over-parameterized regime, with observations as $y_i=\B{x}_i^T\bar{\B{w}}+\epsilon_i$. We assume $(\B{x}_i,\epsilon_i)$ are independent, and $\epsilon_i\sim\cN(0,\sigma^2)$ and $\B{x}_i\sim(\B{0},\B{\Sigma})$ where $\B{\Sigma}$ is an exponential covariance matrix, $\Sigma_{i,j}=0.5^{|i-j|}$. Each coordinate of $\bar{\B{w}}$ is independently chosen from $\text{Unif}[0,1]$ and then $\bar{\B{w}}$ is normalized to have norm one. We run GD and SAM with $\eta=1/2\sigma_{\max}(\B{X})^2$ and $\rho=\eta/6$ for 500 iterations on the least squares loss. We draw 600 noiseless validation samples, denoted as $\B{y}_{\text{test}},\B{X}_{\text{test}}$ from the same model. We record the validation error defined as $\errortest(\B{w})=\|\B{y}_{\text{test}}-\B{X}_{\text{test}}\B{w}_{k}^{\gd}\|_2^2/600$ for GD (and SAM, similarly). We run the whole process described for 100 independent repetitions and report the average and standard deviation of results. We let $\errortest^{\gd},\errortest^{\sam}$ to denote the best validation error achieved for GD and SAM over all iterations, respectively. We also let $k^{\gd},k^{\sam}$ denote the number of iterations with the best error. In Figure~\ref{fig:app-lin}, we compare the best error from GD and SAM, and when they are achieved compared to the noise standard deviation. From Figure~\ref{fig:app-lin} [Left Panel], we see that if noise is small, GD and SAM perform similarly, although Figure~\ref{fig:app-lin} [Middle Panel] shows that SAM achieves the best error marginally earlier than GD. When noise is higher but not too high, SAM has a smaller error compared to GD, and this error is achieved earlier. If we continue to increase noise, SAM performs worse and early stopping does not seem to help SAM much. This result is consistent with the theory we developed. Particularly, our theory in Section~\ref{sec:mainres}, shows that in the noiseless case or small noise regime, SAM has a lower error compared to GD in earlier iterations. In our experiment, this is confirmed that in small noise regimes, SAM both achieves its best error earlier, and has a smaller error compared to GD. However, as we increase noise, variance increases and as SAM has higher variance, it performs worse. We have also shown the error of SAM in Figure~\ref{fig:app-lin} [Right Panel], showing that increasing noise leads to larger error, as expected.

\begin{figure*}[t!]
	\centering
	\begin{tabular}{ccc}
		$\errortest^{\gd}/\errortest^{\sam}$  & $k^{\gd}-k^{\sam}$ & $\errortest^{\sam}$ \\
		\includegraphics[width=0.3\linewidth,trim =0.4cm 0cm .8cm 0cm, clip = true]{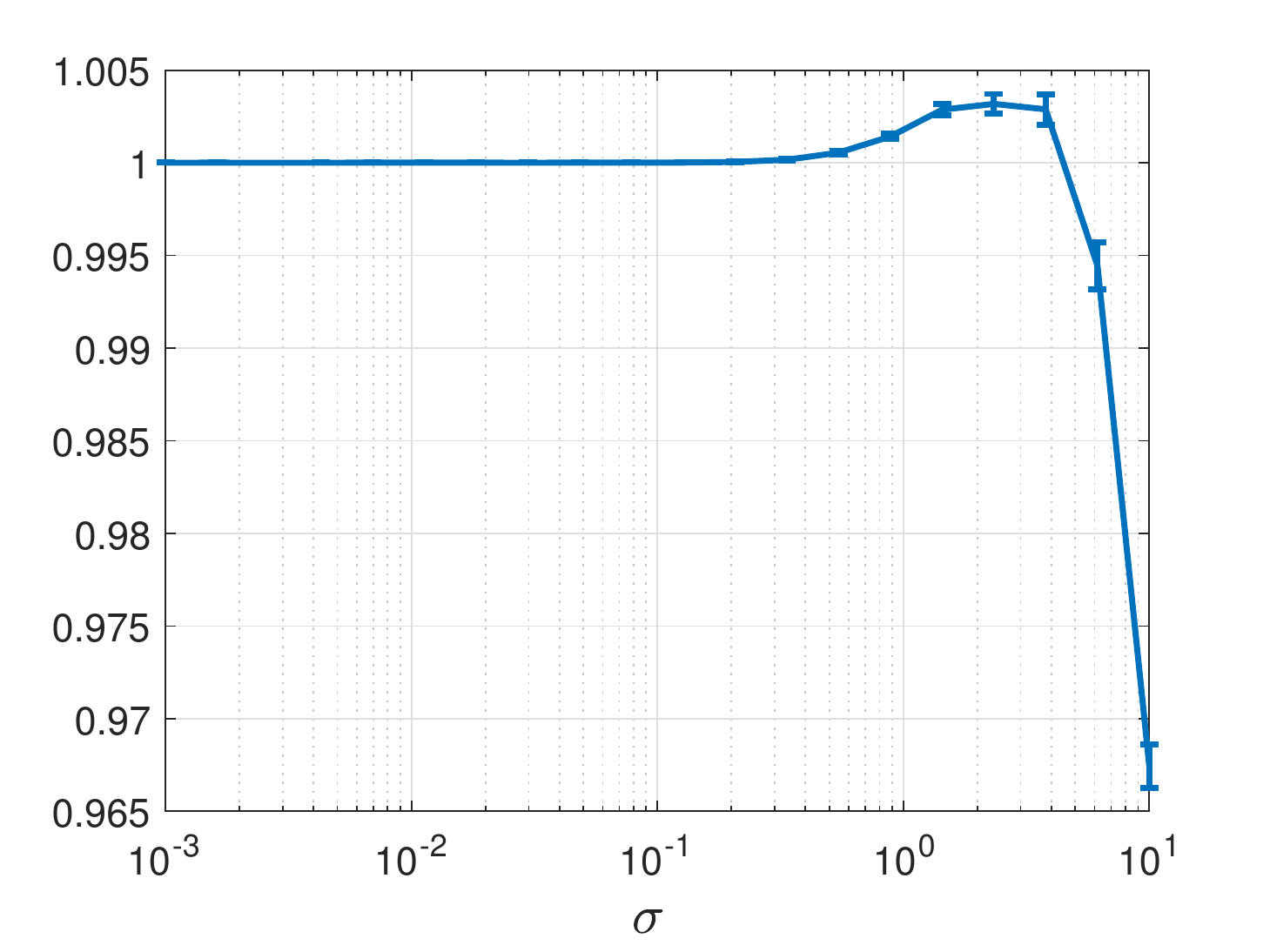}&   
		\includegraphics[width=0.3\linewidth,trim =.4cm 0cm .8cm 0cm, clip = true]{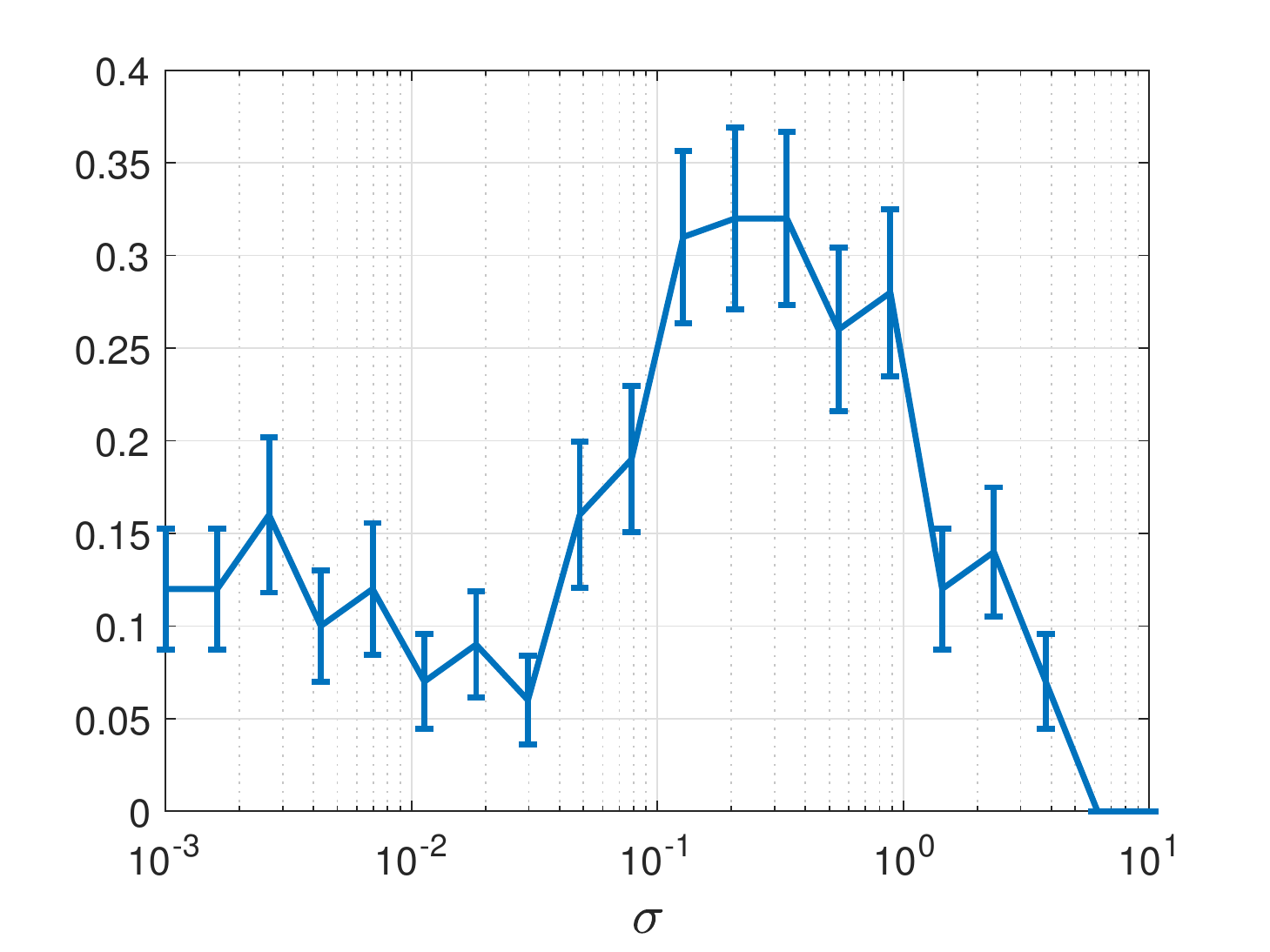}    &
		\includegraphics[width=0.3\linewidth,trim =.4cm 0cm .8cm 0cm, clip = true]{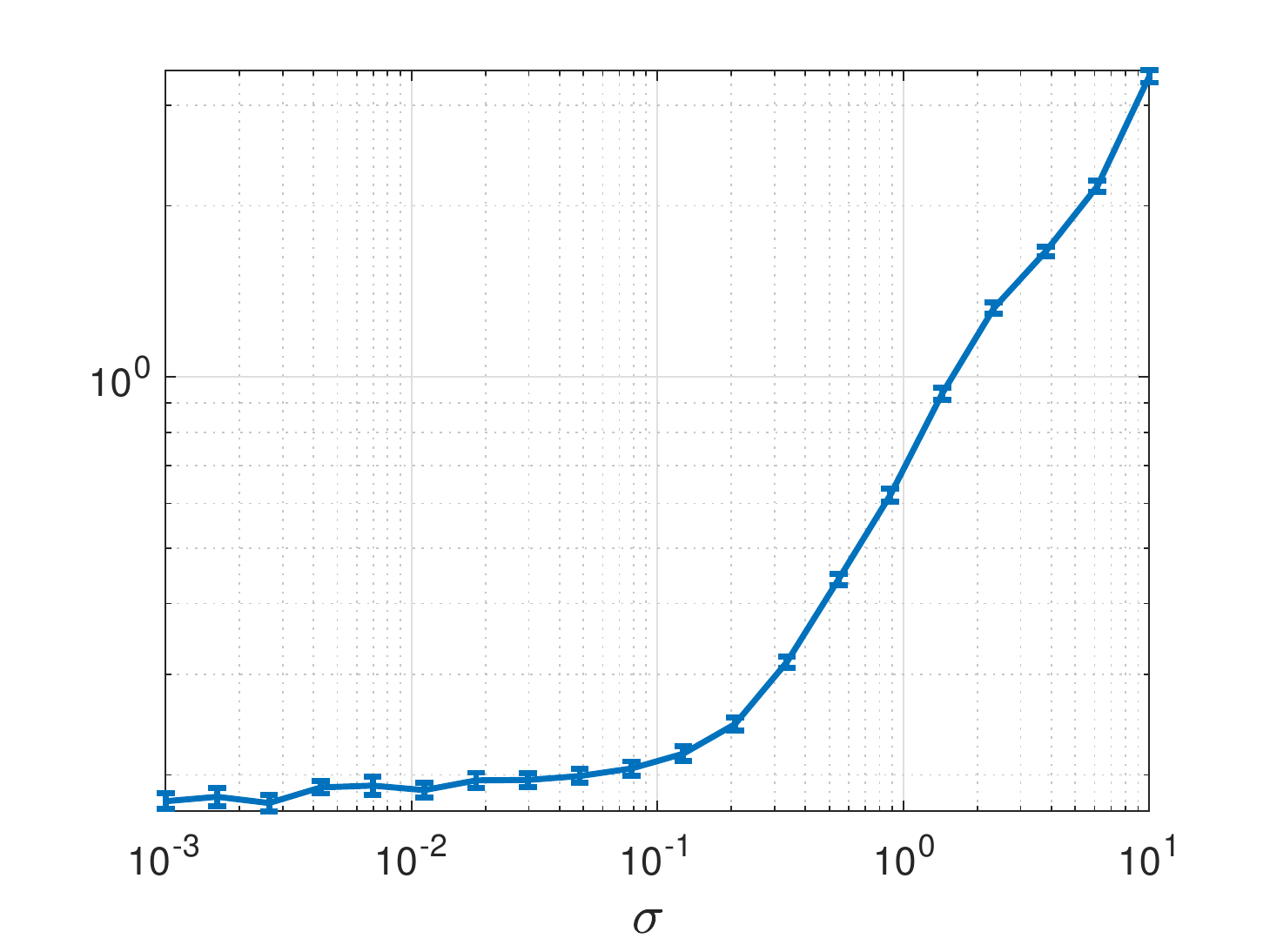} \\ 
	\end{tabular}
	\caption{\small Results for the full-batch linear regression. Left: Ratio of the best error achieved by GD and SAM. Middle: The difference between the number of iterations leading to the best error Right: Best SAM error }
	\label{fig:app-lin}
\end{figure*}

To further examine our theory, we run the model described above for two values of $\sigma=0.05,1$ and plot the ratio of GD error to SAM error for all iteration in Figure~\ref{fig:app-lin-traj}. We see that when noise is smaller, SAM performs better than GD in all iterations. On the other hand, when the noise is high, SAM performs better in early iterations, as SAM's bias is lower, while as variance becomes dominant, GD starts to performs better than SAM in later iterations. Again, this is in agreement with our theory.

\begin{figure*}[t!]
	\centering
	\begin{tabular}{ccc}
		&$\sigma=0.05$ & $\sigma=1$ \\
		\rotatebox{90}{~~~~~$\errortest^{\gd}/\errortest^{\sam}$} &   \includegraphics[width=0.3\linewidth,trim =0.4cm 0cm .8cm 0cm, clip = true]{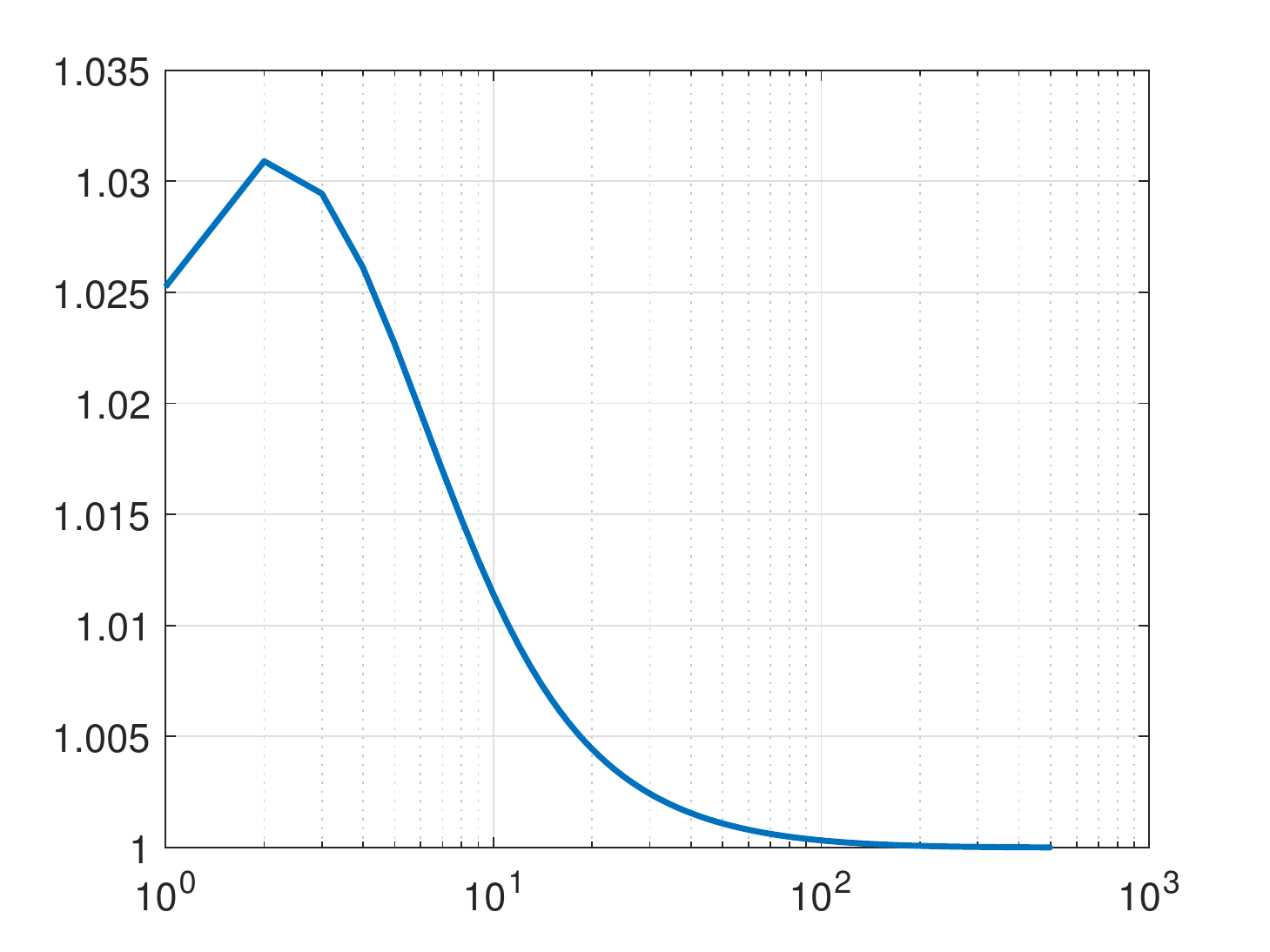}&   
		\includegraphics[width=0.3\linewidth,trim =.4cm 0cm .8cm 0cm, clip = true]{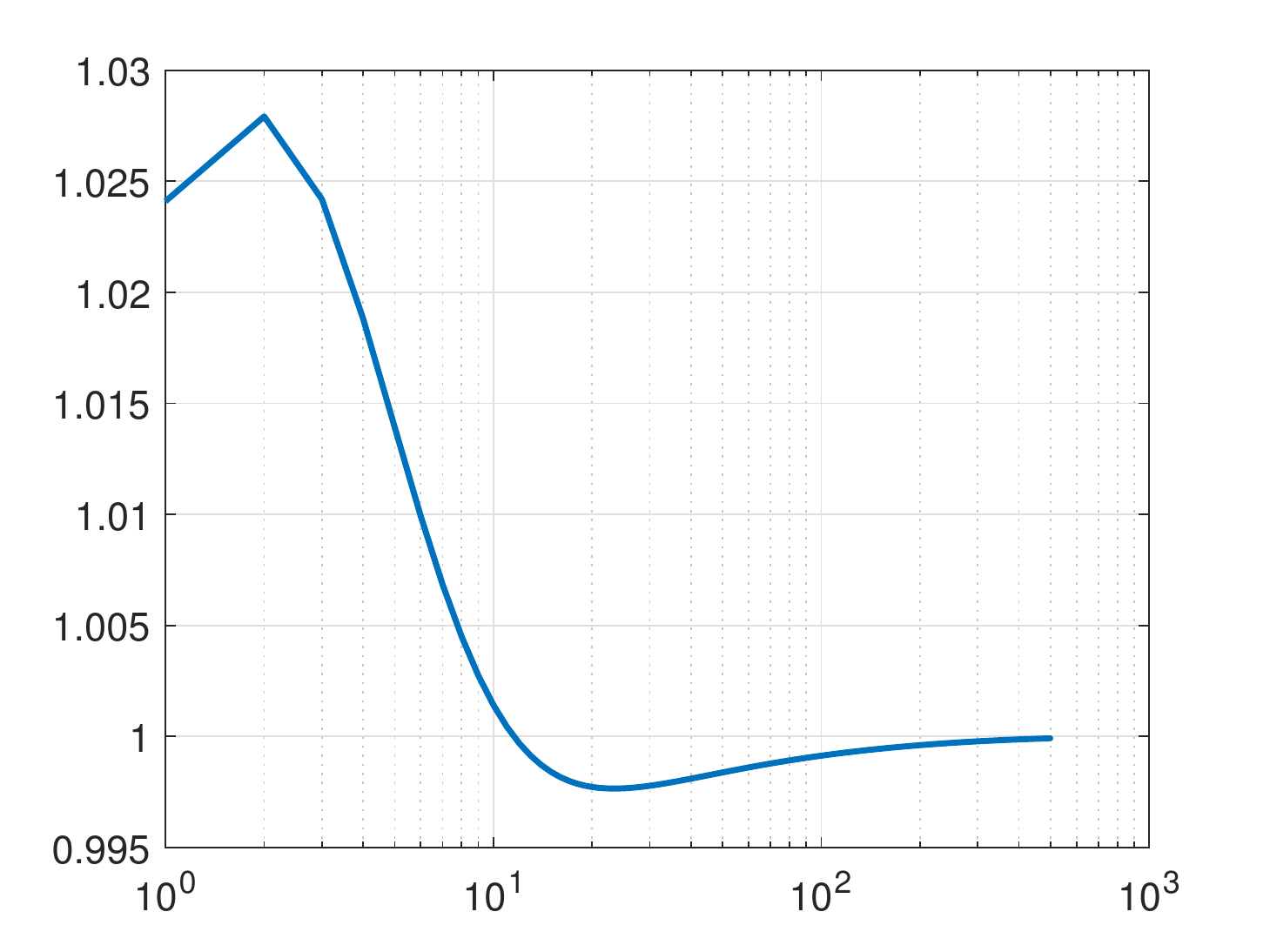}  \\
		& \# of iterations & \# of iterations
	\end{tabular}
	\caption{\small Comparison of ratio of error for GD and SAM over iterations, $\errortest(\B{w}_k^{\gd})/\errortest(\B{w}_k^{\sam})$ for the full-batch linear regression}
	\label{fig:app-lin-traj}
\end{figure*}
\subsubsection{Effect of stochasticity}
Next, we explore the effect of stochasticity in the algorithm. We use the same model from the previous section. However, we set $n=200$ and draw 2000 validation samples. We run the algorithm for an epoch, with batch size of 1. The results for this case are shown in Figure~\ref{fig:app-lin-stoc}. We see from the left panel that in this case, even for large noise variance values, stochastic SAM performs better than SGD. As we discussed in Appendix~\ref{sec:stochastic}, stochastic SAM can have stronger regularization over SGD due to the way SAM is implemented in practice.
\begin{figure*}[t!]
	\centering
	\begin{tabular}{cc}
		$\errortest^{\gd}/\errortest^{\sam}$  &  $\errortest^{\sam}$ \\
		\includegraphics[width=0.3\linewidth,trim =0.4cm 0cm .8cm 0cm, clip = true]{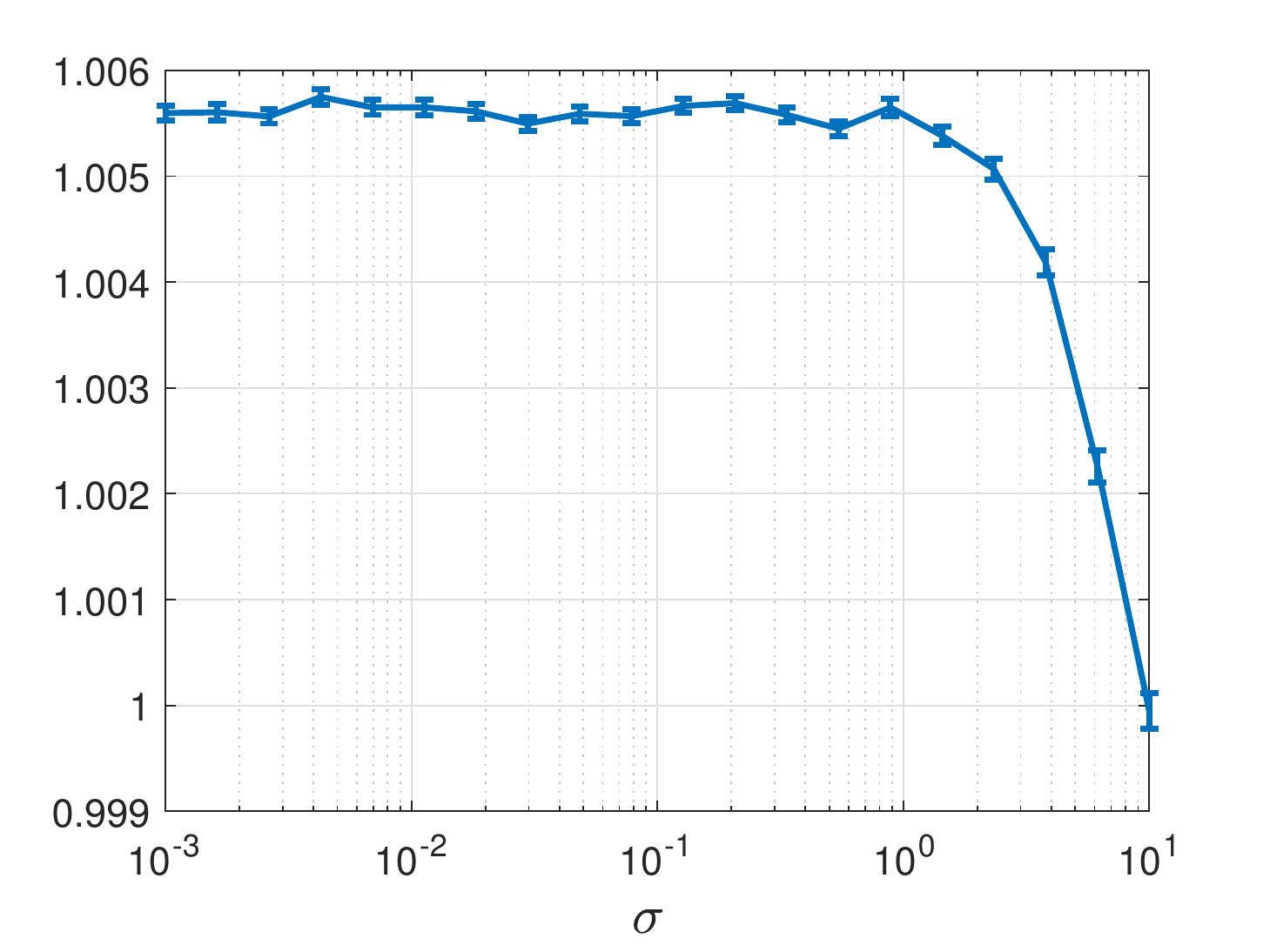}&   
		\includegraphics[width=0.3\linewidth,trim =.4cm 0cm .8cm 0cm, clip = true]{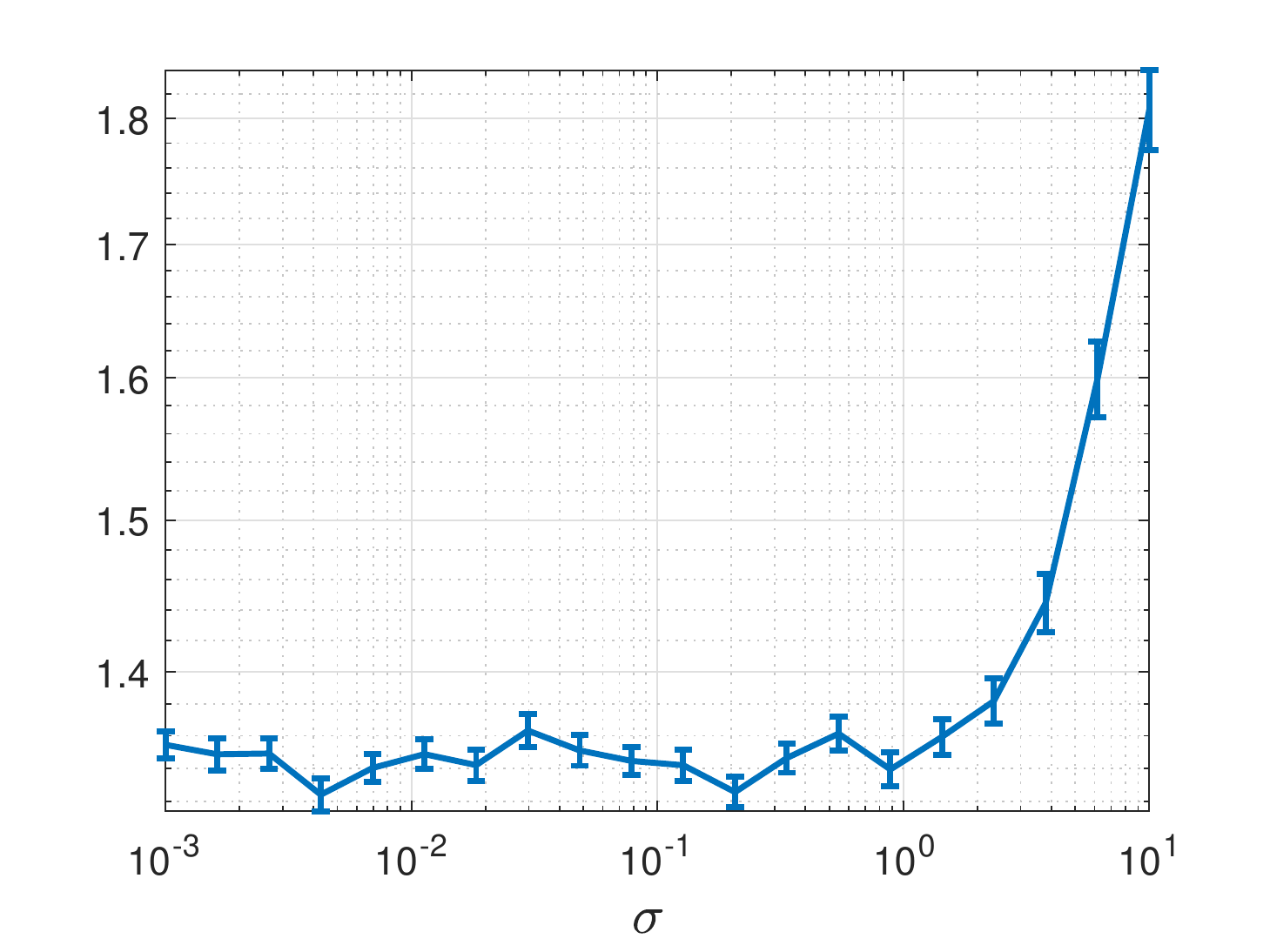} \\ 
	\end{tabular}
	\caption{\small Results for the stochastic linear regression. Left: Ratio of the best error achieved by SGD and stochastic SAM. Right: Best SAM error}
	\label{fig:app-lin-stoc}
\end{figure*}

\subsection{Kernel Regression Experiments}
Next, we investigate the effect on SAM for the kernel regression case. To this end, we set $p=200,n=100$ and draw $\B{X}$ as described in Section~\ref{app:linear-exp}. Next, we expand the set of features by adding terms of the form $x_i^2,x_i^3$ and $x_ix_j$ for 400 pairs of randomly chosen $(i,j)$, overall increasing the dimension of the model to 1000. Then, the observation are drawn as $y=\tilde{\B{x}}^T\B{\bar{w}}+\epsilon$ where $\tilde{\B{x}}$ is the vector of expanded features and $\bar{\B{w}}$ has iid uniform coordinates and then is normalized to have norm 1. The rest of the setup is similar to Section~\ref{app:linear-exp}. We use an indefinite kernel, defined as the difference of a Gaussian kernel with the variance of 100 and an exponential kernel,
$$K(\B{x},\B{y})=\exp\left(-\frac{\|\B{x}-\B{y}\|_2^2}{200}\right)-0.8\exp(-\|\B{x}-\B{y}\|_2)$$
and run the training for 1500 iterations. The whole process is repeated 100 times and then averaged. We use 200 validation points and record the error, defined as 
$$\errortest(h)=\frac{1}{200}\sum_{i=1}^{200}(y_{\text{test},i}-h(\B{x}_{\text{test},i}))^2.$$
Similar to the case of linear regression, we report the best error over the course of the algorithm and the iteration leading to the best error. The results for this case are shown in Figure~\ref{fig:app-kernel}. Overall, we see that GD and SAM perform closely in terms of error. Similar to the case of linear regression, unless noise is too large, SAM performs better than GD, although the improvements are marginal. However, we see that the best error for SAM is achieved earlier than GD, which overall agrees with the insight from our bias-variance analysis, as SAM's bias is smaller than GD and in noiseless cases, bias is dominant. We also show the SAM performance in Figure~\ref{fig:app-kernel} [Right Panel], showing increasing noise variance leads to worse performance.

\begin{figure*}[t!]
	\centering
	\begin{tabular}{ccc}
		$\errortest^{\gd}/\errortest^{\sam}$  & $k^{\gd}-k^{\sam}$ & $\errortest^{\sam}$ \\
		\includegraphics[width=0.3\linewidth,trim =0.4cm 0cm .8cm 0cm, clip = true]{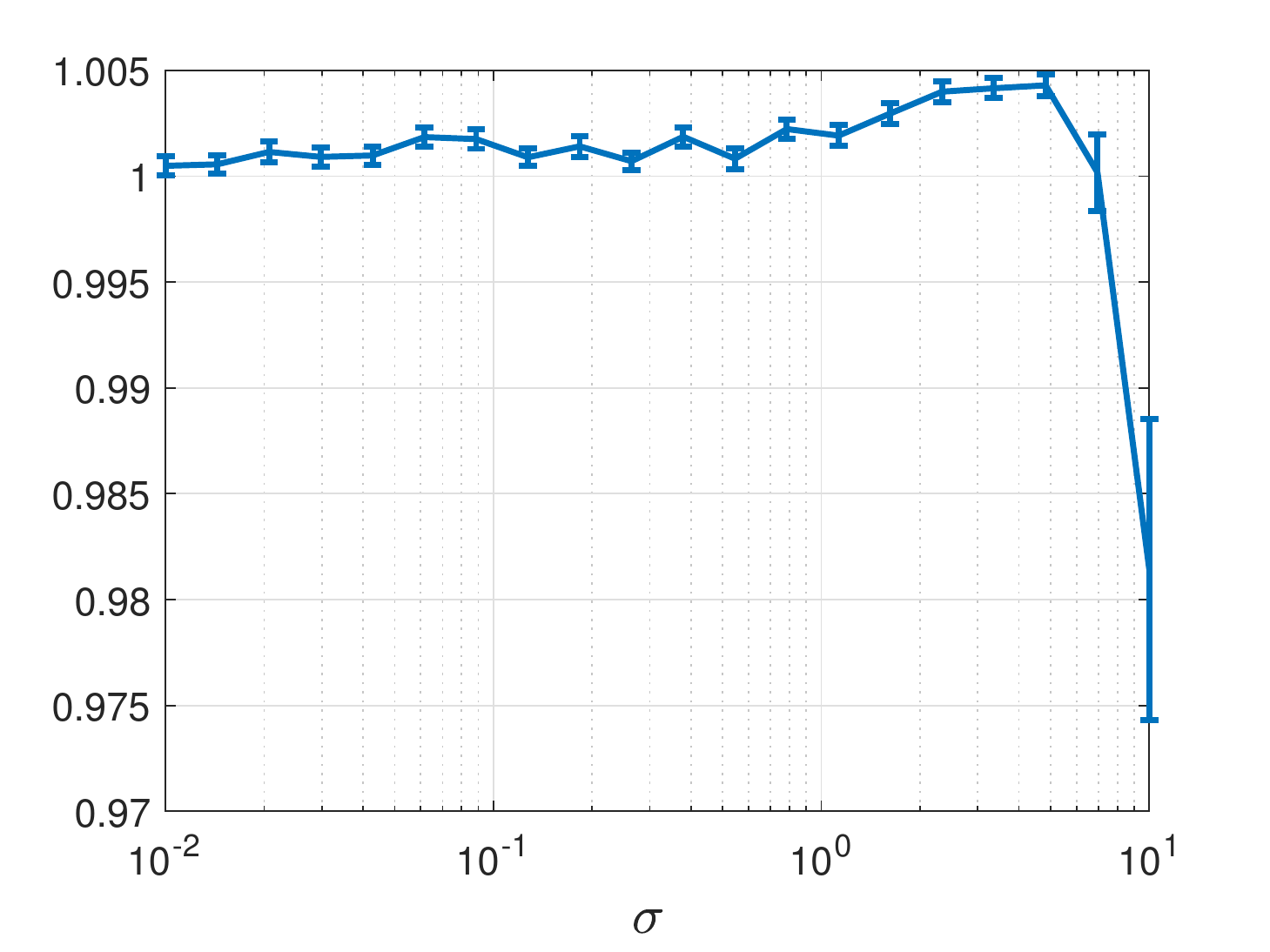}&   
		\includegraphics[width=0.3\linewidth,trim =.4cm 0cm .8cm 0cm, clip = true]{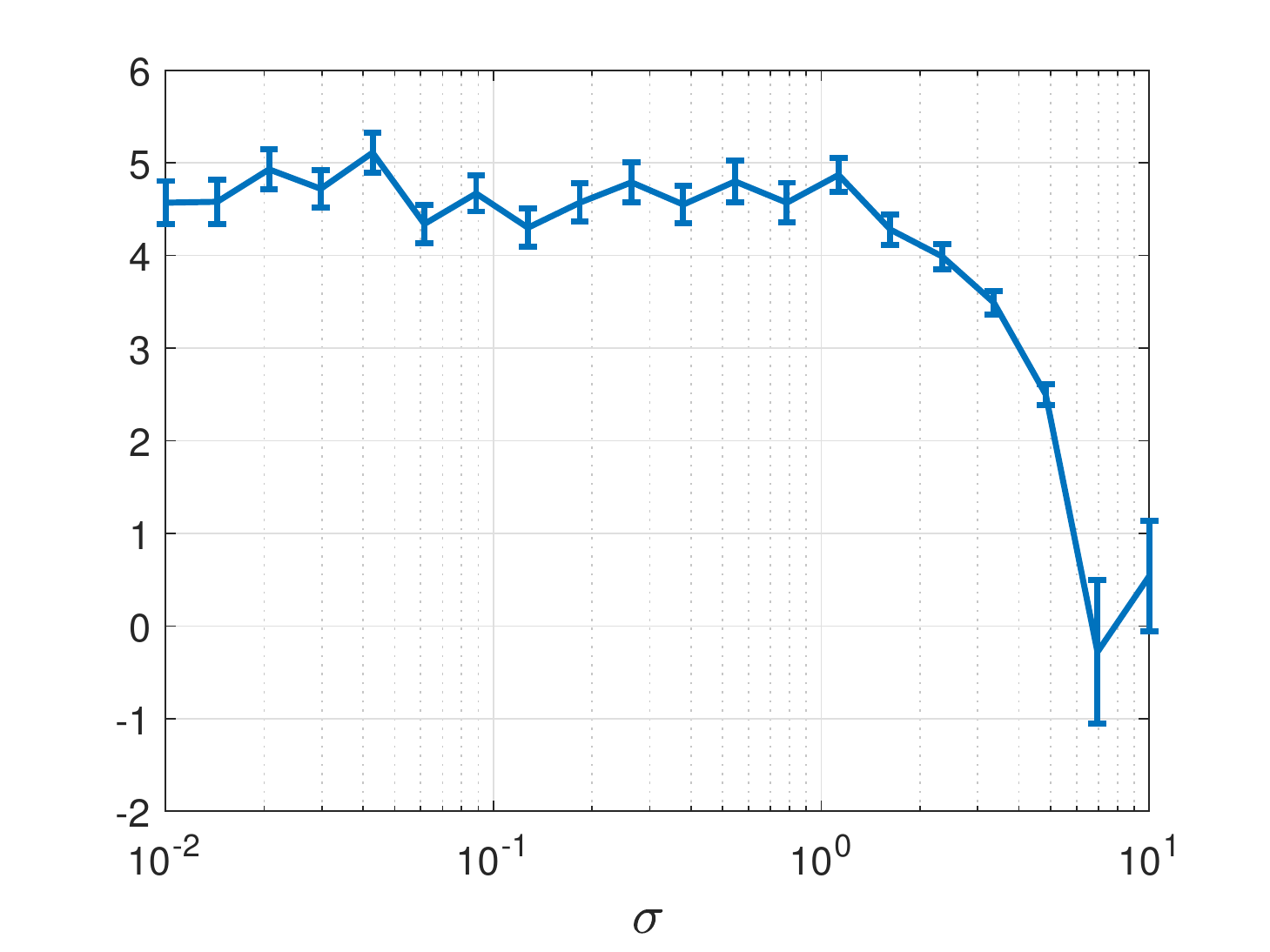}    &
		\includegraphics[width=0.3\linewidth,trim =.4cm 0cm .8cm 0cm, clip = true]{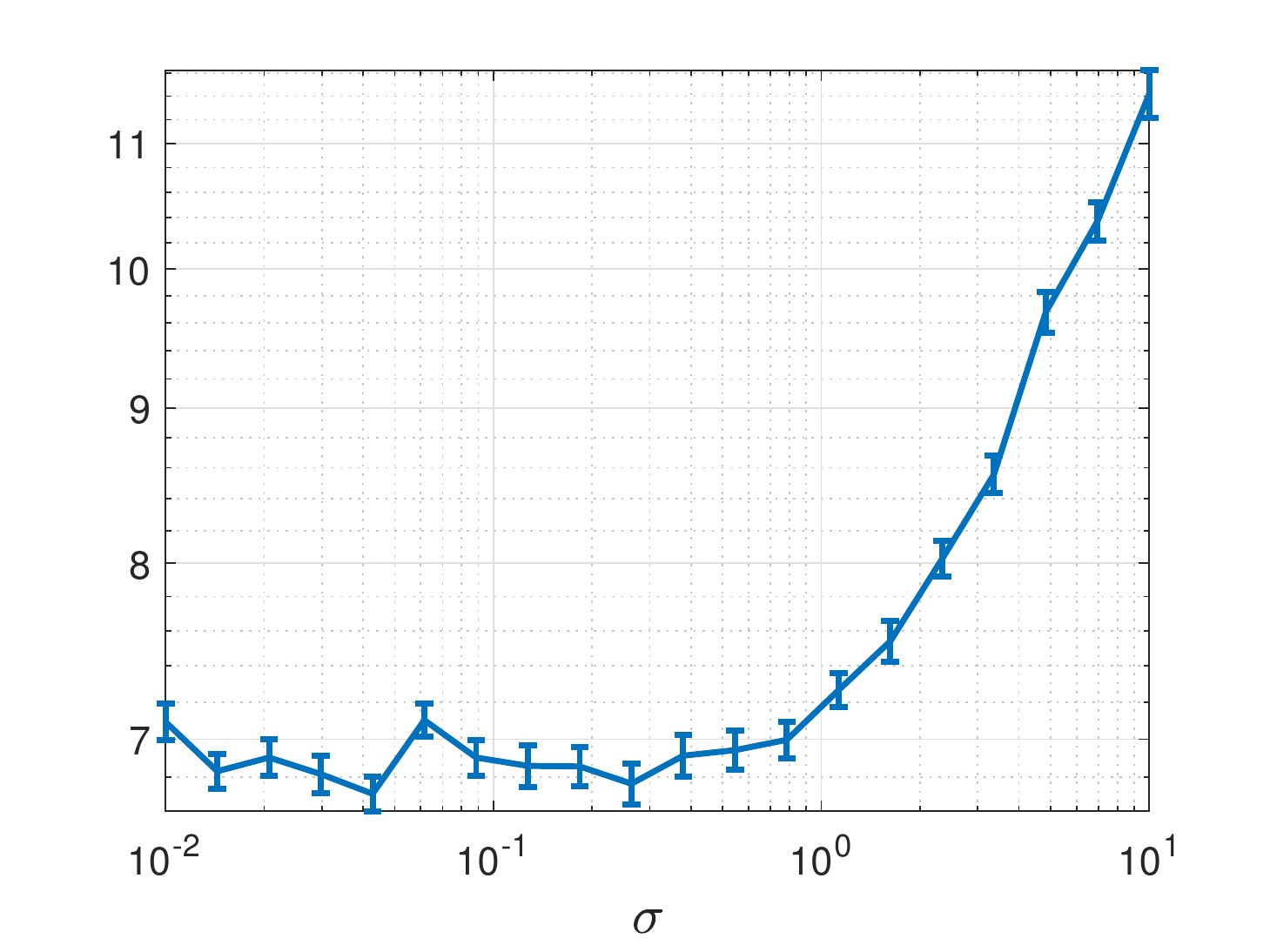} \\ 
	\end{tabular}
	\caption{\small Results for the kernel regression. Left: Ratio of the best error achieved by GD and SAM. Middle: The difference between the number of iterations leading to the best error Right: Best SAM error}
	\label{fig:app-kernel}
\end{figure*}
Finally, we compare the error trajectory of SAM/GD for two values of noise in Figure~\ref{fig:app-kernel-traj}. We see that in almost all iterations, SAM performs better than GD, which agrees with Proposition~\ref{kernel-gen-prop}. Specially, we see that in later iterations, GD performs significantly worse which agrees with our analysis that GD has unbounded error in the non-convex case.

\begin{figure*}[t!]
	\centering
	\begin{tabular}{ccc}
		&$\sigma=0.3$ & $\sigma=3$ \\
		\rotatebox{90}{~~~~~$\errortest^{\gd}/\errortest^{\sam}$} & \includegraphics[width=0.3\linewidth,trim =0.4cm 0cm .8cm 0cm, clip = true]{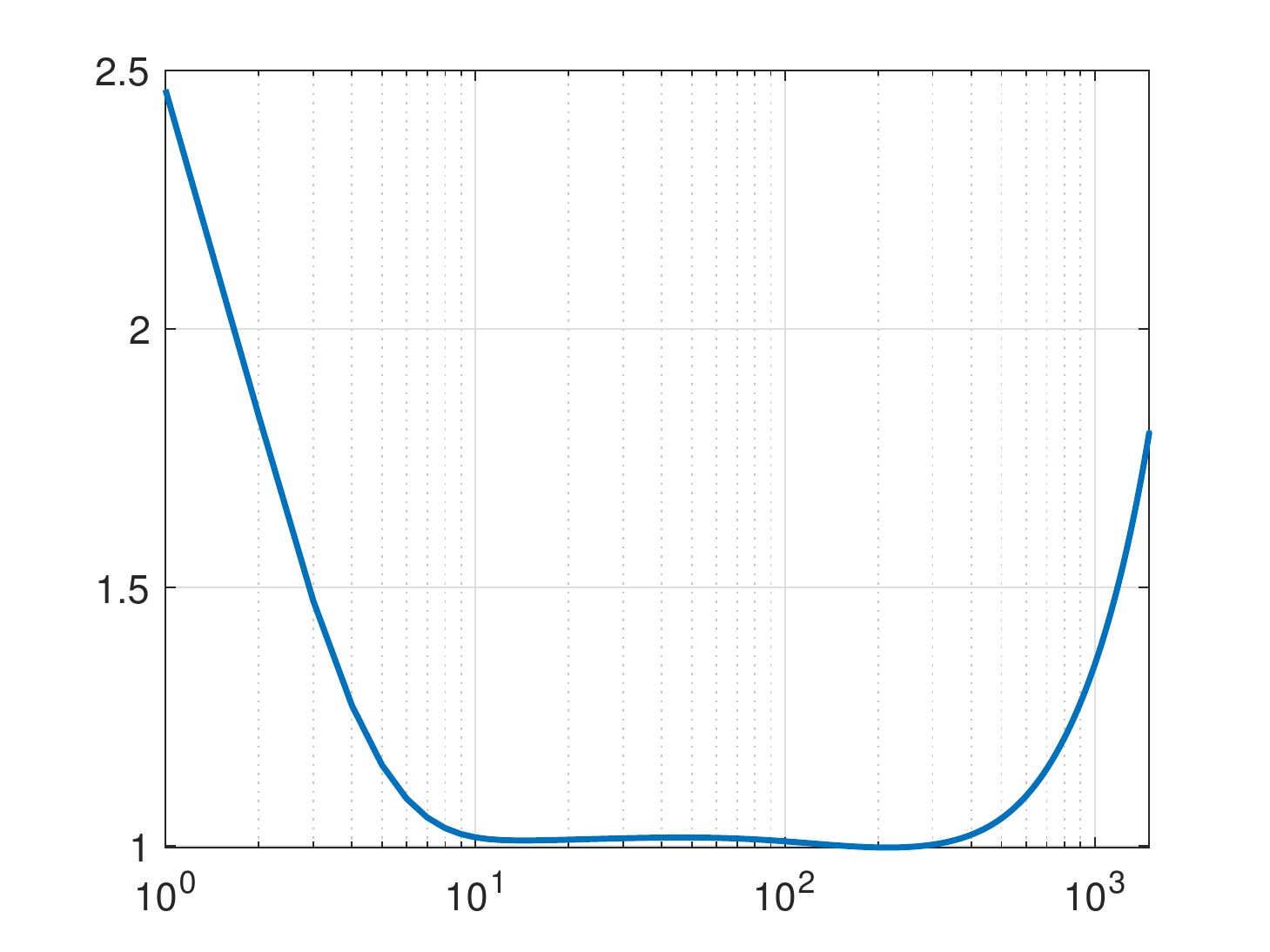}&   
		\includegraphics[width=0.3\linewidth,trim =.4cm 0cm .8cm 0cm, clip = true]{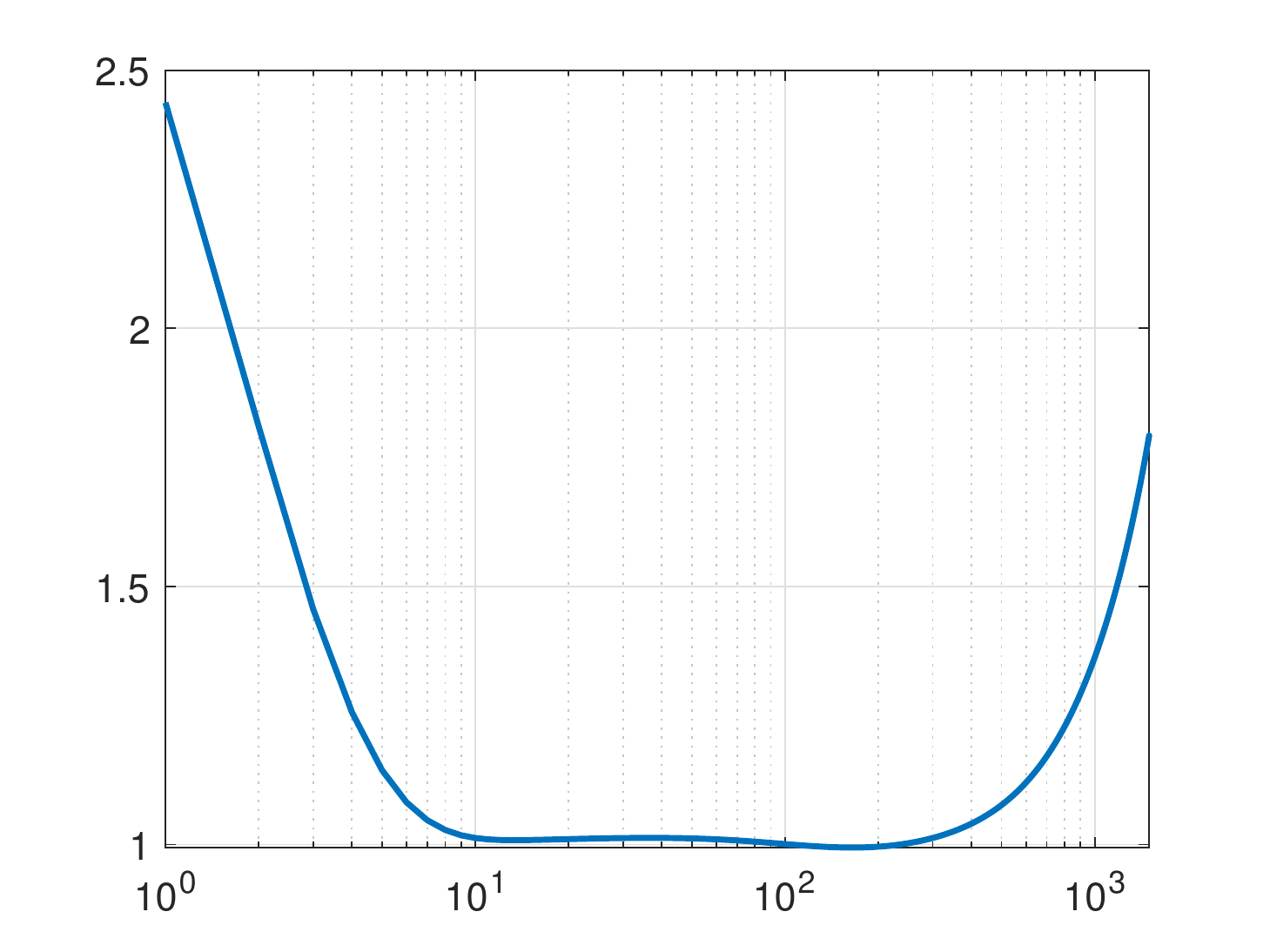}  \\
		& \# of iterations & \# of iterations
	\end{tabular}
	\caption{\small Comparison of ratio of error for GD and SAM over iterations, $\errortest(\B{w}_k^{\gd})/\errortest(\B{w}_k^{\sam})$ for the kernel regression}
	\label{fig:app-kernel-traj}
\end{figure*}

\subsection{Deep Learning Experiments}
\subsubsection{Experimental setup}
Our deep learning experiments are done on MIT Supercloud cluster~\citep{supercloud} using a single NVIDIA V100 GPU. In all experiments, we use batch size of 256, with the peak starting learning rate of 0.35 scheduled to decay linearly. We train networks for 200 epochs, unless stated otherwise. We use momentum coefficient of 0.9 and weight decay coefficient of 0.0005 in all experiments. We run all experiments for three repetitions and report the average and standard deviation.

\subsubsection{Comparison of SGD and SAM}
First, we repeat our experiments from Section~\ref{sec:Numerical} on ResNet18. In Figure~\ref{fig:varyrho18} we explore the effect of $\rho$ on the accuracy. We see that large or small values of $\rho$ lead to worse performance, as expected. However, we see that as we discussed in Section~\ref{sec:Numerical}, the loss of accuracy from large values of  $\rho$ is less than loss of accuracy when taking $\rho=0$, i.e. GD. Specially, in CIFAR100, the loss of accuracy is negligible by going from $\rho=0.3$ to $\rho=0.5$. A similar situation can be observed from Figure~\ref{fig:varyrho18} for the ResNet50 network. As we discussed, this aligns with our theory for the noiseless case.

Next, we take a look at error trajectory for noiseless and noisy setups for ResNet18 in Figure~\ref{fig:vs18}. As we see, the profiles of error are similar to the ones from Section~\ref{sec:Numerical}. Particularly, we see error decreases over epochs for the noiseless case, while start to increase in the noisy setup. Therefore, our insight regarding early stopping in the noisy case holds true here as well. We also see the increasing performance gap between SAM and SGD for the noisy setup, which as we discussed, can be explained by the increase of variance of GD in the non-convex setup.

\begin{figure}[t!]
	\centering
	\begin{tabular}{cccc}
		CIFAR10 (ResNet50) & CIFAR100 (ResNet50) & CIFAR10 (ResNet18) & CIFAR100 (ResNet18)\\
		\includegraphics[width=0.2\linewidth,trim =0.8cm 0cm .8cm 0cm, clip = true]{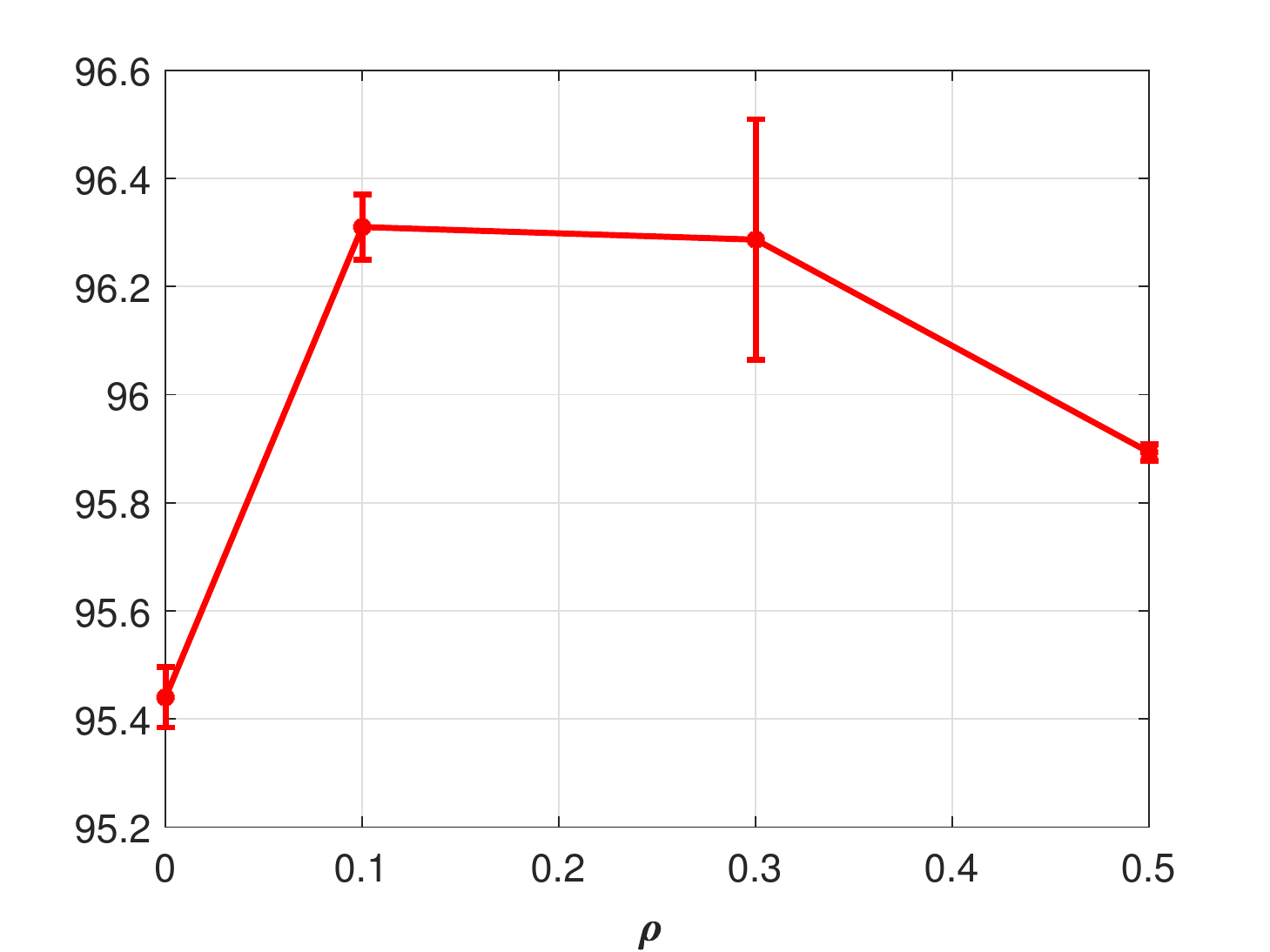}&   
		\includegraphics[width=0.2\linewidth,trim =.8cm 0cm .8cm 0cm, clip = true]{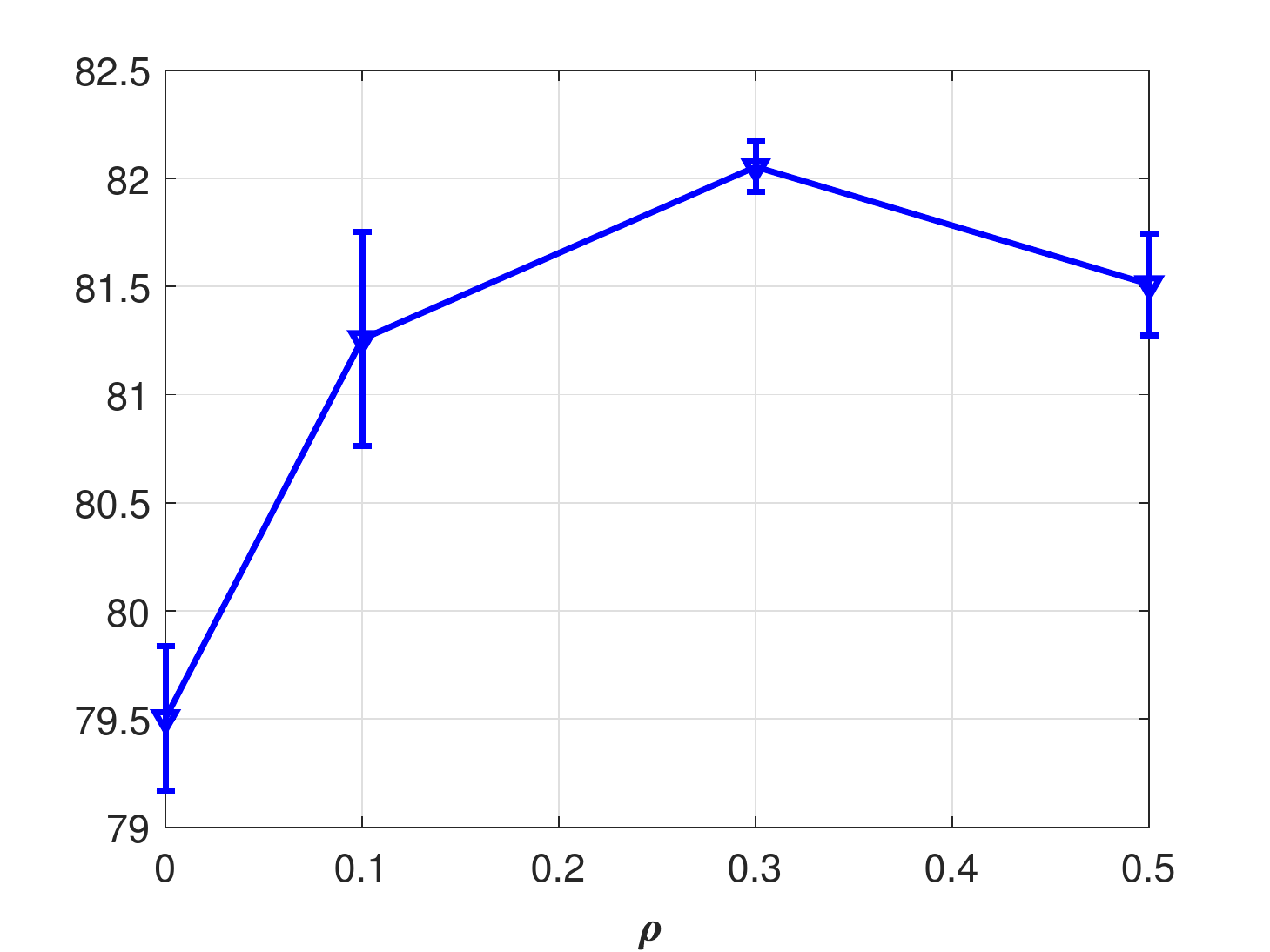}    &
		\includegraphics[width=0.2\linewidth,trim =0.8cm 0cm .8cm 0cm, clip = true]{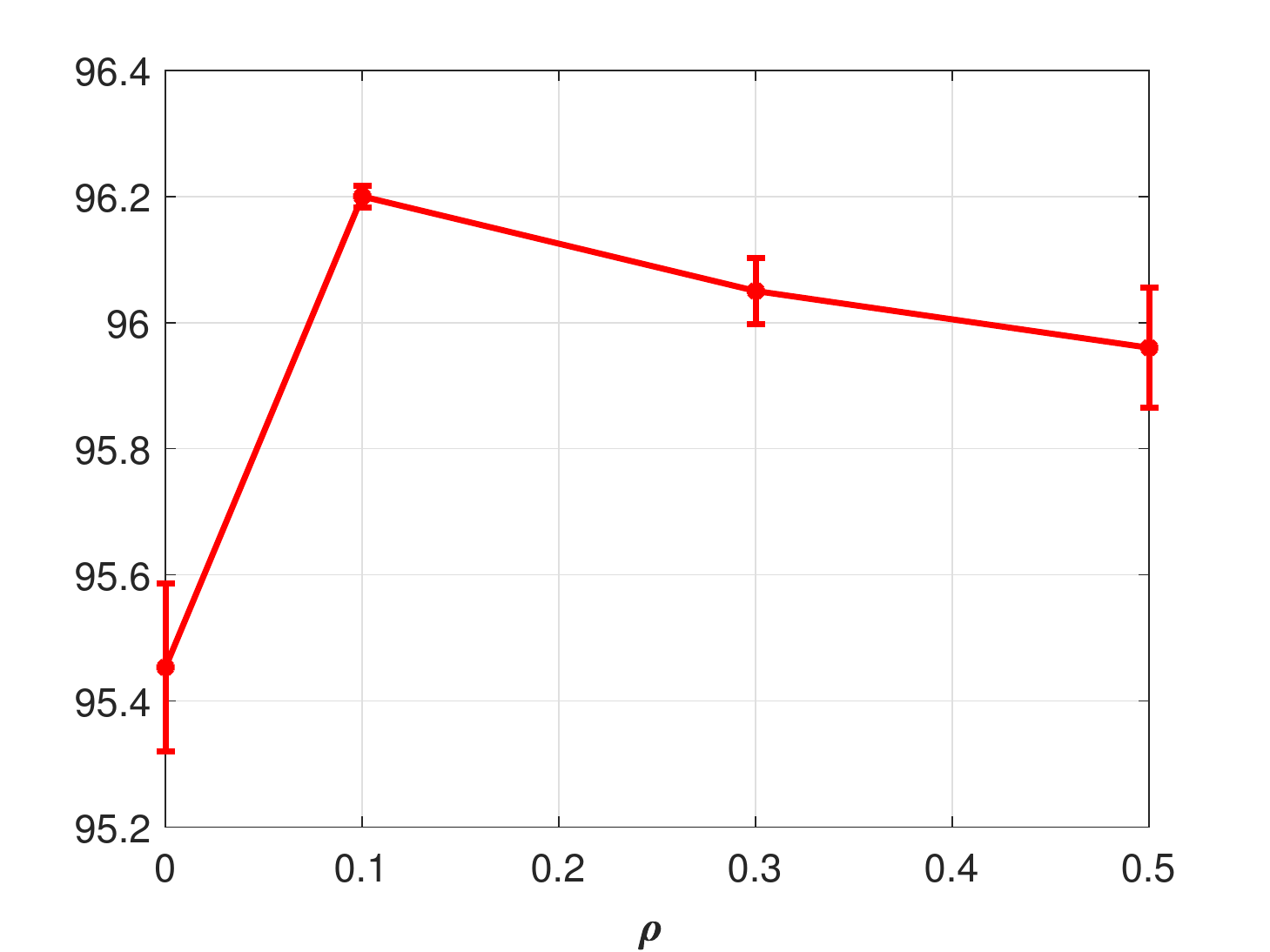}&   
		\includegraphics[width=0.2\linewidth,trim =.8cm 0cm .8cm 0cm, clip = true]{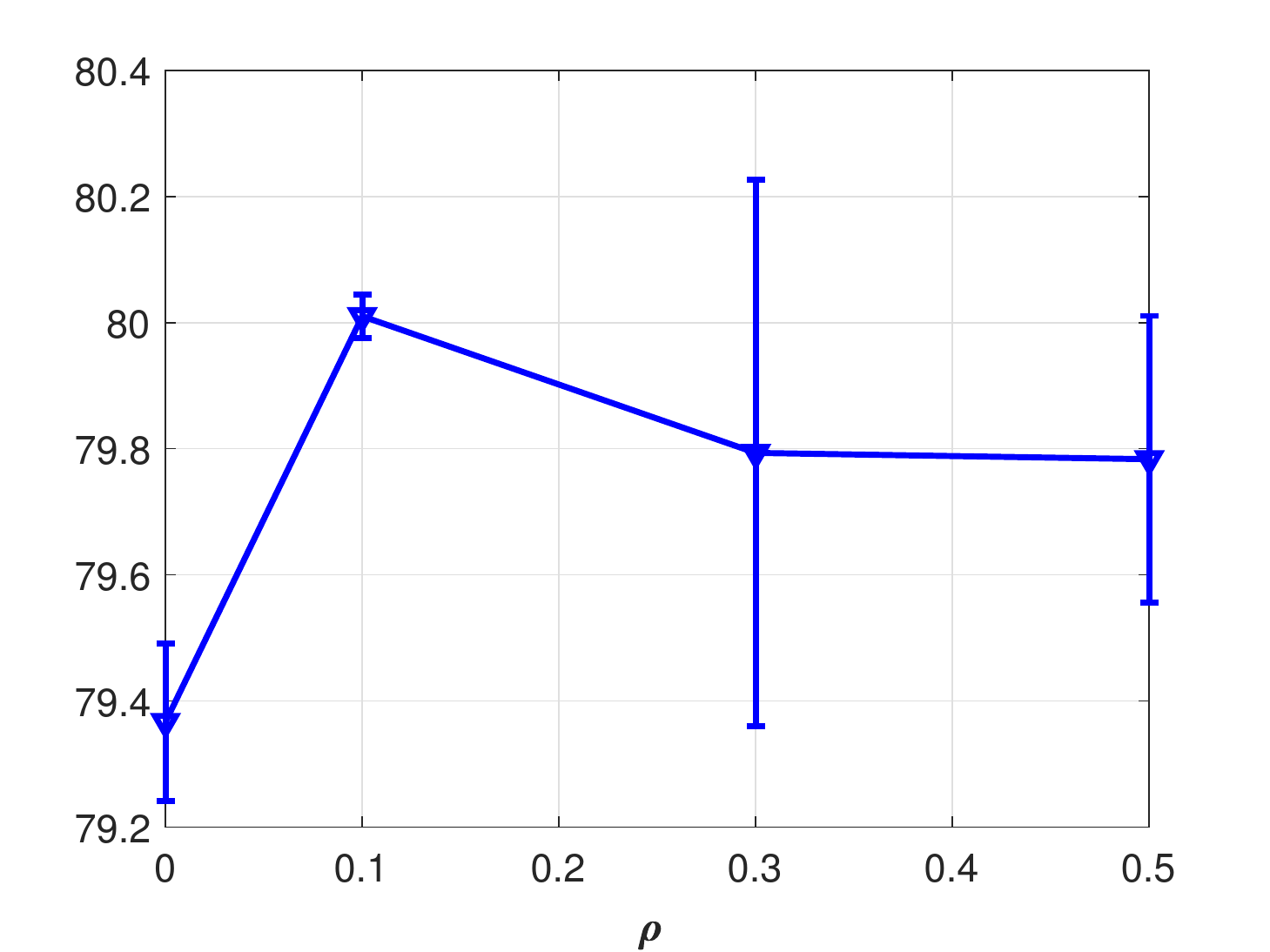} 
	\end{tabular}
	\caption{\small Effect of varying $\rho$ on accuracy on ResNet50 and ResNet18.   }
	\label{fig:varyrho18}
\end{figure}

\begin{figure}[t!]
	\centering
	\begin{tabular}{cccc}
		CIFAR10 & CIFAR100 & CIFAR10-random & CIFAR10-worse\\
		\includegraphics[width=0.2\linewidth,trim =0.8cm 0cm .8cm 0cm, clip = true]{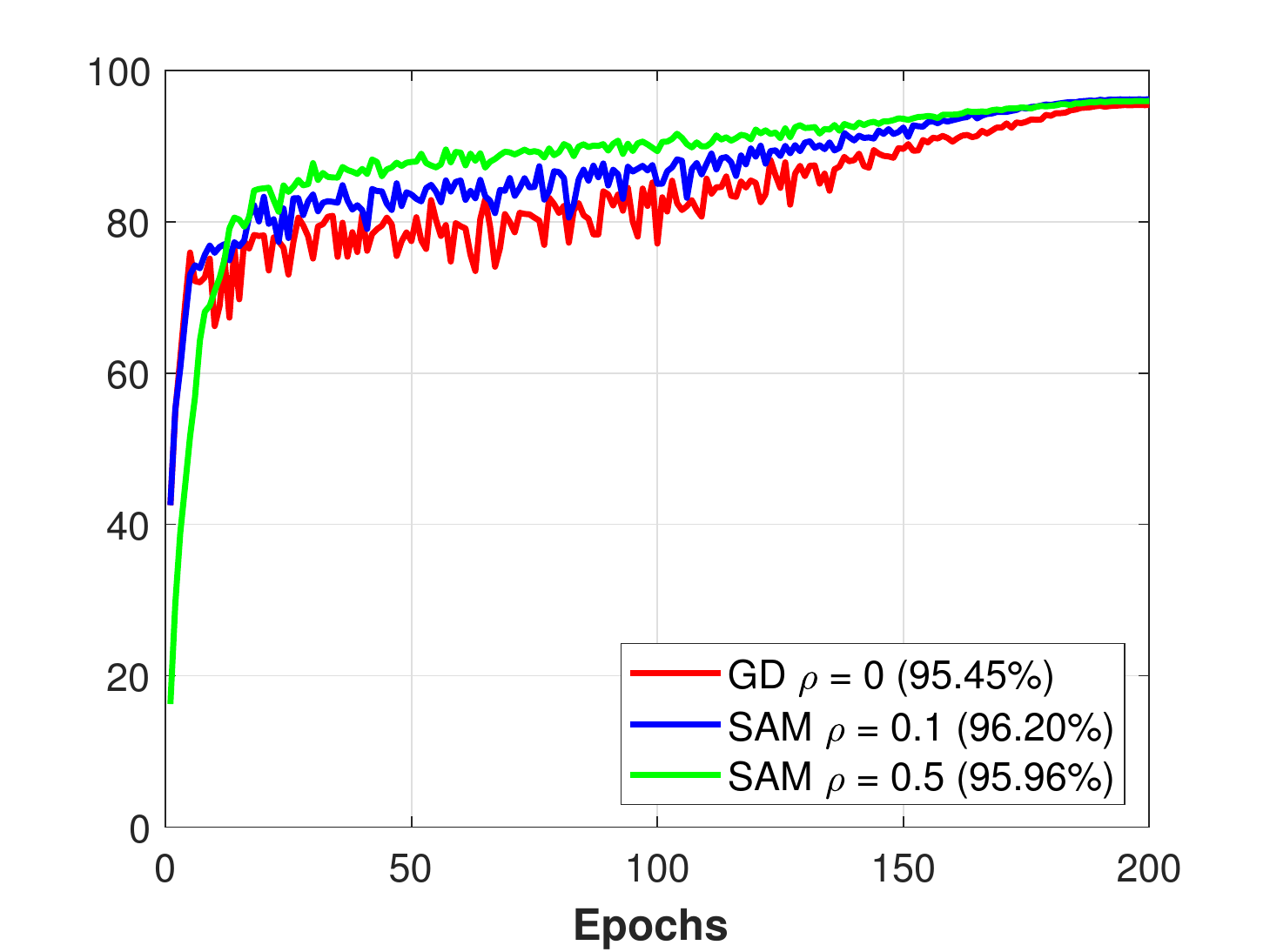}&   
		\includegraphics[width=0.2\linewidth,trim =.8cm 0cm .8cm 0cm, clip = true]{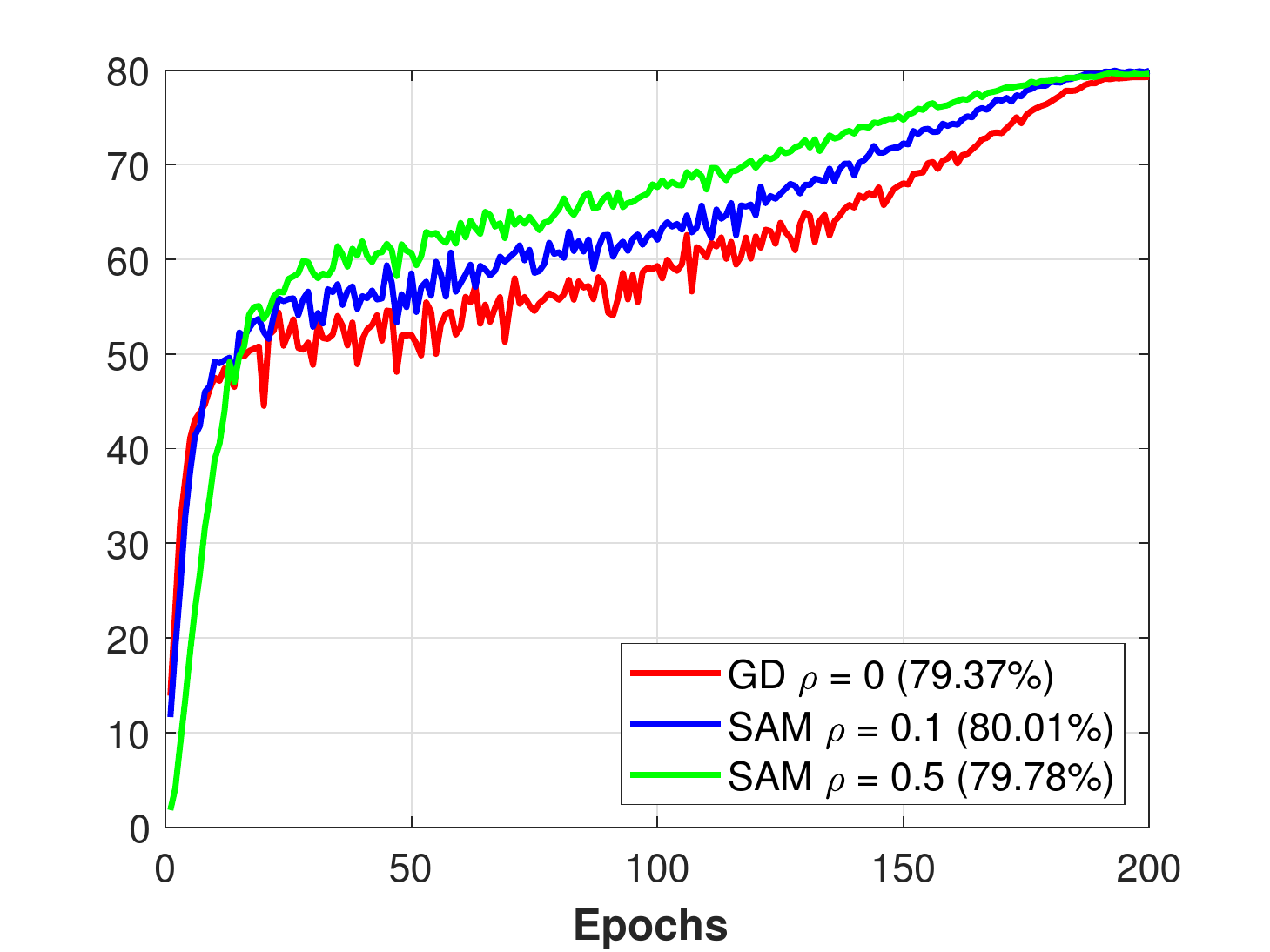}   &
		\includegraphics[width=0.2\linewidth,trim =0.8cm 0cm .8cm 0cm, clip = true]{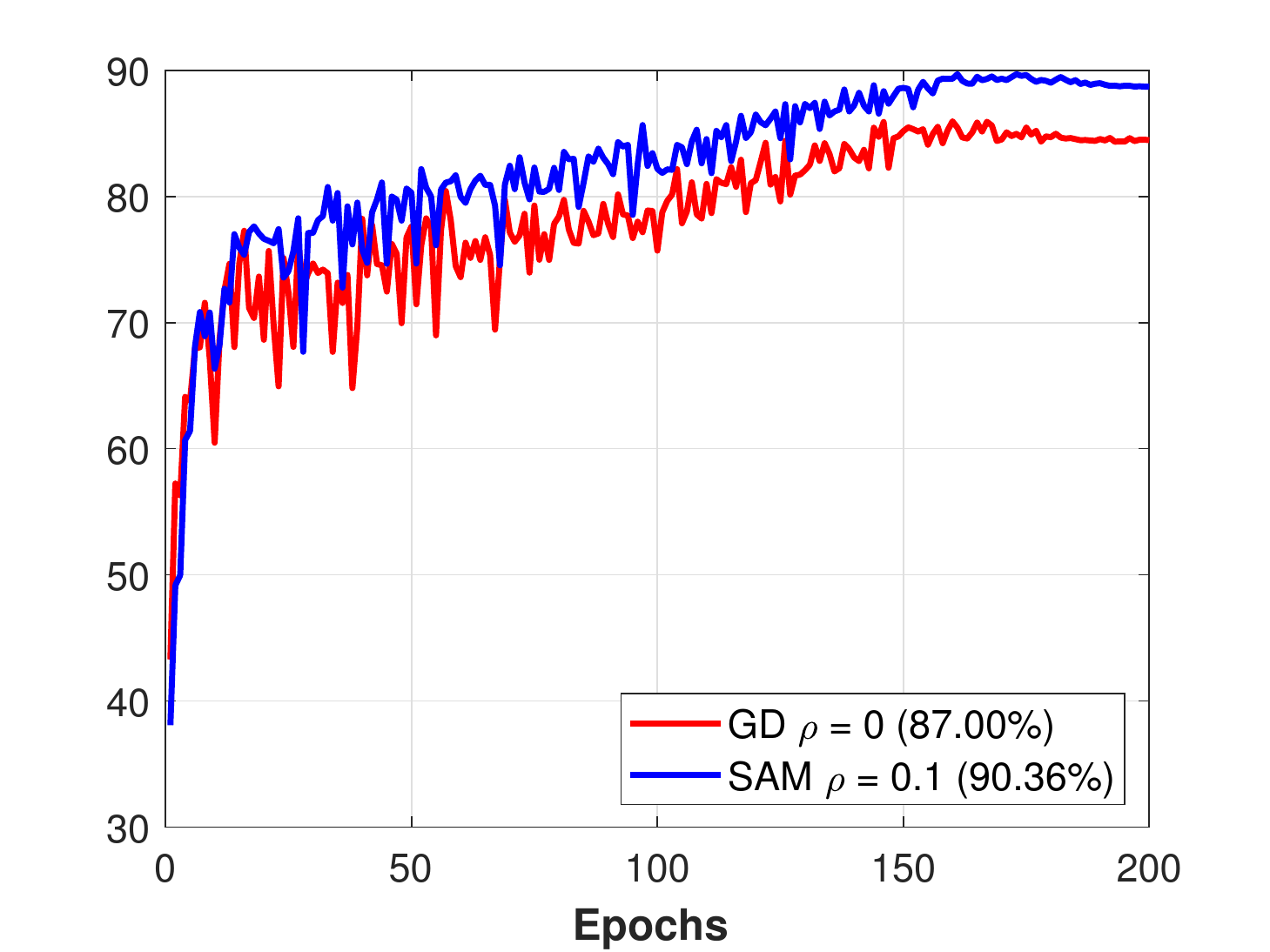}&   
		\includegraphics[width=0.2\linewidth,trim =.8cm 0cm .8cm 0cm, clip = true]{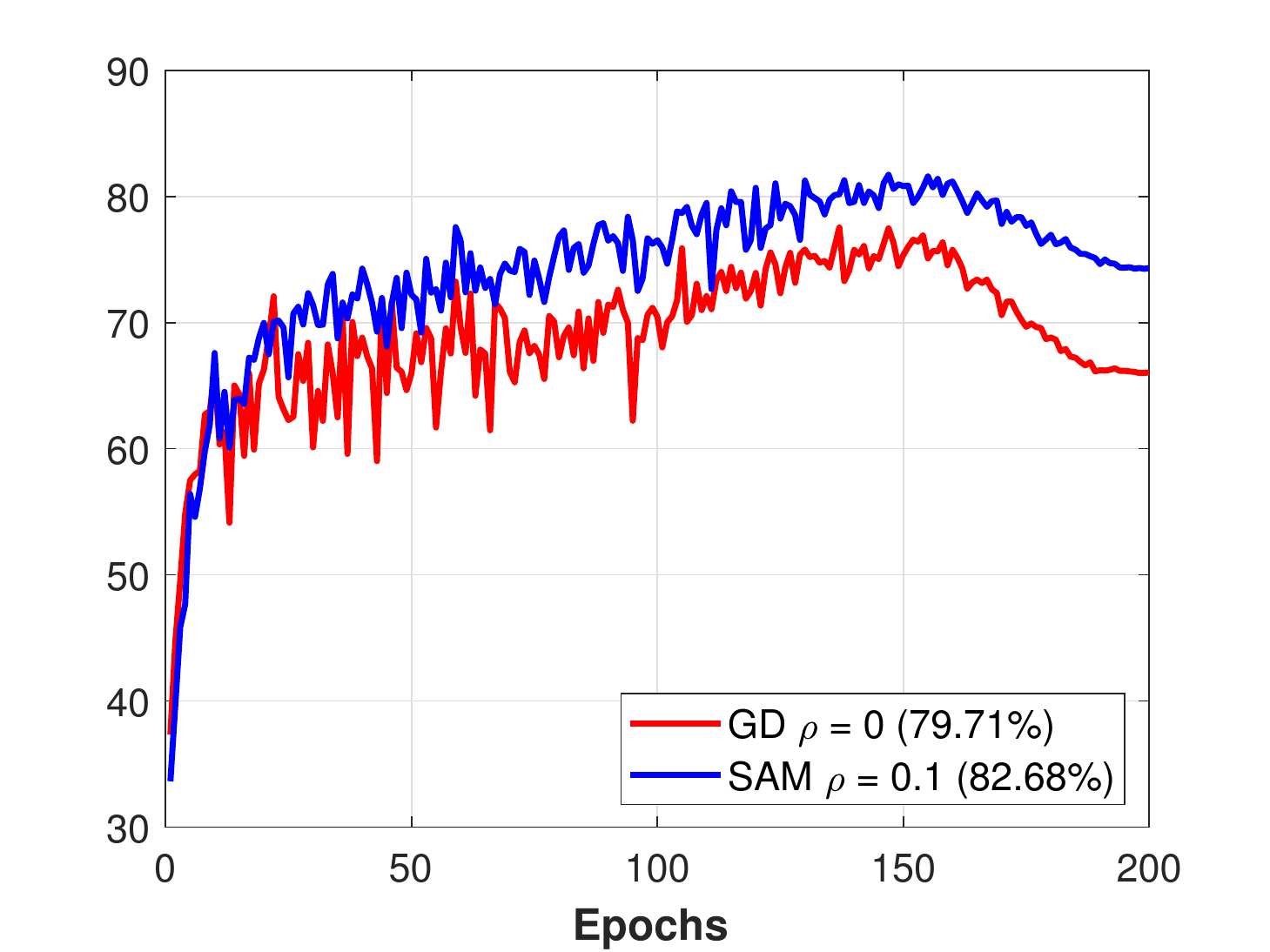}   
	\end{tabular}
	\caption{\small  Accuracy over epochs for SAM and SGD (ResNet18). The number in the parenthesis in the legend gives the average best accuracy.}
	\label{fig:vs18}
\end{figure}


\subsubsection{Effect of decaying $\rho$}
Based on our observations so far, we see that choosing large values of $\rho$ improves the performance in early epochs, while resulting in worse accuracy in later iterations. Therefore, our hypothesis is that starting the training with a large $\rho$ and then decaying $\rho$ might result in better performance. To test this, we start the training with a value of $\rho$ in our search grid of $\{0.1,0.3,0.5\}$ that is one step larger than the optimal $\rho$. For example, if the optimal $\rho$ is 0.1, we start training by $\rho=0.3$. Then, we decay $\rho$.
We consider two decay patterns: The value of $\rho$ is multiplied with a coefficient $\alpha\in\{0.7,0.8,0.9\}$ at either epoch 175, or epochs 150 and 200. We choose the best scenario among the 6 possible choices outlined above and report the results in Table~\ref{samdecay-complete-table} under the SAM-Decay row. SAM-Optimal shows the results for the best value of $\rho$. We run the training for full 200 epochs, but record the validation accuracy over the course of epochs. The early stopping column in Table~\ref{samdecay-complete-table} corresponds to the best accuracy over the first 120 epochs for SGD, and first 50 epochs for SAM-based methods (i.e. 100 forward-backward passes).  We can see that for the full training, decaying $\rho$ results in better or similar accuracy compared to the optimal $\rho$. This is while using the larger value of $\rho$ for full training leads to worse performance. Moreover, we see that the optimal value of $\rho$ generally performs worse than the decay case when training is stopped early, and using a large $\rho$ shows a considerable improvement in early epochs. As discussed, this is in agreement with the insights from our theory, and our observations above.

\begin{table}[t]
	\caption{\small Effect of starting training with a large $\rho$ and decaying over the course of epochs.}
	\label{samdecay-complete-table}
	\vskip 0.15in
	\begin{center}
		\begin{small}
			\begin{sc}
				\begin{tabular}{ccccc}
					\toprule
					Architecture & Dataset & Method & Full Training & Early Stopping   \\
					\midrule
					\multirow{6}{*}{ResNet18} & \multirow{3}{*}{CIFAR10} & SGD &  $95.45\pm0.13$& $85.98\pm1.26$  \\
					&  & SAM-Optimal & $96.20\pm0.02$ & $85.07\pm0.09$ \\
					&  &  SAM-Decay & $96.14\pm0.07$ & $88.60\pm1.68$ \\
					\cmidrule{2-5}
					& \multirow{3}{*}{CIFAR100} & SGD & $79.37\pm0.12$ & $62.58\pm1.23$ \\
					&  & SAM-Optimal & $80.01\pm0.03$ & $59.38\pm1.14$ \\
					&  &  SAM-Decay & $80.15\pm0.34$ & $63.37\pm0.60$ \\
					\midrule
					\multirow{6}{*}{ResNet50} & \multirow{3}{*}{CIFAR10} & SGD & $95.44\pm0.06$ & $82.99\pm0.75$ \\
					&  & SAM-Optimal & $96.31\pm0.06$ & $81.43\pm2.73$ \\
					&  &  SAM-Decay & $96.42\pm0.10$ & $86.79\pm0.38$ \\
					\cmidrule{2-5}
					& \multirow{3}{*}{CIFAR100} & SGD & $79.50\pm0.33$ & $58.87\pm0.62$ \\
					&  & SAM-Optimal & $82.01\pm0.09$ & $60.20\pm0.97$ \\
					&  &  SAM-Decay & $82.02\pm0.27$  & $61.92\pm1.63$\\
					\bottomrule
				\end{tabular}
			\end{sc}
		\end{small}
	\end{center}
	\vskip -0.1in
\end{table}
\subsubsection{Effect of sample splitting}
Finally, we explore the effect of stochasticity in SAM. To this end, we consider a version of SAM where two different sets of samples are used to calculate inner and outer gradients in SAM. That is, for a batch, we break it into two part and use part for the ascend step and the other for the descend. This ensures the independence between the data used for these two gradient steps, which can reduce the effect of stochasticity in SAM. We keep the experimental setup the same as before, but we use two batch size values of 256 (i.e. 128 samples per gradient) and 512 (i.e. 256 samples per gradient), as well as training for 200 and 400 epochs. We also shuffle batches randomly after each epoch. The results on ResNet50 for this case are shown in Table~\ref{split-table}. We see that the results in this case are worse than optimal SAM in Table~\ref{samdecay-complete-table}. This shows that sample splitting does not seem to be helpful in practice.

\begin{table}[t]
	\caption{\small Effect of sample splitting on SAM (ResNet50)}
	\label{split-table}
	\vskip 0.15in
	\begin{center}
		\begin{small}
			\begin{sc}
				\begin{tabular}{cccc}
					\toprule
					Dataset & Epochs & Batch Size 256 & Batch Size 512   \\
					\midrule
					\multirow{2}{*}{CIFAR10} & 200 & $94.39\pm0.11$ &  $95.03\pm0.09$\\
					& 400 & $94.73\pm0.11$ & $95.50\pm0.04$ \\
					\cmidrule{2-4}
					\multirow{2}{*}{CIFAR100} & 200 & $78.63\pm0.48$ & $78.72\pm0.42$ \\
					& 400 & $78.95\pm0.30$ & $79.55\pm0.33$ \\
					\bottomrule
				\end{tabular}
			\end{sc}
		\end{small}
	\end{center}
	\vskip -0.1in
\end{table}
	
\end{document}